\documentclass[manuscript,screen]{acmart}

\acmJournal{TAAS}

\usepackage{mathtools}          
\usepackage{mathrsfs}           
\usepackage{graphicx}           
\usepackage{subcaption}         
\usepackage[space]{grffile}     
\usepackage{url}
\usepackage{tikz}
\usepackage{enumitem}
\usetikzlibrary{shapes.geometric, positioning, arrows.meta, calc, backgrounds}
\usepackage{booktabs}
\usepackage{multirow}
\usepackage{longtable}
\usepackage{array}
\usepackage{thmtools}

\newcolumntype{P}[1]{>{\raggedright\arraybackslash}p{#1}}
\newcolumntype{M}[1]{>{\raggedright\arraybackslash}m{#1}}

\newcommand{\E}{\mathbb{E}}
\newcommand{\Var}{\mathrm{Var}}

\newcommand{\R}{\mathbb{R}}

\newcommand{\calH}{\mathcal{H}}
\newcommand{\calS}{\mathcal{S}}
\newcommand{\calA}{\mathcal{A}}
\newcommand{\calF}{\mathcal{F}}
\newcommand{\calG}{\mathcal{G}}
\newcommand{\bfa}{\mathbf{a}}

\setcopyright{none}              
\renewcommand\footnotetextcopyrightpermission[1]{} 

\begin{document}

\title[Revisiting Factorization]{Revisiting Action Factorization for Complex Action Spaces}

\author{Timothy Flavin}
\email{Timmy-Flavin@utulsa.edu}
\affiliation{%
  \institution{The University of Tulsa}
  \city{Tulsa}
  \state{Oklahoma}
  \country{USA}
}

\author{Sandip Sen}
\email{Sandip-Sen@utulsa.edu}
\affiliation{%
  \institution{The University of Tulsa}
  \city{Tulsa}
  \state{Oklahoma}
  \country{USA}
}

\begin{abstract}
Many real-world control problems involve hybrid discrete-continuous action spaces. For example, steering and signaling in autonomous driving, and aiming and firing in robotics or video-games. Despite real-world hybrid factorization and reinforcement learning framework support for complex action spaces (e.g., Gymnasium, PettingZoo, TorchRL, SeedRL, Mujoco, etc), the default environments within those frameworks often implement uniform action space configurations (LunarLander, Walker2D, Cheetah, SMAC, SUMO, Ant, Atari). Landmark hybrid-action benchmarks (RoboCup 2D HFO, SC2LE, Platform, CARLA, etc) are mostly heavyweight or archival implementations originating from papers which test one or a small number of competing factorization methods on one kind of control. This article provides a cross-sectional study of factorization methods [independent networks, shared encoder, VDN, QPLEX, Joint, Auto-Regressive] on each of three families of algorithms [PPO, SAC, DQN] across three action spaces [discretized, hybrid, continuous] over four lightweight environments [Platform, hybrid-LunarLander, Hybrid-Shoot, CoopPush]. Accounting for some invalid pairings such as joint-continuous, we are left with 220 configurations to analyze each method. We provide two new C++ parallel gymnasium and petting-zoo compliant environments [CoopPush, Hybrid-Shoot] to isolate particular challenges such as state-dependent inter-action dependence. Finally, we introduce VDN-PPO and PPO-MIX which use a branching critic to assign credit to multi-headed PPO. These variants out-perform all other tested PPO factorizations. Our results suggest that branching dueling architectures balance compute and performance most effectively, with Auto-Regressive actions reaching the highest performance overall and native continuous SAC outperforming discrete and hybrid algorithms, albeit both at increased computational cost.
\end{abstract}



\maketitle

\theoremstyle{acmdefinition}
\newtheorem{assumption}[theorem]{Assumption}

\section{Introduction}

While neural networks have been widely accepted in Deep Reinforcement Learning (D-RL) for their capacity to generalize over dependent hybrid feature spaces, there remains an open debate on factorizing complex action spaces. 
In the single-agent domains, sub-actions $\mathcal{A}_1, \ldots, \mathcal{A}_H$ may control individual motor torques or buttons that combine to form the agent's full joint-action space $\mathcal{A} = \prod_{h=1}^H \mathcal{A}_h$ (exponential size in the number of sub-actions). 
In multi-agent RL (MARL) actions $\mathcal{A}_1, \ldots, \mathcal{A}_K$ refers to individual agents' action spaces which may themselves be decomposed into sub-actions.
Not all actions are interdependent, so the full joint action space need not always be explored. 
For example, the value of movement that avoids danger may not depend upon where an agent is looking or firing, but the decision to fire is highly dependent upon where the agent is aiming. 
There are a range of approaches in the literature covering the spectrum from full independence to joint action learning (JAL). 
In this paper we cover the following general strategies:

\begin{enumerate}
    \item \textbf{Independent:} Each sub-action is sampled from a fully independent policy network $\pi_h(a_h \mid s)$, with all networks updated using the same global training signal.
    
    \item \textbf{Shared Encoder:} A shared state representation $\phi(s)$ branches into multiple heads:
    \begin{enumerate}
        \item \textbf{No Mixing:} Each action head is exposed to the global reward signal where the impact of the other heads is considered independent noise. The gradients combine in the encoder so the result is not truly independent.
        \item \textbf{Value Decomposition (VDN-style):} A single state-value $V(s)$ and per-head advantages $A_h(s, a_h)$ are learned. The joint action-value is additively factored: $Q(s, \mathbf{a}) = V(s) + \sum_{h=1}^H A_h(s, a_h)$.
        \item \textbf{Monotonic (QPLEX-style):} A single state-value $V(s)$ is learned, and per-head advantages are combined via a monotonic mixing network: $Q(s, \mathbf{a}) = V(s) + f_{\text{mix}}\big(A_1(s, a_1), \dots, A_H(s, a_H)\big)$, enforcing $\frac{\partial f_{\text{mix}}}{\partial A_h} \ge 0$.
        \item \textbf{Concatenated SAC (SAC-Concat):} The actor jointly samples continuous dimensions via a squashed Gaussian and discrete dimensions via Gumbel-Softmax. The critic treats these as a single input vector $\mathbf{a} = [a_c, a_d]$ to approximate a unified action-value $Q(s, \mathbf{a})$.

        \item \textbf{Branching Dueling SAC (SAC-BDQ):} The critic utilizes a branching architecture conditioned on the continuous actions. The network encodes the state and continuous actions to compute a base value $V(s, a_c)$ and branches to output discrete advantages $A_h(s, a_{d,h})$ per categorical head, yielding $Q(s, a_c, \mathbf{a}_d) = V(s, a_c) + \sum_{h=1}^H A_h(s, a_{d,h})$.
    \end{enumerate} 
    
    \item \textbf{Auto-Regressive (AR):} The joint policy is factored sequentially using the chain rule. Each action dimension $h$ conditions on the state $s$ and previously sampled actions $a_{<h} = (a_1, \dots, a_{h-1})$. Sub-actions are drawn as $a_h \sim \pi_h(\cdot \mid s, a_{<h})$, yielding the joint policy $\pi(\mathbf{a} \mid s) = \prod_{h=1}^H \pi_h(a_h \mid s, a_{<h})$.
    
    \item \textbf{Joint:} The action space is the Cartesian product of all sub-actions $\mathcal{A} = \prod_{h=1}^H \mathcal{A}_h$. Continuous actions can be discretized into bins for a single Multinomial policy distribution. 
\end{enumerate}

While these factorization strategies exist for each algorithm individually, their performance is often evaluated in isolation within a single algorithmic family (e.g., value-based methods) on a set of benchmarks that may or may not be available for comparison across families. It remains an open question how the effectiveness of a given strategy, such as value decomposition, translates between value-based (DQN)~\cite{mnih2013playingDQN}, policy gradient (PG,PPO)~\cite{schulman2017ppo,sutton1999policyGradient}, and actor-critic (DDPG,SAC)~\cite{haarnoja2018SAC,lillicrap2015DDPG} algorithms. Furthermore, it is unclear how the choice between discrete, continuous, or hybrid action spaces interacts with these factorization methods. For example, continuous actions can approach arbitrary precision, but mixing continuous and discrete gradients may destructively interfere in a shared encoder. Auto-regressive actions can tractably factorize a large joint space, but the ``best'' action requires a sequence of evaluations that adds learning variance, inference latency, and potential implementation complexity by normalizing flows~\cite{rezende2015variational}. This paper provides a cross-sectional comparison of factorization methods with regards to runtime, performance, and implementation effort, to guide future researchers on model selection.

In order to measure performance, a tunable environment is needed. Benchmarks like Gymnasium~\cite{towers2024gymnasium} and Atari100k~\cite{atari100k} 
for single agent discrete RL, PettingZoo~\cite{terry2021pettingzoo} and SMACv2~\cite{ellis2023smacv2} for cooperative MARL, and Deepmind~\cite{tassa2018deepmindcontrolsuite} and Berkeley's~\cite{duan2016benchmarkingcontinuous} continuous control suites are imperative for developing and contextualizing new algorithms and their ability to scale to reasonably complex tasks. 
We believe that a principled benchmark environment is missing for factorization with tunable inter-action dependence, action-rate, and action datatype.

The remainder of this paper is structured as follows. Section~\ref{sec:environments} introduces the four benchmark environments including our two novel contributions. Section~\ref{sec:models} describes each model family and introduces VDN-PPO and PPO-MIX. Section~\ref{sec:experimental} details the experimental setup and runtime normalization methodology. Section~\ref{sec:results} presents results and analysis. Our key findings are: (i)~shared encoder architectures provide the best compute to performance trade-off for most settings. (ii)~VDN-PPO and PPO-MIX substantially outperform shared-encoder PPO on discrete action spaces by redistributing credit to action heads with higher state agency. (iii)~monotonic mixing (QPLEX) does not improve over shared encoder in simple single-agent settings. (iv)~Action type (discrete, hybrid, continuous) has less impact on performance than factorization strategy.

\section{Background / Preliminaries}

\subsection{Reinforcement Learning}

Standard single-agent reinforcement learning problems are often modeled as a Markov Decision Process (MDP) consisting of the tuple $(\mathcal{S}, \mathcal{A}, P, R, \gamma)$. $\mathcal{S} \subset \mathbb{R}^N$ is the set of all environment states, $\mathcal{A} = \prod_{h=1}^{H} \mathcal{A}_h$ is the joint action space factored over $H$ action heads, $P:\mathcal{S}\times \mathcal{A} \rightarrow \Delta(\mathcal{S})$ the transition distribution, and $R:\mathcal{S} \times \mathcal{A} \rightarrow \mathbb{R}$ is the reward function with $\gamma$ as the discount factor and $r_t \in R$ the reward at time t. We denote the state and joint action at time $t$ as $s_t,\mathbf{a}_t$ sampled from factored policy $\pi(\mathbf{a}_t|s_t) = \prod_{h=1}^{H} \pi_h(a_{h,t} \mid s_t)$. Our goal is to find a policy $\pi$ that maximizes the expected sum of discounted rewards, $\mathbb{E}[G_t]$ where $G_t=\sum_{\tau=t}^{\infty} \gamma^{\tau - t}R(s_\tau,\mathbf{a}_\tau)$ and $\mathbf{a}_t \sim \pi(\cdot | s_t),s_{t+1}\sim P(\cdot | s_t, \mathbf{a}_t)$. We do not know $P$ or $R$ in advance. The expected return from time $t$ onward given $\pi$, ($\mathbb{E}[G_t|\pi,s_t]$) is the state-value function: $$V^\pi(s) := \mathbb{E}_{\pi}[\sum_{\tau=t}^{\infty} \gamma^{\tau-t} r(s_\tau, \mathbf{a}_\tau) \mid s_t = s]$$ and the state-action value ($\mathbb{E}[G_t|\pi,s_t,a_t]$) is: $$Q^\pi(s,\mathbf{a}) := \mathbb{E}_{\pi}[\sum_{\tau=t}^{\infty} \gamma^{\tau-t} r(s_\tau, \mathbf{a}_\tau) \mid s_t = s, \mathbf{a}_t = \mathbf{a}]$$

\subsection{Multi-Action as a Special Case of Multi-Agent Control}

A complex action space can be factored into a joint space $\mathcal{A} = \prod_{h=1}^{H} \mathcal{A}_h$, encompassing discrete, continuous, or hybrid dimensions. For a multi-agent system with $K$ agents, each taking $D$ discrete and $C$ continuous actions, we can formulate this as a multi-action problem with $H := K(D+C)$ distinct action heads. 

A cooperative Decentralized Partially Observable MDP (Dec-POMDP) defines an observation function $\mathcal{O}$ mapping the global state to local observations $o_i$ for each agent $i$, with a shared global reward function $R(s, \mathbf{a})$. When such an environment grants full central observability to all agents ($o_i \equiv s$ for all $i$), the Dec-POMDP collapses into a standard single-agent MDP with a high-dimensional, multi-action space. Consequently, contemporary Multi-Agent RL (MARL) factorization techniques can be applied directly to our single-agent setting.

State-of-the-art MARL relies on Centralized Training with Decentralized Execution (CTDE) to bridge the gap between global value estimation and local policy execution. Because our single-agent actors inherently possess full-state observability during execution, we are freed from the strict decentralization constraints of CTDE. This allows us to simplify MARL architectures, such as QPLEX, while taking advantage of their credit assignment capabilities to bypass the computational infeasibility of full JAL.

\section{Related Work}

Factored action spaces across single-agent reinforcement learning (RL) and cooperative multi-agent RL (MARL) provides a rich set of methods for breaking down complex actions into tractable forms.

\subsection{Independent vs Centralized Critics}
At one extreme, independent learners (e.g., IQL~\cite{tan1993IQL}, IPPO~\cite{yu2022IPPO}) treat each action dimension or agent as a separate entity learning from a shared reward signal. While linearly scalable, this approach can suffer from non-stationary instability, local equilibria, and credit assignment difficulties~\cite{matignon2012independent,hernandez2019survey}. Centralized actor-critic methods like MADDPG~\cite{lowe2017MADDPG} conditioned their value function on all actions, providing a separate gradient signal to each decentralized actor for credit assignment. Centralized critics can struggle with representational capacity over the full joint state-action space, especially for many agents~\cite{foerster2018COMA,rashid2020QMIX}.

\subsection{Value Function Factorization}
Another body of work focuses on factoring the joint-action value function $Q_{tot}(s, \mathbf{a})$ into individual per-action components. For greedy action selection to be tractable in combinatorial spaces, it is imperative that the factorization satisfies the Individual-Global-Max (IGM) principle. The IGM principle is equally applicable in single-agent settings with factored action spaces; its primary function is computational tractability of joint argmax (exponential in $H$ without IGM, polynomial with it), rather than decentralized execution. IGM requires $\arg\max_{\mathbf{a}} Q_{tot}(s, \mathbf{a}) = (\arg\max_{a_1} Q_1, \dots, \arg\max_{a_h} Q_h)$.
Value Decomposition Networks (VDN)~\cite{sunehag2017VDN} strictly enforce IGM by assuming additive contributions, defining $Q_{tot}$ as the sum of individual $Q_h$ values. QMIX~\cite{rashid2020QMIX} broadens this representational capacity while maintaining IGM by passing the individual values through a monotonic mixing network, ensuring $\frac{\partial Q_{tot}}{\partial Q_h} \ge 0$. Q-PLEX~\cite{wang2020qplex} mixes advantages instead of full Q values to allow the value-head to absorb some of the Q-value approximation variance. 

Adopting a dueling architecture~\cite{tavakoli2018branching} shifts IGM from Q-values to advantages: $V(s)$ is factored out and $\arg\max_{\mathbf{a}} A_{tot} = (\arg\max_{a_1} A_1, \dots, \arg\max_{a_h} A_h)$; since $Q_{tot} = V(s) + A_{tot}$, maximizing the joint advantage guarantees maximizing $Q_{tot}$. Methods like QPLEX~\cite{wang2020qplex} leverage this by combining dueling architectures with monotonic mixing networks over the advantages, effectively capturing complex inter-action dependencies while strictly guaranteeing IGM. QTRAN~\cite{son2019qtran} approaches this by relaxing structural constraints and enforcing IGM through loss regularization. For policy gradient methods, COMA~\cite{foerster2018COMA} forms per-agent advantages via a counterfactual baseline at an increased evaluation cost. We instead use a first-order Taylor approximation of the mixer gradient to save critic evaluations. VDMPO~\cite{wang2025vdmpo} scales global reward per agent using individual value function contributions, but it does not perform credit assignment to individual actions/agents. We scale the global advantage per action head by its state-level importance (see Section~\ref{sec:method_factored_ppo}).

\subsection{Structured Policy Decomposition}
An alternative paradigm imposes an explicit structure on the policy itself. Auto-regressive methods factorize the joint policy $\pi(a|s)$ into a chain of conditional probabilities, $\Pi_i \pi(a_i|s,a_{<i})$, where each sub-action is selected sequentially, conditioned on the preceding ones~\cite{korenkevych2019autoregressive,li2023ace,rezende2015variational,vinyals2019grandmaster}.
This approach can capture arbitrary dependencies between actions at the cost of sequential computation time and often requires a pre-defined action ordering.

\subsection{Hybrid Action Decompositions}
Common solutions to mixed discrete-continuous spaces can involve parameterizing discrete actions with continuous values~\cite{xiong2018pdqn} or using networks with mixed output heads~\cite{fan2019hybridppo,chen2022hybridsac}. For actor-critic methods, this often requires differentiable relaxations of discrete distributions, such as the Gumbel-Softmax~\cite{jang2016Gumbel} estimator, to enable backpropagation through discrete action choices~\cite{chen2022hybridsac, fan2019hybridppo}. There exist hybrid implementations for PPO~\cite{fan2019hybridppo} and SAC~\cite{chen2022hybridsac,hysac,delalleau2019sacConcat,campos2022sacConcatAnalysis}, with discretization such as bang-bang control for hybrid DQN~\cite{seyde2021bangbang}. For SAC in particular (which is typically fully continuous), it is unclear whether it is better to train a critic which takes both relaxed discrete and continuous actions as input, or continuous actions as input with discrete $Q$ value outputs to serve as a parameterized model. We present the single-value critic $Q(\mathbf{a_d},\mathbf{a_c},s)$ as SAC-Concat, and the parameterized critic $Q=V(s,\mathbf{a_c})+\sum_h^d A_h(s,\mathbf{a_c})$ as SAC-BDQ.

\subsection{Action Embeddings}
Latent action embedding methods such as HyAR~\cite{li2021hyar} and action representation learning~\cite{chandak2019learning} map hybrid spaces into compact continuous latent spaces. Because no canonical method exists to implement such embeddings uniformly across DQN, PPO, and SAC, we exclude them from our comparisons.

\section{Environments}\label{sec:environments}
\begin{figure}
    \centering
    \includegraphics[width=0.8\linewidth]{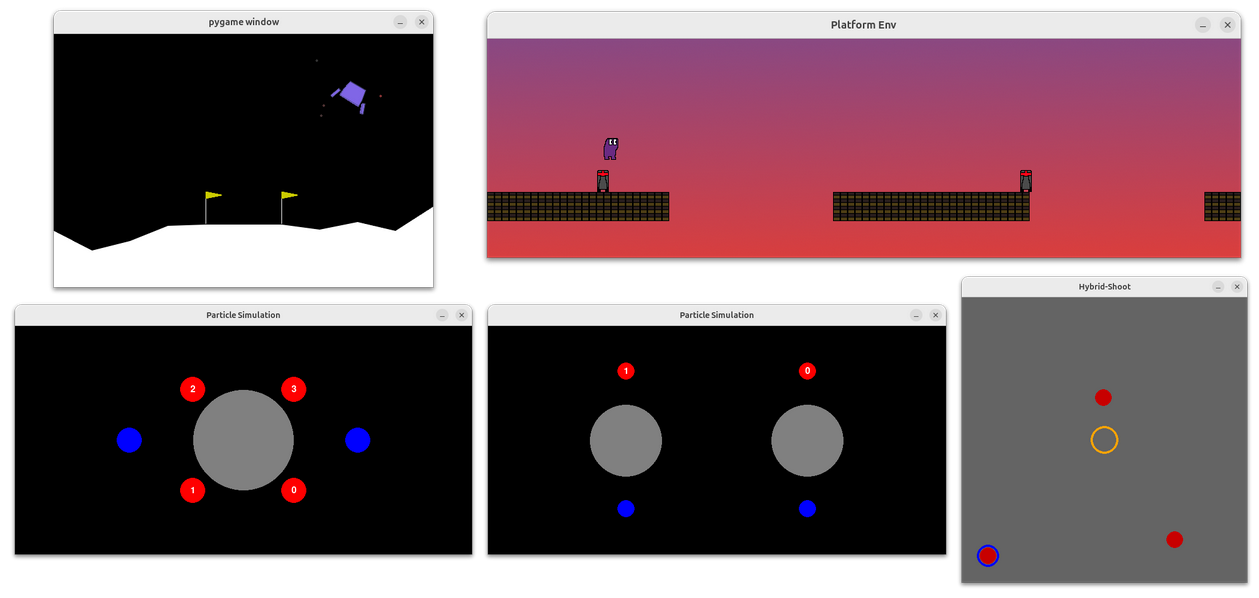}
    \caption{\textbf{Lunar-Landerv3}: Shown top left, \textbf{CoopPush}: \textcolor{red}{Particles}, \textcolor{gray}{Boulders}, and \textcolor{blue}{landmarks}. Default (bottom left) and independent (bottom center). \textbf{Hybrid-Shoot}: Targets:\textcolor{red}{\textbullet} Selected: \textcolor{blue}{o} Shoot Location:\textcolor{orange}{o}. \textbf{Platform}: Agent (Purple) Obstacles (Grey)}
    \label{fig:envs}
\end{figure}

For each environment below, we take the continuous-action variant and discretize continuous intervals into discrete bins by the following transform: Let a continuous value $x \in [x_{\min}, x_{\max}]$ be mapped to a zero-indexed discrete bin index $b \in \{0, 1, \dots, N-1\}$. The uniform bin width is defined inline as $\Delta x = \frac{x_{\max} - x_{\min}}{N}$, and the resulting discretized bin index is computed via $b = \min\left( \left\lfloor \frac{x - x_{\min}}{\Delta x} \right\rfloor, N - 1 \right)$, where $\lfloor \cdot \rfloor$ denotes the floor function. CoopPush and Hybrid-Shoot natively support parallel execution in C++ and we use EnvPool~\cite{envpool} for lunar lander. For hybrid-platform we fall back on gymnasium's SyncVecEnv~\cite{towers2024gymnasium} API for vectorized execution.

\subsection{Contextual-Decoupler}
We introduce Contextual-Decoupler as a minimal benchmark environment designed to isolate and evaluate factored credit assignment in multi-head action spaces. At each timestep, the environment state is defined by a discrete tuple $s = (c, \tau_0, \tau_1)$ sampled uniformly from $\{0, 1\} \times \{0, \dots, N-1\}^2$, where $c \in \{0,1\}$ denotes which action head is active and $\tau_0, \tau_1$ represent the target actions for action heads $a_0$ and $a_1$ respectively. The agent selects a joint action $a = (a_0, a_1) \in \{0, \dots, N-1\}^2$. The scalar reward $R(s,a)$ is governed by the active head $c$, choosing the target action $\tau_c$ while slightly penalizing the inactive head for any action besides zero: $R(s,a) = \mathbb{I}(a_c = \tau_c) - \mathbb{I}(a_c \neq \tau_c) - 0.1 \cdot \mathbb{I}(a_{1-c} \neq 0)$, where $\mathbb{I}(\cdot)$ is the indicator function. The transition dynamics are completely i.i.d.\ across timesteps, meaning the next state $s'$ is drawn uniformly at random such that $\mathcal{P}(s' \mid s, a) = \frac{1}{2N^2}$ for all $s, a$. The episode truncates deterministically after a fixed horizon of $T$ steps. By eliminating temporal state dependencies, this formulation isolates structural action coordination from sequential credit assignment difficulties.

\subsection{Platform}
We include the Platform domain~\cite{masson2016reinforcement}, a classic benchmark designed to study parameterized action spaces. In Platform, an agent (pictured in purple in Figure~\ref{fig:envs}) must traverse a sequence of platforms separated by gaps while avoiding falling or colliding with enemies (grey obstacles). At each step, the agent selects a discrete action type $a_d \in \{0, 1, 2\}$ corresponding to \textbf{Run} (move horizontally), \textbf{Hop} (small jump to clear gaps), and \textbf{Leap} (large jump), respectively. Each discrete action is parameterized by a continuous scalar controlling the magnitude of the movement: running speed $p_0 \in [0, 30]$, hopping power $p_1 \in [0, 720]$, and leaping power $p_2 \in [0, 430]$. When an action is taken, the environment executes the selected action $a_d$ with its corresponding parameter $p_{a_d}$ while ignoring the others. The observation space is a tuple containing a scaled 9-dimensional state vector (representing the kinematics of the agent and enemies alongside the geometric properties of the platforms and gaps) and a discrete time step constraint. By requiring the simultaneous selection of an action type and its continuous magnitude, Platform isolates the difficulty of evaluating discrete choices whose values are inherently dependent on the precision of their continuous parameters. As with the other environments, we discretize the continuous parameter intervals into $N$ discrete bins using our uniform mapping to evaluate algorithms across different levels of control granularity.

\subsection{Lunar Lander}
We discretize actions into 3 bins for actions marked \textbf{discrete} in the wrapper and 5 bins for \textbf{continuous} so that DQN can be viewed as the baseline for discretization granularity. LunarLanderv3 provides a grounded baseline to confirm each model family works as expected under known factorizations with limited inter-action dependence. 

\subsection{Hybrid-Shoot}

We introduce Hybrid-Shoot as a configurable version of a classic parameterized action space problem. In Hybrid-Shoot, $N$ targets positioned randomly on an 2D-plane generate a negative reward each round. One action selects a target, and the other action selects an $[a_x,a_y]$ location. In dependent mode, a target can only be shot when it is selected and $||\langle a_x, a_y\rangle - \langle t_x, t_y\rangle||<c$ where $c$ is the hit radius. Positions $a_x$ and $a_y$ are chosen by separate sub-actions. A target that is currently selected produces no negative reward, and a shot target produces no negative reward starting on the next turn. In independent mode, a target may be shot whether it is selected or not and a single parameter $p$ determines $\langle a_x, a_y\rangle$ by a 16x16 Hilbert curve so that 2D locality is maintained by a single scalar action. This way, the best selection and shoot location do not depend on each-other and there is no dependent $x,y$ components. By altering between dependent and independent modes, and by changing $c$ and the action types of any action, we can isolate inter-action dependence vs precision based learning difficulties. The $x,y$ components are discretized into a 10x10 grid with hit radius 0.1 so that all algorithms can can reach all targets.

\subsection{CoopPush}

Finally we introduce a cooperative push environment where there are three classes of entities 1. Particles, 2. Boulders, and 3. Landmarks. 
Particles and boulders collide with one-another and a reward is given when a boulder overlaps with a landmark that it has not visited. 
The boulders do not move on their own, so the particles need to push them. 
The environment supports a dense reward for the change in distance between each boulder and its nearest unvisited landmark. 
The environment also has two termination modes, \texttt{visit\_one} and \texttt{visit\_all} where the game ends when each boulder has visited at least one, or all possible landmarks respectively. 
The state $s \in \mathbb{R}^{2(N_p + N_b + N_l)}$ concatenates the 2D positions of all $N_p$ particles, $N_b$ boulders, and $N_l$ landmarks. Each particle takes $(\delta_x,\delta_y) \in [-1,1]^2$ as a continuous action with movement capped at magnitude 1.0, or discrete actions with cardinality 9 for no-op and 8 directions. 
Particles are assigned to starting positions in a random order upon reset so that models cannot memorize a state-independent policy.

We consider two layouts pictured in Figure~\ref{fig:envs}. \textbf{Default} vs \textbf{Independent} measures initial action dependence. In the \textbf{Independent} setting there is not enough time to push the other particle's boulder, so the optimal action does not depend on the other particle. In \textbf{Default}, if the left particles are pushing the boulder to the right, it is better for the right particles to get out of the way than to stalemate the boulder. Finally, the engine supports $\delta_t$ and \texttt{num\_physics\_steps} to control action granularity as in Mujoco~\cite{mujoco}. By increasing the steps taken and $\delta_t$ (lower control frequency) continuous precision becomes more advantageous than duty-cycle discrete control.

\section{Model Families}
\label{sec:models}

To study action factorization across algorithm families, we implement variants of DQN, SAC, and PPO. Rather than treating each as a bespoke architecture, we build them from two reusable building blocks: a \emph{Stochastic Actor} that emits factored discrete and continuous sub-actions, and a \emph{Branching Dueling Critic} (denoted \texttt{QS}) that decomposes the joint action-value into a shared state value and per-head advantages. The data flows for these blocks, together with the wiring that realizes monolithic, branching, and auto-regressive factorizations, are shown in the topologies below. A single mechanism ties the families together: each factorization corresponds to a choice of advantage-weighting term $W_h$ applied to the heads of the QS critic before aggregation (see the weighting table in the QS topology). The subsections that follow specify the mathematical formulations, constraints, and credit-assignment mechanics that these shared structures induce.

\tikzset{
    tensor/.style={rectangle, rounded corners, draw=black!80, thick, fill=gray!10, minimum height=2em, minimum width=2.5em, align=center, font=\scriptsize},
    state/.style={tensor, fill=blue!10},
    action/.style={tensor, fill=green!10},
    value/.style={tensor, fill=orange!10},
    weight/.style={tensor, fill=purple!10},
    network/.style={trapezium, trapezium left angle=75, trapezium right angle=75, shape border rotate=270, draw=black!80, thick, fill=white, minimum height=3.5em, align=center, font=\bfseries\scriptsize, inner sep=2pt},
    head/.style={rectangle, draw=black!80, thick, fill=white, minimum height=2em, align=center, font=\bfseries\scriptsize},
    block/.style={rectangle, rounded corners, draw=black!80, thick, fill=white, minimum height=2.5em, align=center, font=\bfseries\scriptsize},
    operator/.style={circle, draw=black!80, thick, fill=white, inner sep=1pt, minimum size=1.2em, font=\bfseries\scriptsize},
    stochastic/.style={rectangle, draw=black!80, thick, dashed, fill=yellow!10, minimum height=1.8em, align=center, font=\scriptsize},
    connection/.style={-{Stealth[scale=1.0]}, thick, draw=black!70},
    dashed_connection/.style={-{Stealth[scale=1.0]}, thick, dashed, draw=black!50}
}

\subsection*{1. Stochastic Actor Topology}
The \texttt{StochasticActor} serves as the foundational policy network for both PPO and SAC. It splits into continuous and discrete branches. The discrete branch offers two pathways depending on the algorithm: standard Categorical sampling (PPO) or a differentiable Gumbel-Softmax relaxation (SAC).

\begin{center}
\begin{tikzpicture}[node distance=0.6cm and 0.8cm]
    \node[state] (input) {$X$};
    \node[network, right=0.6cm of input] (encoder) {Actor\\Encoder};
    \node[head, above right=0.1cm and 0.6cm of encoder] (cont_head) {Cont.\\Head};
    \node[head, below right=0.1cm and 0.6cm of encoder] (disc_head) {Disc.\\Heads};
    \node[stochastic, right=0.6cm of cont_head] (normal) {$\mathcal{N}(\mu, \sigma)$\\+ $\tanh$};
    \node[action, right=0.6cm of normal] (cont_action) {$a_c$};
    \coordinate[right=0.4cm of disc_head] (disc_split);
    \node[stochastic, above right=0.0cm and 0.6cm of disc_split, yshift=0.3cm] (softmax) {Categorical\\(PPO)};
    \node[stochastic, below right=0.0cm and 0.6cm of disc_split, yshift=-0.3cm] (gumbel) {Gumbel-Softmax\\(SAC/Optional)};
    \node[action, right=0.6cm of softmax] (ad_ppo) {$a_d$};
    \node[action, right=0.6cm of gumbel] (ad_sac) {$a_d$};

    \draw[connection] (input) -- (encoder);
    \draw[connection] (encoder) -- (cont_head);
    \draw[connection] (encoder) -- (disc_head);
    
    \draw[connection] (cont_head) -- node[above, font=\scriptsize, yshift=0.3cm] {$\mu, \log\sigma$} (normal);
    \draw[connection] (normal) -- (cont_action);
    
    \draw[connection] (disc_head) -- node[above, font=\scriptsize, yshift=0.3cm, xshift=-0.2cm] {logits} (disc_split) |- (softmax);
    \draw[connection] (disc_split) |- (gumbel);
    \draw[connection] (softmax) -- (ad_ppo);
    \draw[connection] (gumbel) -- (ad_sac);
\end{tikzpicture}
\end{center}

\subsection*{2. Branching Dueling Critic Topology (QS)}
The \texttt{QS} network is the unified critic. It factors the action-value into a state-value $V$ and dimensional advantages $A_h$. Crucially, the combination is mediated by an algorithm-dependent weight $W_h$ applied to each advantage head before summation.

\begin{center}
\begin{tikzpicture}[node distance=0.6cm and 0.8cm]
    \node[state] (input) {$X$};
    
    \node[network, right=0.6cm of input] (encoder) {Critic\\Encoder};
    
    \node[head, above right=0.2cm and 0.6cm of encoder] (value_head) {Value\\Head};
    \node[head, below right=0.2cm and 0.6cm of encoder] (adv_head) {Advantage\\Heads};
    
    \node[value, right=0.6cm of value_head] (v_val) {$V(X)$};
    \node[value, right=0.6cm of adv_head] (a_val) {$A_h(X)$};
    
    \node[operator, right=0.6cm of a_val] (mult) {$\otimes$};
    \node[weight, below=0.4cm of mult] (weight_node) {$W_h$};
    
    \node[operator, right=1.2cm of v_val] (sum) {$\oplus$};
    
    \node[value, right=0.6cm of sum, fill=orange!20] (q_out) {$Q(X)$};

    \node[draw, dashed, fill=white, text width=6cm, align=left, below=0.8cm of encoder, font=\scriptsize, xshift=0.1cm, yshift=-0.1cm] (table) {
        \textbf{Algorithm Weighting ($W_h$)}:\\
        \textbf{Shared / VDN}: $W_h = 1$ \\
        \textbf{VDN-PPO}: $W_h = w^\alpha_h(X)$ \textit{(Importance)} \\
        \textbf{PPO-MIX / QPLEX}: $W_h = \frac{\partial f_{\text{mix}}}{\partial A_h}$
    };

    \draw[connection] (input) -- (encoder);
    \draw[connection] (encoder) -- (value_head);
    \draw[connection] (encoder) -- (adv_head);
    
    \draw[connection] (value_head) -- (v_val);
    \draw[connection] (adv_head) -- (a_val);
    
    \draw[connection] (a_val) -- (mult);
    \draw[dashed_connection] (weight_node) -- (mult);
    
    \draw[connection] (v_val) -- (sum);
    \draw[connection] (mult) -- node[right, font=\scriptsize] {Weighted $A_h$} (sum);
    \draw[connection] (sum) -- (q_out);
\end{tikzpicture}
\end{center}

\subsection*{3. Factorization Wiring Arrangements}
Using the abstract topologies defined above, we wire the specific inputs $X$ to isolate different factorization strategies. 

\begin{center}
\begin{tikzpicture}[node distance=0.6cm and 0.8cm]

    \node[font=\bfseries\small] (title1) {A. SAC-Concat (Monolithic)};
    
    \node[state, below=0.5cm of title1, xshift=-1cm] (s1) {$s$};
    \node[action, below=0.1cm of s1] (ac1) {$a_c$};
    \node[action, below=0.1cm of ac1] (ad1) {$a_d$};
    
    \node[operator, right=0.5cm of ac1, minimum size=1em] (concat1) {$\parallel$};
    \node[block, right=0.8cm of concat1] (critic1) {Standard\\Critic};
    \node[value, right=0.8cm of critic1] (q1) {$Q(s, \mathbf{a})$};
    
    \draw[connection] (s1) -| (concat1);
    \draw[connection] (ac1) -- (concat1);
    \draw[connection] (ad1) -| (concat1);
    \draw[connection] (concat1) -- node[above, font=\scriptsize] {$X$} (critic1);
    \draw[connection] (critic1) -- (q1);

    \node[font=\bfseries\small, right=5cm of title1] (title2) {B. SAC-BDQ};
    
    \node[state, below=0.6cm of title2, xshift=-1cm] (s2) {$s$};
    \node[action, below=0.2cm of s2] (ac2) {$a_c$};
    \node[action, below=0.4cm of ac2] (ad2) {$a_{d,h}$}; 
    
    \node[operator, right=0.5cm of s2, yshift=-0.3cm, minimum size=1em] (concat2) {$\parallel$};
    \node[block, right=0.8cm of concat2] (qs2) {QS Critic\\Block};
    
    \node[operator, right=0.8cm of qs2] (sum2) {$\oplus$};
    \node[value, right=0.6cm of sum2] (q2) {$Q(s, a_c, \mathbf{a}_d)$};

    \draw[connection] (s2) -| (concat2);
    \draw[connection] (ac2) -| (concat2);
    \draw[connection] (concat2) -- node[above, font=\scriptsize] {$X$} (qs2);
    \draw[connection] (qs2) -- node[above, font=\scriptsize] {$V, A_h$} (sum2);
    \draw[connection] (ad2) -| node[right, font=\scriptsize] {selects} (sum2);
    \draw[connection] (sum2) -- (q2);

    \node[font=\bfseries\small, below=3.5cm of title1] (title3) {C. Auto-Regressive Actor};
    
    \node[state, below=0.6cm of title3, xshift=-2cm] (s3) {$s$};
    
    \node[block, right=0.8cm of s3] (actor1) {Actor\\Head 1};
    \node[action, right=0.8cm of actor1] (a1) {$a_1$};
    
    \node[operator, below=0.5cm of a1, minimum size=1em] (concat3) {$\parallel$};
    \node[block, right=0.8cm of concat3] (actor2) {Actor\\Head 2};
    \node[action, right=0.8cm of actor2] (a2) {$a_2$};
    
    \node[right=0.4cm of a2] (dots) {$\dots$};

    \draw[connection] (s3) -- node[above, font=\scriptsize] {$X_1$} (actor1);
    \draw[connection] (actor1) -- (a1);
    
    \draw[connection] (s3) |- (concat3);
    \draw[connection] (a1) -- (concat3);
    \draw[connection] (concat3) -- node[above, font=\scriptsize] {$X_2$} (actor2);
    \draw[connection] (actor2) -- (a2);
    \draw[connection] (a2) -- (dots);

\end{tikzpicture}
\end{center}

\subsection{Branching Dueling Deep Q-Learning (BDQ)}

Our branching~\cite{tavakoli2018branching} dueling~\cite{wang2016dueling} DQN instantiates the QS Critic topology directly: the critic encoder feeds a single global state-value head $V(s)$ and an advantage head per action branch $h$, producing a max-zero advantage vector $\vec{A}^h$ with $\max(\vec{A}^h)=0$. Max-centering is critical in two cases because it makes the shared value function identifiable when advantages are monotonically mixed (see Appendix~\ref{appendix:centering}). Otherwise, the mixing network can learn to represent an arbitrary portion of the value function.

The joint action-value $Q(s, \mathbf{a})$ is assembled from these heads according to the weighting $W_h$ of the topology. In Shared Encoder and VDN modes ($W_h=1$), the heads sum to the joint advantage,
\begin{equation}
    Q(s, \mathbf{a}) = V(s) + \sum_{h=1}^{H} A^h(s, a_h),
    \label{eq:bdq_vdn}
\end{equation}
so that in the shared-encoder limit each branch independently recovers $Q^h(s,a_h)=V(s)+A^h(s,a_h)$. In QPLEX mode, the advantages are instead combined by a state-conditioned monotonic mixing network, $Q(s,\mathbf{a}) = V(s) + f_{\text{mix}}(A^1,\dots,A^H \mid s)$ with $\partial f_{\text{mix}}/\partial A^h \ge 0$. This is a single-agent reduction of QPLEX~\cite{wang2020qplex}: because the dueling architecture supplies full state observability through one global $V(s)$, the mixer operates over max-zero \emph{advantages} rather than raw $Q$-values. Finally, in Auto-Regressive mode the joint action is rolled out by taking the argmax of each advantage head in sequence, feeding earlier sub-actions forward as in the Auto-Regressive Actor wiring. Because DQN selects actions by hard switching rather than sampling on the forward pass, this sequential evaluation can inflate critic variance relative to the single-pass stochastic sampling of SAC or PPO. Finally, we use 11 action bins for all "continuous" action dimensions. 11 dims gives finer control for continuous actions such as engine thrusts in lunar lander, but it also increases the observation size for centralized mixing networks and the overestimation bias caused by argmax targeting. The purpose of including ``continuous" DQN is to show the impact of bin granularity on the various factorizations.

\subsection{Branching Critic PPO: VDN-PPO and PPO-MIX}
\label{sec:method_factored_ppo}

We adapt Proximal Policy Optimization~\cite{schulman2017ppo} to complex action spaces using the Stochastic Actor topology, with categorical heads for discrete sub-actions and squashed Gaussian heads for continuous ones. A standard PPO critic emits a scalar baseline $V(s)$ that feeds an identical Generalized Advantage Estimation (GAE) signal to every head. We instead reuse the QS critic as a \emph{branching} critic, decomposing the joint state-action value as
\begin{equation}
    Q_\phi(s, \mathbf{a}) = V_\phi(s) + \sum_{h=1}^{H} A_{\phi,h}(s, a_h),
    \label{eq:bdq_ppo}
\end{equation}
where each advantage stream $A_{\phi,h}$ is normalized to zero mean under the \emph{current} policy: $\mathbb{E}_{a_h \sim \pi_h}[A_{\phi,h}(s,a_h)]=0$. This on-policy normalization is what distinguishes our critic from standard BDQ, which centers advantages under a uniform or greedy distribution: the policy gradient requires $V_\phi(s)$ to target the on-policy state value $V^\pi(s)$, so the per-head advantages remain valid baselines for the current policy.

Discrete heads are centered by weighting each advantage with its categorical action probability. For continuous heads, the critic produces several advantage `bins' spanning the action range, and centering requires the probability mass $m_{h,c}$ the policy places in each bin $c$ of head $h$. For actions drawn from a squashed Gaussian $\pi_h(s)\rightarrow(\mu_h,\sigma_h)$, this mass follows from the bin's interval endpoints $x_c,x_{c+1}$:
\begin{equation}
    a_h\sim \tanh\left(\mathcal{N}(\mu_h,\sigma_h)\right) \;\Rightarrow\; m_{h,c}=\mathrm{CDF}\left(\tanh^{-1}(x_{c+1});\mu_h,\sigma_h\right)-\mathrm{CDF}\left(\tanh^{-1}(x_{c});\mu_h,\sigma_h\right).
    \label{eq:bin_mass}
\end{equation}
We compute $m_{h,c}$ once for all states in the rollout buffer before any updates, so that the advantage means and the resulting GAE are both grounded in the behavior policy that produced the buffer.

\paragraph{$V$-only target.}
Including the next-step joint action $Q(s_{t+1},\mathbf{a}_{t+1})$ in the regression target would inject substantial variance and compute overhead from noisy advantage estimation across all $H$ heads. We therefore train the branching critic with a $V$-only target $y_t = r_t + \gamma V_\phi(s_{t+1})$. Expanding the squared error of $Q_\phi = V_\phi + \sum_h A_{\phi,h}$ against this target yields
\begin{equation}
    \mathcal{L}(\phi) = \tfrac{1}{2}\,\mathbb{E}_{\tau\sim\pi}\left[\left(\sum_{h=1}^{H} A_{\phi,h}(s_t, a_{h,t}) - \delta_t^V\right)^2\right], \quad \delta_t^V = r_t + \gamma V_\phi(s_{t+1}) - V_\phi(s_t),
    \label{eq:critic_loss}
\end{equation}
which shows that training fits the \emph{sum} of per-head advantages to the $V$-based TD residual $\delta_t^V$, directly bridging the BDQ architecture with GAE (the full derivation and its convergence properties are given in Appendix~(\ref{appendix:theory}). As shown in Figure~\ref{fig:importance}, an offline memory buffer for the critic stabilizes the estimated advantages at the cost of some bias; remaining unbiased under such a buffer would require importance-sampling corrections as in V-Trace~\cite{espeholt2018impala,singh1996rtrace}.

\subsubsection{VDN-PPO: Additive Credit Assignment}

Under the additive decomposition~\eqref{eq:bdq_ppo}, VDN-PPO realizes the $W_h = w_h^\alpha(s)$ entry of the QS weighting table: it assigns credit to each head through an \emph{importance-weighted} GAE. We measure a head's importance by the range of its advantage function and form the corresponding annealed weight,
\begin{equation}
    I_h(s) = \max_{a} A_{\phi,h}(s, a) - \min_{a} A_{\phi,h}(s, a), \qquad
    w_h^\alpha(s) = \frac{I_h(s)^\alpha}{\sum_{k=1}^H I_k(s)^\alpha},
    \label{eq:importance_weights}
\end{equation}
where $\alpha\in[0,1]$ is annealed during training and the convention $0^0=1$ yields uniform weights $w_h^0=1/H$ at $\alpha=0$. The importance-weighted GAE for head $h$ is
\begin{equation}
    \hat{A}_t^{(h)} = \sum_{l=0}^{T-t-1} (\gamma \lambda)^l \; w_h^\alpha(s_{t+l}) \cdot \delta_{t+l}^V,
    \label{eq:iw_gae}
\end{equation}
and the weights $w_h^\alpha$ are treated as stop-gradient constants, so no gradient flows through them into the policy parameters $\theta$. Because $w_h^\alpha(s)$ depends on the state alone, this design admits three formal guarantees, proved in Appendix~\ref{appendix:theory}:

\begin{enumerate}
    \item \textbf{Unbiasedness (Appendix Theorem \ref{thm:unbiased}).} Since $w_h^\alpha(s_{t+l})$ is $\sigma(s_{t+l})$-measurable, it is conditionally independent of the sampled action and factors out of the inner expectation, $\mathbb{E}[w_h^\alpha(s_{t+l})\,\delta_{t+l}^V \mid s_{t+l},\mathbf{a}_{t+l}] = w_h^\alpha(s_{t+l})\,\mathbb{E}[\delta_{t+l}^V \mid s_{t+l},\mathbf{a}_{t+l}]$. Under a converged critic ($V_\phi=V^\pi$) the inner expectation equals the true advantage $A^\pi(s_{t+l},\mathbf{a}_{t+l})$, so $\hat{A}_t^{(h)}$ is unbiased for the importance-weighted advantage along the trajectory.
    \item \textbf{Global gradient recovery (Corollary~\ref{cor:sum}).} Because $\sum_{h=1}^H w_h^\alpha(s)=1$ identically, summing the per-head estimates recovers standard GAE, $\sum_h \hat{A}_t^{(h)} = \hat{A}_t^{\mathrm{GAE}}$. The joint policy gradient therefore remains unbiased for every value of $\alpha$.
    \item \textbf{Variance redistribution (Proposition~\ref{prop:variance_reduction}).} When $I_h(s)>0$ for all $h$ (generically true after exploration), low-importance heads satisfy $(w_h^\alpha)^2 < 1/H^2$, strictly reducing their variance contribution, while heads with $I_h=0$ receive $w_h=0$ and inject no noise at all. The freed variance budget is reallocated to high-importance heads where signal dominates noise.
\end{enumerate}

Operationally, the annealing schedule interpolates between these regimes: early in training $\alpha\approx 0$ gives uniform credit (standard GAE), and as the critic converges $\alpha\to 1$ concentrates credit on the heads whose sub-actions most affect the $Q$-value, suppressing pure noise from heads with no agency at the current state.

\subsubsection{PPO-MIX: Monotonic Non-Linear Advantage Mixing}

The additive structure~\eqref{eq:bdq_ppo} assumes each head contributes to the joint advantage independently. When sub-actions are strictly coupled (e.g., aiming and firing), such interactions may be unrepresentable. FACMAC~\cite{peng2021facmac} shows that non-monotonic factorization can represent tasks monotonic methods cannot; PPO-MIX instead retains monotonic mixing to preserve the IGM property and the tractability of per-head greedy action selection, trading representational capacity for that guarantee. It replaces the linear sum with the QPLEX-style mixer ($W_h = \partial f_{\text{mix}}/\partial A_h$ in the QS table),
\begin{equation}
    Q_\phi(s, \mathbf{a}) = V_\phi(s) + f_{\text{mix}}\left(A_{\phi,1}(s, a_1), \dots, A_{\phi,H}(s, a_H) \mid s\right), \quad \frac{\partial f_{\text{mix}}}{\partial A_h} \geq 0,
    \label{eq:ppo_mix}
\end{equation}
where $f_{\text{mix}}$ is realized by a state-conditioned hypernetwork with strictly non-negative weights, guaranteeing monotonicity and hence IGM.

Credit assignment now requires the gradient of the mixed output at the \emph{current joint action}. A first-order Taylor expansion of $f_{\text{mix}}$ gives the per-head credit
\begin{equation}
    \tilde{w}_h(s, \mathbf{a}) = \frac{\partial f_{\text{mix}}}{\partial A_h}\bigg|_{\mathbf{a}} \!\cdot A_{\phi,h}(s, a_h), \qquad
    w_h^{\text{mix}}(s,\mathbf{a}) = \frac{\tilde{w}_h}{\sum_{k} \tilde{w}_k}.
    \label{eq:mix_weights}
\end{equation}
Evaluating $\partial f_{\text{mix}}/\partial A_h$ at the sampled action is \emph{necessary}: the mixer is non-linear in the advantages, so the local gradient depends on where in advantage-space the joint action falls. A state-only surrogate, evaluating the gradient at the advantage means, which are zero by construction—would collapse to a degenerate zero weight. COMA~\cite{foerster2018COMA} compares the zero value to the value at the action as a finite difference instead of the Taylor approximation that we employ.

\paragraph{Variance cost of action-dependent weights.}
This action-dependence breaks the factorization that underwrites VDN-PPO's unbiasedness. Marginalizing over future actions at step $t+l$, the expectation of the weighted residual no longer decomposes (conditioning on $s_{t+l}$ left implicit):
\begin{equation}
    \mathbb{E}_{\mathbf{a}_{t+l}}\left[w_h^{\text{mix}}(s_{t+l},\mathbf{a}_{t+l})\;\delta_{t+l}^V\right]
    = \mathbb{E}[w_h^{\text{mix}}]\,\mathbb{E}[\delta_{t+l}^V]
      + \underbrace{\mathrm{Cov}\left(w_h^{\text{mix}},\;\delta_{t+l}^V\right)}_{\neq\, 0}.
    \label{eq:cov_term}
\end{equation}
Because $w_h^{\text{mix}}$ and $\delta_{t+l}^V$ share the sampled action $\mathbf{a}_{t+l}$—actions producing large advantages also inflate the mixing gradient—the covariance is generically non-zero. Two consequences follow: (i) the per-head estimates no longer sum to standard GAE, since $\sum_h \partial f_{\text{mix}}/\partial A_h \neq 1$ for a non-linear mixer, so the requirements for policy gradient validity in Theorem~\ref{thm:policy_gradient} are not met, and (ii) the residual covariance injects additional variance into every head's estimate beyond what VDN weighting produces. In short, PPO-MIX trades the theoretical guarantees of VDN-PPO for the representational capacity needed to model strongly coupled sub-actions. Section~\ref{sec:results_branching} shows that PPO-MIX can learn state-dependent importance, and in practice it outperforms VDN slightly for continuous action spaces while matching it for discrete actions. 

\paragraph{Continuous Action Advantage Centering.}
The mechanism of computing bin masses via the CDF is theoretically necessary to preserve the on-policy value baseline $V_\phi(s)$ in continuous domains. By definition, the value function and advantages must satisfy $V(s) = \mathbb{E}_{\mathbf{a}}[Q(s, \mathbf{a})]$, which requires the expected advantage to be strictly zero. In discrete action spaces, this constraint is trivially satisfied by evaluating the sum $V(s) = V(s) + \sum_a \pi(a|s) A(s, a)$, allowing us to center the advantages by scaling each discrete action's advantage by its corresponding policy probability $\pi_a$. However, for continuous actions, this discrete summation is mathematically invalid and the expectation instead demands taking the integral over the action dimension, $\int \pi(a|s) A(s, a) \,da = 0$. By discretizing the continuous advantage function into intervals and evaluating the Cumulative Distribution Function at the bin boundaries, we analytically integrate the policy's probability density over each segment to find the exact probability mass $m_{h,c}$. Using this mass as the proportion by which to scale the bin's advantage acts as a rigorous, tractable approximation of the continuous integral. This ensures the continuous advantage streams are correctly zero-centered under the current policy without introducing the high variance that would result from relying on Monte Carlo action sampling to approximate the expectation.

\subsection{Hybrid SAC}

We adapt SAC to complex action spaces with two variants, both drawn from the factorization wiring diagram. For monolithic evaluation (\textbf{SAC-Concat})~\cite{delalleau2019sacConcat}, the state and both sub-actions are concatenated into the standard critic, using a Gumbel-Softmax~\cite{jang2016Gumbel} relaxation to keep the discrete branch differentiable. Credit assignment through the deep gradient is natural here, but the entropy signals must be separated: a single target entropy drives one distribution to collapse while the other explodes, so we maintain distinct temperature coefficients $\alpha_d,\alpha_c$ for the discrete and continuous dimensions as in ~\cite{delalleau2019sacConcat}. Alternatively, \textbf{SAC-BDQ} (hybrid-sac~\cite{chen2022hybridsac} + branching dueling) routes discrete sub-actions into the QS critic, drawing them from a softmax over advantages as in Munchausen Q-learning~\cite{vieillard2020munchausen}—analogous to P-DQN~\cite{xiong2018pdqn} under entropy constraints—so that the branching critic itself performs the discrete credit assignment. SAC-BDQ reduces to Discrete-SAC for the discrete action space.

\section{Experimental Setup}
\label{sec:experimental}

On each of the environments shown in Figure~\ref{fig:envs}, we analyze Discrete, Hybrid and continuous action types. For the hybrid action configurations, the first $D$ actions are discrete and the following $C$ actions are continuous.

\begin{table}[ht]
    \caption{Action space factorizations per environment. $\mathcal{D}_k$ denotes a discrete action head of cardinality $k$, and $\mathbb{R}^n$ denotes a continuous action head of $n$ dimensions.}
    \label{tab:actionshapes}
    \centering
    \begin{tabular}{lccc}
        \toprule
        \textbf{Environment} & \textbf{Discrete} & \textbf{Hybrid} & \textbf{Continuous} \\
        \midrule
        Contextual Decup. & $\mathcal{D}_2 \times \mathcal{D}_5 \times \mathcal{D}_5$ & N/A & N/A \\
        CoopPush (Default) & $\mathcal{D}_9 \times \mathcal{D}_9 \times \mathcal{D}_9 \times \mathcal{D}_9$ & $\mathcal{D}_9 \times \mathcal{D}_9 \times \mathbb{R}^2 \times \mathbb{R}^2$ & $\mathbb{R}^2 \times \mathbb{R}^2 \times \mathbb{R}^2 \times \mathbb{R}^2$ \\
        CoopPush (Indep.) & $\mathcal{D}_9 \times \mathcal{D}_9$ & $\mathcal{D}_9 \times \mathbb{R}^2$ & $\mathbb{R}^2 \times \mathbb{R}^2$ \\
        Shoot (Dependent) & $\mathcal{D}_3 \times \mathcal{D}_{10} \times \mathcal{D}_{10}$ & $\mathcal{D}_3 \times \mathbb{R} \times \mathbb{R}$ & $\mathbb{R} \times \mathbb{R} \times \mathbb{R}$ \\
        Shoot (Independent) & $\mathcal{D}_3 \times \mathcal{D}_{100}$ & $\mathcal{D}_3 \times \mathbb{R}$ & $\mathbb{R} \times \mathbb{R}$ \\
        Platform & $\mathcal{D}_3 \times \mathcal{D}_{10} \times \mathcal{D}_{10}$ & $\mathcal{D}_3 \times \mathbb{R} \times \mathbb{R}$ & $\mathbb{R} \times \mathbb{R} \times \mathbb{R}$ \\
        Lander & $\mathcal{D}_3 \times \mathcal{D}_3$ & $\mathcal{D}_3 \times \mathbb{R}$ & $\mathbb{R} \times \mathbb{R}$ \\
        \bottomrule
    \end{tabular}
\end{table}

\subsection{Hyperparameter Configuration}
Across all evaluated environments, we established a standardized set of core hyperparameters for Proximal Policy Optimization (PPO), Soft Actor-Critic (SAC), and Deep Q-Networks (DQN). The neural network architectures utilized two hidden layers of 256 units each for the \textit{Push}, \textit{Shoot}, and \textit{Platform} tasks, while the \textit{Lander} tasks utilized 128 units per layer. A uniform learning rate of $5.0 \times 10^{-4}$ was maintained across all algorithms and environments. The discount factor ($\gamma$) was set to $0.98$ for all \textit{Push} tasks and $0.99$ for the remaining environments. Due to differing sample efficiencies and task complexities, the total maximum environment steps (\texttt{max\_steps}) varied by environment based on the number of steps required for the canonical versions of each algorithm to converge (Discrete PPO, BD-DQN, SAC-Concat). Notably, SAC and DQN required identical step counts for convergence across all configurations, while PPO consistently required higher sample limits (offset computationally by parallel environment execution where collection time becomes 1/20th runtime for all algorithms), as summarized in Table~\ref{tab:max_steps}.

\begin{table}[htpb]
\centering
\caption{Total Environment Interactions (\texttt{max\_steps}) by Task and Algorithm}
\label{tab:max_steps}
\begin{tabular}{lcc}
\toprule
\textbf{Environment Grouping} & \textbf{PPO} & \textbf{SAC \& DQN} \\
\midrule
Push (Discrete \& Continuous)   & $2,000,000$ & $1,000,000$ \\
Push (Hybrid)                   & $3,000,000$ & $2,000,000$ \\
Lander (Discrete \& Continuous) & $600,000$   & $400,000$   \\
Lander (Hybrid)                 & $1,000,000$ & $600,000$   \\
Shoot (All Variants)            & $1,500,000$ & $800,000$   \\
Platform (All Variants)         & $500,000$   & $200,000$   \\
\bottomrule
\end{tabular}
\end{table}

\subsection{Contextual Decomposition Verification}
\label{sec:method_contextual}

The Contextual-Decoupler of Section~\ref{sec:environments} is a contextual bandit ($\gamma$-discounted but i.i.d.\ in state) whose ground-truth per-head credit assignment is known in closed form: the active head $h=c$ drives the dominant $\pm1$ reward while the inactive head $h\neq c$ contributes only a $-0.1$ shaping penalty, so the true importance is $I_h^\star(s)=\mathbb{I}[h=c]$ and the active head's $\pm1$ variance acts as pure noise on the inactive head's policy-gradient signal. This is precisely the regime in which Proposition~\ref{prop:variance_reduction} (Appendix~\ref{appendix:theory}) predicts a benefit from importance reweighting, because the uniform estimator ($\alpha=0$) propagates the active head's $\delta^V$ variance onto the inactive head undiminished. Because the task isolates pure algorithmic capability rather than environment-specific control, it is not included in the ``practical'' sweep of Table~\ref{tab:max_steps}; instead we run two targeted studies on $N=5$ actions per head ($\mathcal{D}_2\times\mathcal{D}_5\times\mathcal{D}_5$) over $10^5$ environment steps.

\paragraph{Cross-family reward sweep.}
To analyze the impact of critic factorization performance across \emph{all} three algorithm families, we run every pairing of \{PPO, SAC, DQN\} with critic structure \{shared-nomix, VDN, Q-PLEX\}. SAC is evaluated in its \textbf{SAC-BDQ} ($Q$-critic) and \textbf{SAC-Concat} ($V$-critic) forms. All agents collect $10^5$ steps and we report the smoothed vectorized reward curve over training (Figure~\ref{fig:branchreward}).

\paragraph{Factorial PPO study.}
To isolate the \emph{magnitude} of each source of gain within PPO and to test Proposition~\ref{prop:variance_reduction} directly, we run a full factorial over critic structure and importance estimator. The critic is one of \textbf{nomix} (a single scalar value baseline, no per-head decomposition), \textbf{VDN} (additive mixer $\sum_h A_{\phi,h}$,~\eqref{eq:bdq_ppo}), or \textbf{Q-PLEX} (the state-conditioned monotonic hypernetwork mixer of~\eqref{eq:ppo_mix}). The importance estimator is one of:
\begin{itemize}[leftmargin=2em,topsep=2pt,itemsep=1pt]
    \item \textbf{uniform} ($\alpha\equiv0$): weights pinned to $w_h=1/H$, disabling reweighting; this isolates the critic-factorization effect.
    \item \textbf{grad}: the first-order mixer-gradient surrogate $w_h\propto|A_{\phi,h}(s,a_h)\cdot\partial_{A_h}Q_{\text{tot}}|^\alpha$ of~\eqref{eq:mix_weights}, evaluated at the observed action.
    \item \textbf{range}: the exact range-based importance of~\eqref{eq:importance_weights}, $w_h\propto I_h(s)^\alpha$ with $I_h(s)=\max_a A_{\phi,h}-\min_a A_{\phi,h}$.
\end{itemize}
For \textbf{grad} and \textbf{range} the annealing parameter is ramped linearly $\alpha\!:0\to1$ over the first $40$ updates, reaching its terminal value well within the $\approx48$ updates of a run. All agents use PPO with $4$ epochs, minibatch size $128$, learning rate $10^{-3}$, a $[64,64]$ trunk, and (for Q-PLEX) a $64$-dimensional mixing embedding. We collect $2048$-transition rollouts across $16$ synchronous environments and report $n=16$ random seeds per cell. We track three learning-dynamics metrics and one direct variance metric:
\begin{itemize}[leftmargin=2em,topsep=2pt,itemsep=1pt]
    \item \textbf{Final reward}: mean smoothed episode return over the last $50$ episodes.
    \item \textbf{AUC}: mean of the smoothed return curve over all of training (a sample-efficiency proxy).
    \item \textbf{Steps$\to$50}: environment steps to first reach a smoothed return of $50$ (lower is better; censored at $10^5$).
    \item \textbf{inact/act ratio}: the empirical variance of the \emph{inactive} head's importance-weighted GAE advantage $\hat{A}^{(h)}$ relative to that of the \emph{active} head, measured over the final third of updates. This is the direct empirical analogue of the $(w_h^\alpha)^2$ scaling in Proposition~\ref{prop:variance_reduction}.
\end{itemize}
We additionally score the learned weights $w_h^\alpha(s)$ against the ground-truth importance $I_h^\star(s)$ to test whether they recover the active head (Figure~\ref{fig:importance}).

\section{Results and Analysis}
\label{sec:results}

\subsection{Branching Critics on the Contextual Decoupler}
\label{sec:results_branching}

The Contextual-Decoupler (Section~\ref{sec:environments}, design in Section~\ref{sec:method_contextual}) lets us decompose the performance of branching critics into separable sources, because the ground-truth credit assignment $I_h^\star(s)=\mathbb{I}[h=c]$ is known in closed form. We report three studies: a cross-family reward sweep (Figure~\ref{fig:branchreward}) showing that critic factorization helps every algorithm family, a factorial PPO study (Tables~\ref{tab:factored_main}-\ref{tab:factored_sig}) isolating the magnitude of each source within PPO, and an importance-recovery probe (Figure~\ref{fig:importance}) confirming the range weights track the true active head.

\paragraph{Critic factorization helps every algorithm family.}
Figure~\ref{fig:branchreward} sweeps \{PPO, SAC, DQN\} against \{shared-nomix, VDN, Q-PLEX\} critics. The factorization effect is visible across all three families. \textbf{SAC-BDQ} routes the discrete heads through the branching dueling critic and solves the task almost immediately, reaching a near-perfect score within a few thousand steps. All three \textbf{DQN} variants plateau around $75$, capped not by the critic but by the residual $\epsilon$-greedy exploration actions that the on-policy reward curve still pays for. \textbf{VDN-PPO and PPO-MIX (Q-PLEX)} converge to roughly $80$ by the end of training, whereas the unfactored \textbf{PPO shared-nomix} baseline crawls to only $\sim15$: with a single scalar value, the inactive head's gradient is dominated by the active head's $\pm1$ fluctuation and the weak $-0.1$ shaping signal is never recovered within the budget. \textbf{SAC-Concat} reaches only $\sim50$; SAC requires a very low target entropy on this task and is sensitive to its setting, since a single temperature drives one head's distribution to collapse while the other explodes (\S\ref{sec:method_factored_ppo}), the engineering motivation for the separate $\alpha_d,\alpha_c$ coefficients of SAC-Concat (as in the original paper~\cite{delalleau2019sacConcat}).

\begin{figure}[h]
    \centering
    \includegraphics[width=0.7\linewidth]{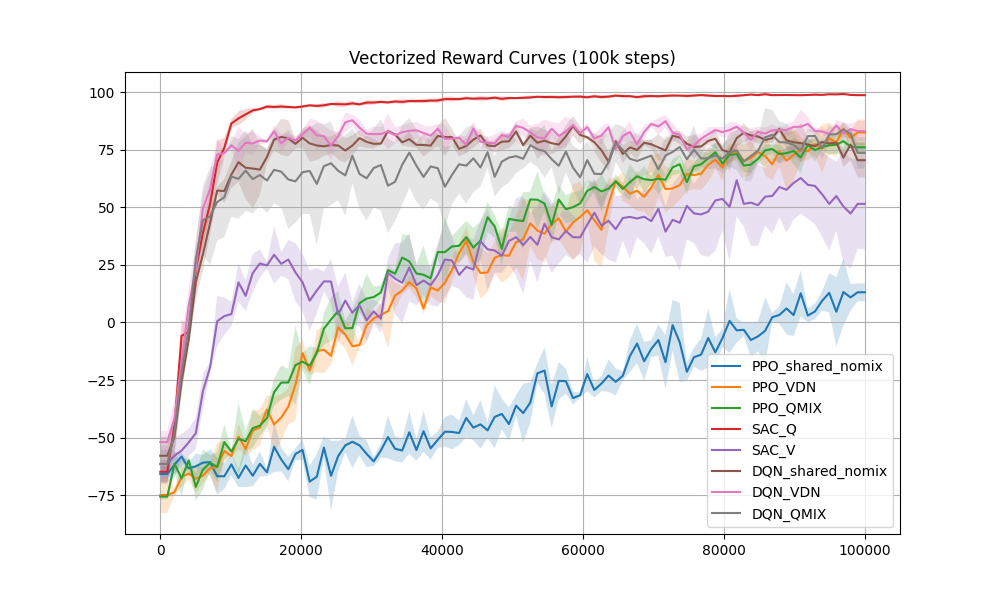}
    \caption{Vectorized reward curves on the Contextual-Decoupler ($10^5$ steps) for each algorithm family paired with a shared-nomix, VDN, or Q-PLEX critic. SAC-BDQ ($Q$-critic) solves the task almost immediately; the DQN variants plateau near $75$ (held back by $\epsilon$-greedy exploration actions); VDN-PPO and PPO-MIX (Q-PLEX) converge to $\sim80$ while the unfactored PPO shared-nomix baseline reaches only $\sim15$. SAC-Concat ($V$-critic) reaches $\sim50$ and is sensitive to the target entropy. (5 seeds; shaded region: SEM.)}
    \label{fig:branchreward}
\end{figure}

Table~\ref{tab:factored_main} runs the factorial over critic structure and importance estimator (Section~\ref{sec:method_contextual}). Replacing the scalar baseline (nomix) with any factored critic raises the final reward from $6.7$ to $82$-$89$: the gap is highly significant on every metric (nomix vs.\ uniform mixers gives Welch $t\approx28$-$40$, $p<10^{-14}$, Cohen's $d\approx11$--$14$), and nomix never reaches the threshold of $50$ within budget. The factored critic absorbs the action-dependent return into the per-head advantage streams, so $V_\phi$ fits the clean state value and each head receives an advantage scoped to its own action, the $V$-only mechanism of~\eqref{eq:critic_loss} (Appendix~\ref{appendix:theory}).

\begin{table}[t]
\centering
\caption{Aggregate results on the Contextual-Decoupler ($n=16$ seeds, mean $\pm$ std). Final reward and AUC: higher is better; Steps$\to$50: lower is better. The inact/act ratio is the inactive-head advantage variance relative to the active head; values below $1$ indicate the variance suppression predicted by Proposition~\ref{prop:variance_reduction} (Appendix~\ref{appendix:theory}).}
\label{tab:factored_main}
\begin{tabular}{lcccc}
\toprule
Config & Final reward & AUC & Steps$\to$50 & inact/act ratio \\
\midrule
nomix          & $6.7 \pm 8.9$  & $-37.2 \pm 5.7$ & $100000 \pm 0$    & --- \\
\midrule
VDN (uniform)  & $82.3 \pm 4.5$ & $23.2 \pm 4.7$  & $61376 \pm 5472$  & $1.000 \pm 0.000$ \\
VDN (grad)     & $85.3 \pm 3.3$ & $25.2 \pm 4.2$  & $59861 \pm 4918$  & $0.863 \pm 0.011$ \\
VDN (range)    & $85.9 \pm 2.5$ & $26.6 \pm 4.0$  & $56932 \pm 4067$  & $0.836 \pm 0.021$ \\
\midrule
Q-PLEX (uniform) & $85.3 \pm 2.4$ & $29.0 \pm 3.5$  & $55821 \pm 5310$  & $1.000 \pm 0.000$ \\
Q-PLEX (grad)    & $89.2 \pm 1.5$ & $33.0 \pm 3.0$  & $51478 \pm 3225$  & $0.853 \pm 0.014$ \\
Q-PLEX (range)   & $87.4 \pm 2.5$ & $33.2 \pm 3.6$  & $50165 \pm 3310$  & $0.826 \pm 0.047$ \\
\bottomrule
\end{tabular}
\end{table}

\paragraph{Importance weighting reduces inactive-head variance exactly as predicted.}
At $\alpha=0$ the inactive and active-head advantages are identical, so the ratio is exactly $1.000$. With $\alpha\to1$ the inactive-head variance is suppressed to $0.83$--$0.86\times$ the active head's (Table~\ref{tab:factored_main}); a paired within-run one-sided test confirms the suppression for every weighted configuration ($t\in[-25,-11]$, $p<10^{-8}$). This is a direct empirical confirmation of the $(w_h^\alpha)^2$ variance scaling in Proposition~\ref{prop:variance_reduction} (Appendix~\ref{appendix:theory}). Because the weights sum to one and are state-measurable, reweighting is unbiased (Corollary~\ref{cor:sum}), so we expect it to improve \emph{dynamics} rather than the asymptotic optimum. Table~\ref{tab:factored_sig} confirms this: the uniform~$\to$~weighted effect is consistent in sign across all three metrics and both mixers: higher final reward, higher AUC, fewer steps to threshold, reaching significance in most cells and most strongly for Q-PLEX (final $+3.9$, $p=10^{-5}$, $d=1.9$). Magnitudes are modest ($\sim3$--$4$ reward, $\sim4$--$6{,}000$ fewer steps, i.e.\ $\sim4\%$ reward and $\sim8$--$10\%$ sample efficiency), second-order relative to the $\approx77$-point factorization effect, but robust. Weighting also tightens seed-to-seed spread (VDN final std $4.5\to2.5$; Q-PLEX $2.4\to1.5$), an outer-level stability gain from the same mechanism.

\begin{table}[t]
\centering
\caption{Significance of the variance-reduction effect: uniform ($\alpha=0$) vs.\ weighted estimators ($n=16$, Welch's two-sided $t$). $\Delta$ is (weighted $-$ uniform); for Steps$\to$50 a negative $\Delta$ is an improvement in the number of steps taken to get to a score of 50 compared to uniform weighting. $d$ is Cohen's $d$. Significance: \textsuperscript{$\ast$}$p<0.05$, \textsuperscript{$\ast\ast$}$p<0.01$, \textsuperscript{$\ast\ast\ast$}$p<0.001$.}
\label{tab:factored_sig}
\begin{tabular}{llrrrl}
\toprule
Comparison & Metric & $\Delta$ & $t$ & $d$ & sig. \\
\midrule
VDN: range $-$ uniform
  & Final       & $+3.66$    & $2.88$  & $1.02$  & \textsuperscript{$\ast\ast$} \\
  & AUC         & $+3.41$    & $2.22$  & $0.78$  & \textsuperscript{$\ast$} \\
  & Steps$\to$50& $-4444$    & $-2.61$ & $-0.92$ & \textsuperscript{$\ast$} \\
\midrule
VDN: grad $-$ uniform
  & Final       & $+3.04$    & $2.19$  & $0.77$  & \textsuperscript{$\ast$} \\
  & AUC         & $+1.99$    & $1.28$  & $0.45$  & ns \\
  & Steps$\to$50& $-1515$    & $-0.82$ & $-0.29$ & ns \\
\midrule
Q-PLEX: grad $-$ uniform
  & Final       & $+3.88$    & $5.46$  & $1.93$  & \textsuperscript{$\ast\ast\ast$} \\
  & AUC         & $+3.99$    & $3.50$  & $1.24$  & \textsuperscript{$\ast\ast$} \\
  & Steps$\to$50& $-4343$    & $-2.80$ & $-0.99$ & \textsuperscript{$\ast\ast$} \\
\midrule
Q-PLEX: range $-$ uniform
  & Final       & $+2.08$    & $2.37$  & $0.84$  & \textsuperscript{$\ast$} \\
  & AUC         & $+4.21$    & $3.40$  & $1.20$  & \textsuperscript{$\ast\ast$} \\
  & Steps$\to$50& $-5656$    & $-3.62$ & $-1.28$ & \textsuperscript{$\ast\ast$} \\
\bottomrule
\end{tabular}
\end{table}

\paragraph{The exact and first-order estimators are equivalent for control, but only the exact one is interpretable.}
The \textbf{grad} (first-order,~\eqref{eq:mix_weights}) and \textbf{range} (exact,~\eqref{eq:importance_weights}) weights are statistically indistinguishable for learning: VDN slightly favours \textbf{range} and Q-PLEX slightly favours \textbf{grad} on final reward ($|t|<2.5$), with no consistent winner. Both deliver the full variance-reduction benefit. We nonetheless recommend \textbf{range} as the default, since it doubles as a faithful interpretability probe: as shown in Figure~\ref{fig:importance}, its weights track the ground-truth active head over training, with $r\approx0.97$ correlation to $I_h^\star$, whereas the first-order surrogate is not a reliable estimate of which head is responsible despite equal learning performance. 
Figure~\ref{fig:importance} also shows that giving the critic an offline buffer (breaking the on-policy requirement of Theorem~\ref{thm:unbiased} and incurring the bias bounded by Proposition~\ref{prop:bias_critic_error} in Appendix~\ref{appendix:theory}) can improve weight correlation, but in practice we found the added bias damages the policy gradient too greatly and leads to degraded performance. For this reason and theoretical correctness we use the online buffer for all experiments going forward.

\begin{figure}[h]
    \centering
    \begin{minipage}[b]{0.62\linewidth}
        \centering
        \includegraphics[width=\linewidth]{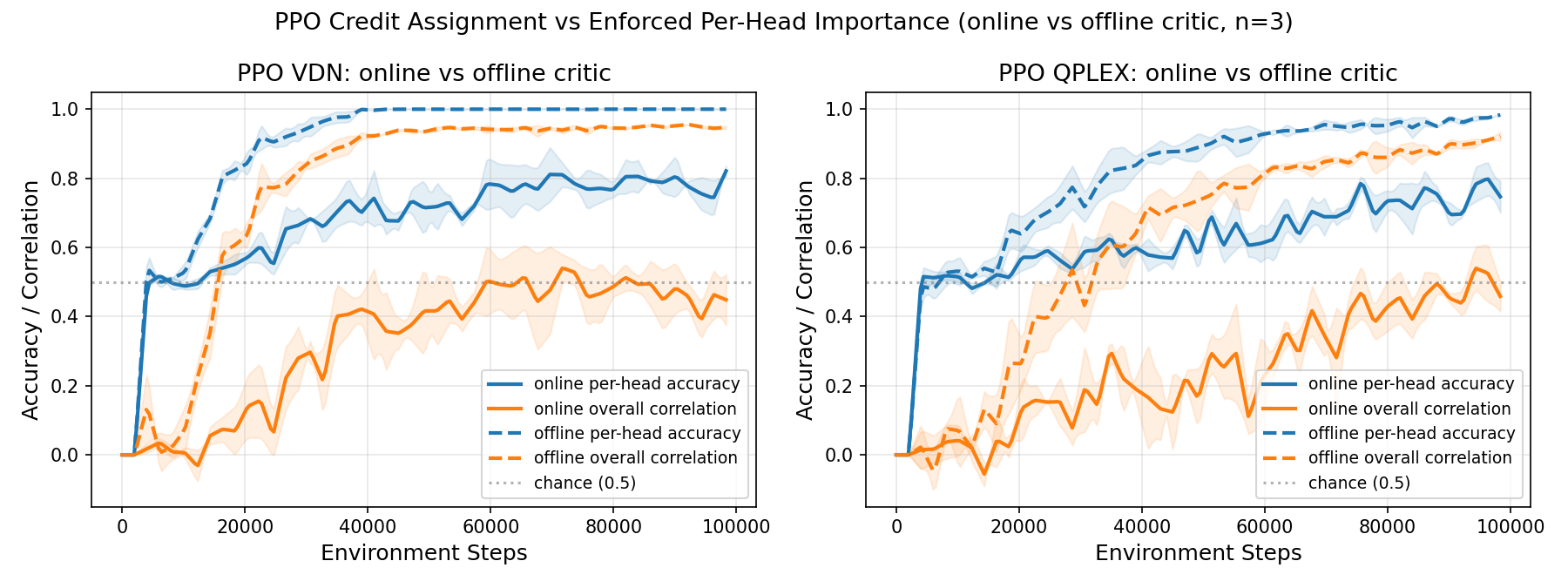}
    \end{minipage}
    \hfill
    \begin{minipage}[b]{0.34\linewidth}
        \centering
        \includegraphics[width=\linewidth]{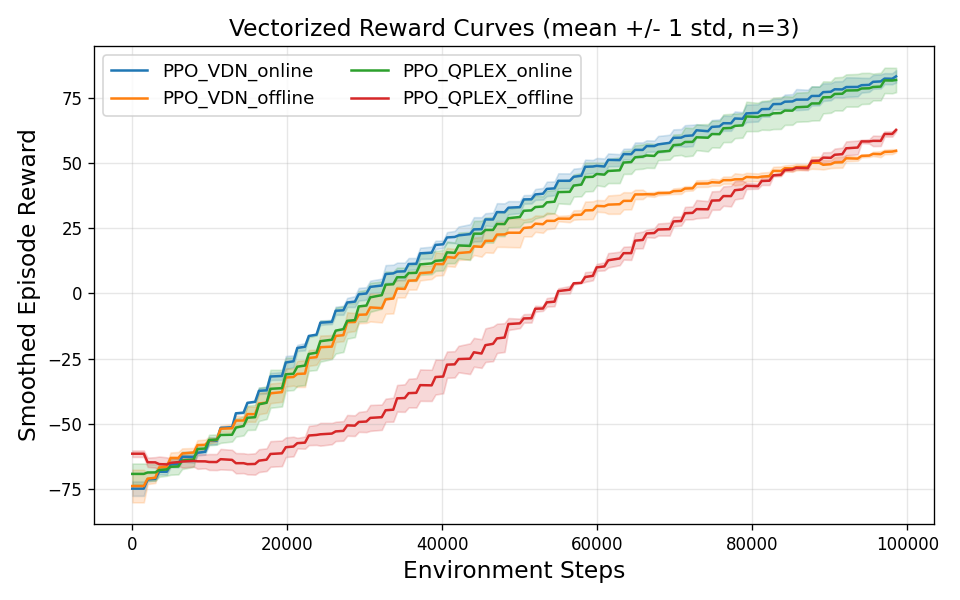}
    \end{minipage}
    
    \caption{Orange is Correlation between the true state's active head $I_h^\star$ and the internal range-based importance weighting over training. Blue is accuracy (percent of the time the more important head was correctly chosen) Both VDN and QPLEX progress towards true importance. Giving the critic an offline buffer adds theoretical bias (Proposition~\ref{prop:bias_critic_error}, Appendix~\ref{appendix:theory}) leading to a worse env reward, but it but improves importance estimation. (5 seeds; shaded region: SEM.) We use the online-only critic throughout this paper's other experiments for it's empirical performance and theoretical correctness, but future work may investigate dual critics with the offline critic handling importance while the online critic supplies the GAE baseline.}
    \label{fig:importance}
\end{figure}

\paragraph{Summary of sources and magnitudes.}
The environment cleanly separates two contributions, ordered by magnitude: (i)Policy Gradient ~\emph{critic factorization} is the headline effect ($+77$ reward, $d\approx11$), arising because the per-head advantage streams absorb action-dependent return and let $V_\phi$ track $V^\pi$; (ii)~\emph{importance-weighted variance reduction} is a statistically robust second-order refinement to suppress inactive-head estimator variance $\sim15\%$ ($p<10^{-8}$), translating to $\sim4\%$ higher reward, $\sim8$--$10\%$ better sample efficiency, and reduced seed variance, all without changing the asymptotic optimum as Corollary~\ref{cor:sum} requires. PPO-MIX retains state-dependent importance but pays the action-dependent covariance cost of~\eqref{eq:cov_term}, which is why it tracks rather than beats VDN-PPO here.

\subsection{Summary Results}

Grouping by \{env, algorithm, action dtype\} results in 45 plots with up to 5 factorizations each. For completeness, we include each individual result in Appendix~\ref{app:appendixD}, but we sum over environments to isolate the global impacts on algorithm (PPO, DQN, SAC) and Datatype (Discrete, Hybrid, Continuous) over our benchmark suite. We perform min-max normalization on a per-environment basis so that the highest and lowest score achieved by each algorithm, any seed, represent the Y-axis scale [0.0,1.0] for all pairings. We used BD-DQN, Discrete PPO, and SAC to determine the number of steps required to converge on each environment for each family of algorithm. Steps were tuned for each algorithm because this paper is not intended to choose between DQN, PPO and SAC, but rather to compare the impact of factorization on each. By treating the X axis as steps from [0.0,1.0] and interpolating, we report the mean normalized performance per normalized env step for pairing in Figure~\ref{fig:allCombos} over 5 seeds per experiment for a total of 1,100 runs.

\begin{figure}
    \centering
    \includegraphics[width=1.0\linewidth]{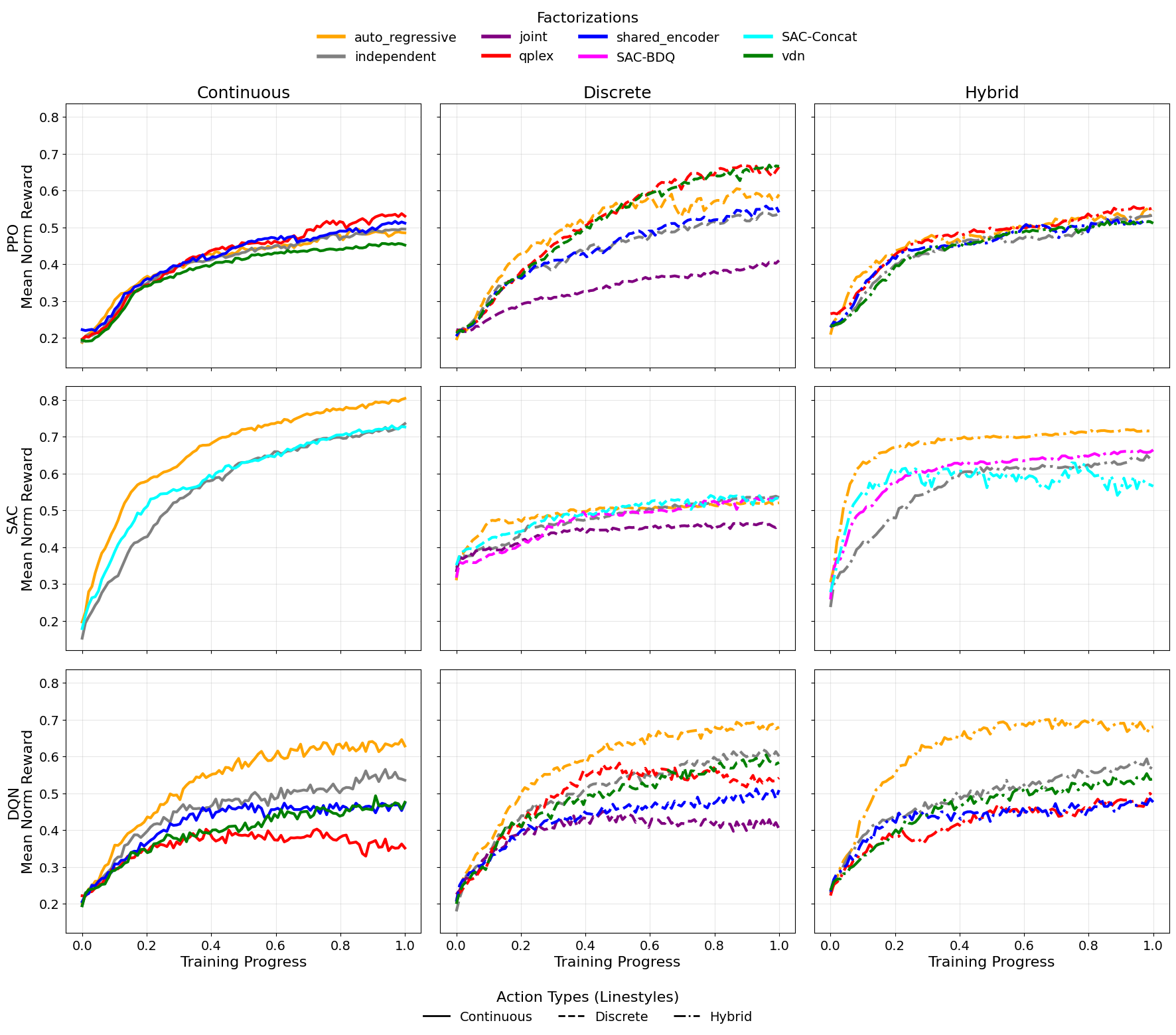}
    \caption{Mean results over all environments, grouped by action datatype (columns) and algorithm family (rows). Note that hybrid actions were given longer to train on some environments to account for gradient interference}
    \label{fig:allCombos}
    \descriptionlabel{}
\end{figure}

\subsubsection{PPO}

For continuous action spaces, the Monotonic critic seems to slightly benefit PPO, but "branching" PPO otherwise reduces to the standard continuous implementation. The improvement is not substantial given extra implementation effort. For discrete actions, however, we show that two reductions in variance strongly improve performance over baseline implementations of PPO, including IPPO. Finally, Joint action factorization lags behind the others, likely because our action spaces fall mostly between 300 and 6,000 choices. All 3 families get relatively similar Joint performance, probably limited by encoder functional capacity at 64 or 256 neurons feeding into hundreds of actions.

\subsubsection{SAC}

Pure continuous action SAC outperforms all other factorizations in this environment sweep, while discrete SAC at $\sim$0.5 underperforms. In the hybrid regime, we find that parameterized / SAC-BDQ outperforms the standard concatenation.

\subsubsection{DQN}

Independent networks and VDN/BDQ-DQN style factorizations consistently perform well here, but QPLEX does not scale well to continuous actions relative to discrete ones. Independent networks scale in the number of total parameters with the action space size, so they perform consistently across environments. 

\subsection{Notable Environment Results}

The complete set of environment training curves are included as supplementary material, but we have chosen a few to highlight here as either instructive or surprising examples. 

\begin{figure}[htbp]
    \centering
    \begin{subfigure}[b]{0.48\textwidth}
        \centering
        \includegraphics[width=\textwidth]{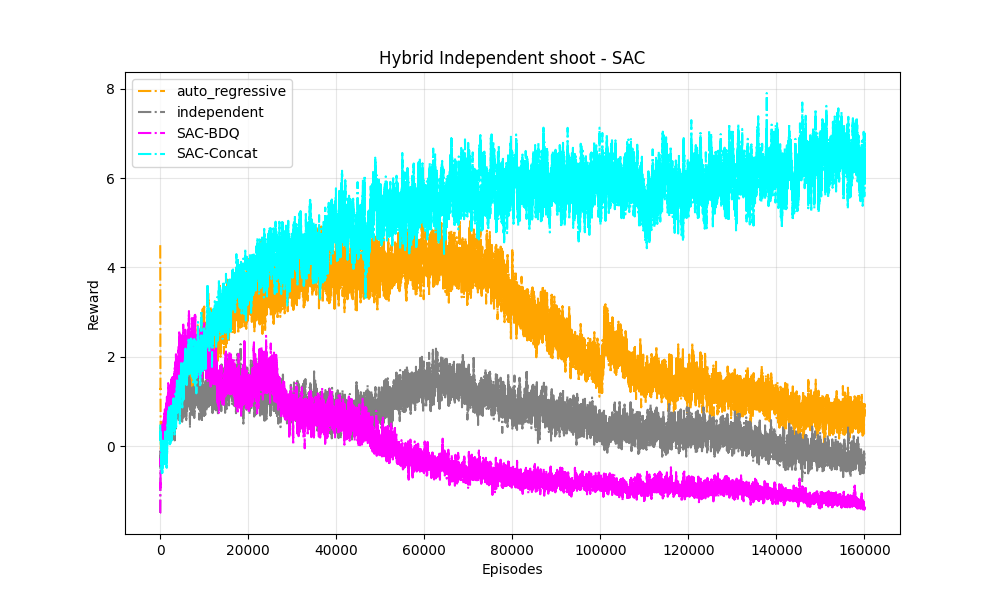} 
        \caption{Independent Shoot Results SAC continuous gets $\sim8$}
        \label{fig:indep_sac}
    \end{subfigure}
    \hfill 
    \begin{subfigure}[b]{0.48\textwidth}
        \centering
        \includegraphics[width=\textwidth]{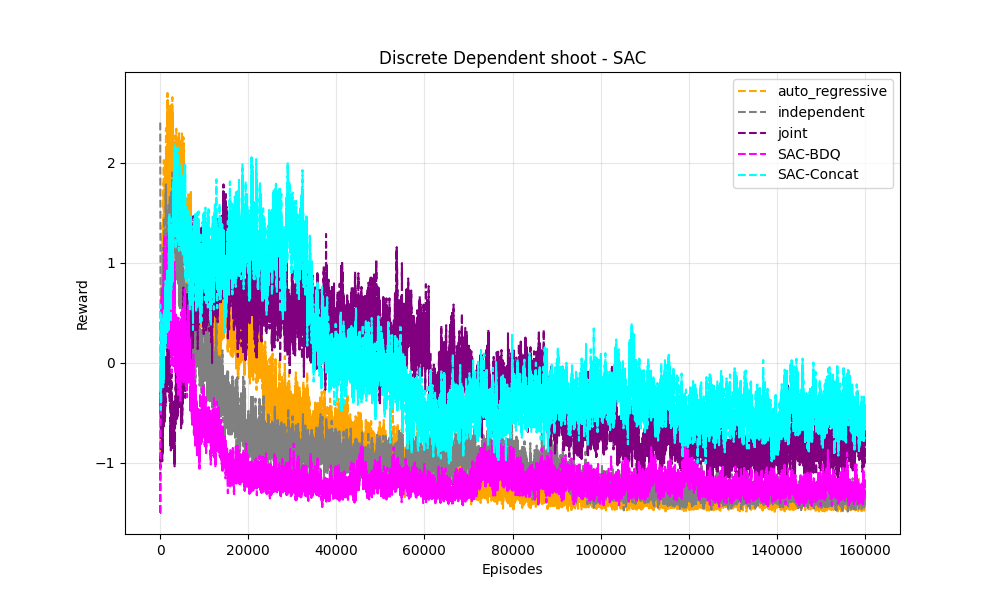} 
        \caption{Dependent Shoot Results SAC (Continuous gets $\sim35$}
        \label{fig:depsac}
    \end{subfigure}
    \caption{The discrete and hybrid shoot environments are SAC's worst performance because 300 action choices combined with a discrete entropy target allows auto-tuned $\alpha$ to overpower the learning signal, especially because hitting a target reduces the number of optimal actions (and therefore entropy reward)}
    \label{fig:twosac}
\end{figure}

As shown in Figure~\ref{fig:twosac}, discrete actions in the shoot environment perform poorly under SAC. We identify that this performance is closely tied to the target entropy formulation, $\mathcal{H}_{target} = \epsilon_H \log(|A_d|)$, where $\epsilon_H$ is the target entropy percentage hyperparameter. While $\epsilon_H = 0.5$ is effective across standard environments, the shoot environment's large action space and low-entropy optimal policy necessitate a lower $\epsilon_H$ of 0.1 or 0.2; otherwise, the entropy signal dominates the extrinsic environment reward.

We hypothesize that this performance degradation occurs because the soft exploration objective is fundamentally misaligned with the environment's task dynamics. At the beginning of an episode, three active targets exist, meaning the optimal policy distributes probability mass across a corresponding subset of the valid target locations and jammer choices. As targets are successfully hit, the number of plausible optimal actions strictly decreases. Consequently, the maximum possible entropy of the optimal policy must monotonically decrease as the episode progresses toward completion. Likewise, the discrete tuner uses $\pi(s,a)=\frac{e^{Q(s,a)/\alpha}}{\sum_b e^{Q(s,b)/\alpha}}$ to generate actions. This means that BDQ mode can directly force the entropy to meet the target while Concat mode uses reward to apply feedback to the discrete actor network. Finally, continuous actions do not scale with action granularity so the target entropy is less important.

Because SAC's soft objective formulation rewards the agent for occupying states that support higher-entropy policies, it inherently penalizes the transition into these progressed, lower-entropy states. Furthermore, because the observation remains static unless a target is successfully hit, taking sub-optimal exploratory actions does not yield state transitions. Advancing the task requires a strictly less stochastic policy, placing the environment's underlying goal in direct mathematical conflict with SAC's soft exploration mechanism.

\begin{figure}
    \centering
    \includegraphics[width=1.0\linewidth]{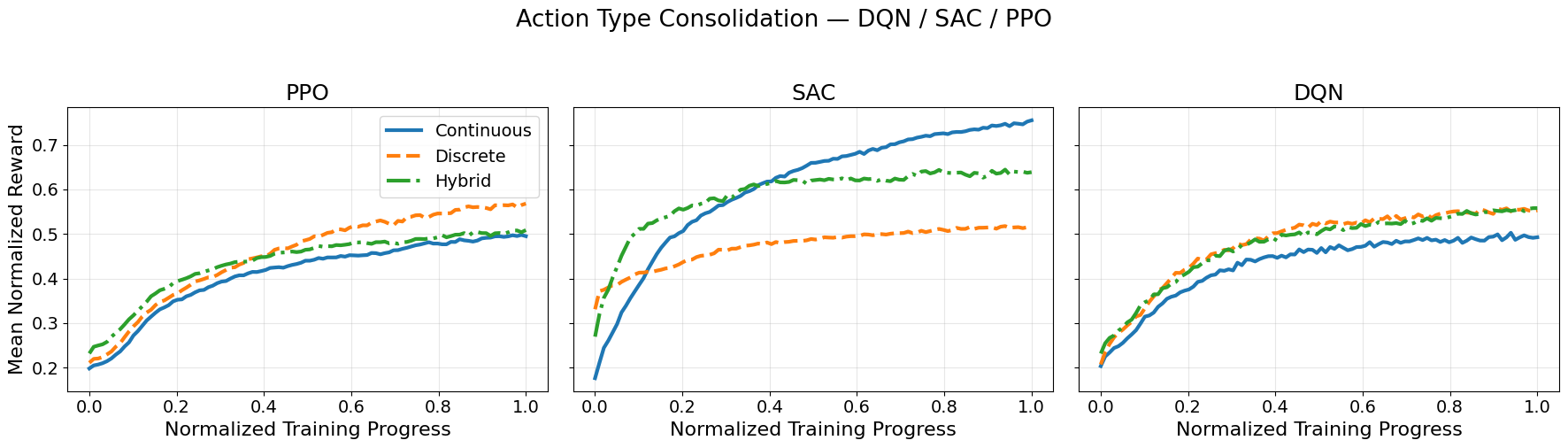}
    \caption{All datatype results aggregated by algorithm. SAC performs best with native continuous actions while PPO and DQN perform best with discrete, but the differences are small relative to factorization (Hybrid environments that got extra steps during training were truncated to make the comparison fair)}
    \label{fig:dtype}
\end{figure}

We see in Figures~\ref{fig:dtype} and~\ref{fig:allCombos} that action type has relatively less impact than factorization with the exception of SAC as discussed above. For off-policy learners, auto-regressive action offers the most universal solution to action factorization as it is able model dependence directly while scaling model parameters linearly with the number of actions. For on-policy, PPO with a factored critic provides superior performance at a significantly lessened computational cost relative to SAC/DQN and AR-PPO. Policy gradient methods are very sensitive to variance, so we hypothesize that the chain of normally distributed continuous action inputs offered more noise than signal. Conversely, variance reduction is the exact method which improves performance. 

\section{Discussion}

\subsection{Computational Complexity}

The computational footprint of a factorization method is a critical criterion for model selection, governed not only by theoretical operation counts but also by hardware-level parallelizability. Let $E$ denote the operation cost of the state encoder, $d$ the final hidden layer dimension, and $|\mathcal{A}_h|$ the cardinality of action head $h \in \{1, \dots, H\}$.

The theoretical forward-pass complexities scale as follows:
\begin{itemize}
    \item \textbf{Joint:} $O\big(E + d \prod_{h=1}^H |\mathcal{A}_h|\big)$. The final fully connected layer scales exponentially with the number of sub-actions.
    \item \textbf{Branching (Shared Encoder):} $O\big(E + d \sum_{h=1}^H |\mathcal{A}_h|\big)$. The exponential output requirement is reduced to a linear sum by attaching distinct linear heads to a single shared encoder pass.
    \item \textbf{Independent \& Auto-Regressive (AR):} $O\big(H \cdot E + d \sum_{h=1}^H |\mathcal{A}_h|\big)$. Because each action dimension utilizes a completely distinct network, the base state-encoding cost $E$ scales linearly with $H$.
\end{itemize}

However, theoretical complexity does not directly translate to wall-clock runtime due to architectural sequentiality. While Independent networks possess the highest total operation count, their decoupled forward passes can be batched and executed simultaneously on modern parallel hardware, effectively reducing inference latency to the speed of the slowest network: $\max_h O(E + d |\mathcal{A}_h|)$. Conversely, the chain-rule dependencies inherent in AR policies strictly prohibit parallel execution; head $h$ cannot be evaluated until head $h-1$ completes, forcing a sequential bottleneck bounded by the sum of all operations. 

Furthermore, the dense, monolithic matrix multiplications characteristic of Joint factorization are highly optimized on GPU architectures. Provided the exponentially scaled joint action space does not exceed VRAM limits, a single massive parallel operation can approximate constant time in practice. Consequently, despite requiring vastly more theoretical FLOPs, Joint architectures frequently outpace Branching and AR methods in wall-clock inference speed by avoiding kernel-launch overhead and sequential routing.

\section{Conclusion}

We presented a cross-sectional evaluation of six action factorization strategies across DQN, PPO, and SAC on discrete, continuous, and hybrid action spaces in over 220 configurations, alongside two novel lightweight benchmarks (CoopPush, Hybrid-Shoot) with tunable inter-action dependence. Auto-regressive actions offer the best action dependence performance across the board, at the cost of O(N) latency and total parameters. Shared encoder architectures offer the best compute to performance ratio in fully observable settings, where VDN in particular is appealing due to its trivial implementation effort. VDN-PPO and PPO-MIX delivers a substantial improvement over shared-encoder PPO in discrete spaces by redistributing variance to high-agency heads. In continuous action spaces the improvement is less significant. A notable direction for future work is equilibrium selection: both environments contain multiple symmetric optima that differ in exploration difficulty, which may differentially affect methods that select actions and explore using joint spaces compared to factored spaces. In the case of DQN, it can be seen that monotonic mixing networks can degrade performance as action spaces get larger ($11^3$ for shoot and $11^8$ for push). Investigating how these factorization representational capacity and stability interact with asynchronous parallelization frameworks (IMPALA~\cite{espeholt2018impala}, PQN~\cite{li2023pqn}) is another open question.

\section{Practitioner Recommendations}

Based on the results across all 220+ configurations, we offer the following model-selection recommendations: (i)~\textbf{Default to Branching Dueling / VDN, especially for Policy Gradient methods} as it requires minimal implementation overhead and computational cost and it performs competitively across all three algorithm families. (ii) ~\textbf{Auto-Regressive} actions are the most generally performant method if the computational overhead and latency is affordable. It is an open question whether action ordering has a strong impact in general. (iii)~\textbf{Avoid Q-PLEX/PPO-MIX} when the action space or number of actions grow large. Despite the highest performance on several configurations, DQN shows clear degradation when moving from smaller discrete action spaces size \ref{fig:allCombos}. Representational capacity or overestimation bias are both potential candidate causes for the degraded ability to accurately assign credit. It is very possible that scaling the monotonic mixer parameter count with the action space size would alleviate some of the degradation, but the concatenated action space is still exponential in the number of agents so future work is required to study mixing-scaling at length. Implementation overhead and added compute may not be worth it with multiple subtle implementation pitfalls (such as max centering being required for identifiability instead of the more stable mean-centered approach for dueling implementations, Elu vs Abs optimizer stability, and mixing network size as an additional hyper parameter). (iv)~\textbf{Concat vs BDQ SAC} depend heavily on the environment which will perform better. In our test there is not a clear default parameterization for hybrid spaces and the separate continuous and discrete entropy must be handled with care. 

We hope that this article serves two primary purposes. 1: The introduction and application of a novel variance reduction algorithm for policy gradient methods which trains a branching dueling critic on value-only targets with sub-action GAE importance weighting to reduce monte carlo return variance. 2: A Cross-sectional comparison over standard (AR, Concat-SAC, BD-DQN) and lesser known (BDQ/P-SAC, Single-Agent QPLEX, and AR-Hybrid-PPO) factorization strategies in a minimal but principled benchmark suite of dependent action environments. Future work will include scaling properties to significantly larger action spaces and environments with difficult exploration dynamics where branching or deterministic action selection may struggle relative to joint or auto-regressive learners. We hope to expand the list of benchmarks and supported algorithms to provide a consolidated selection guide for practitioners facing real world problems that are naturally represented with both continuous and discrete components.

\bibliographystyle{ACM-Reference-Format}
\bibliography{refs}

\appendix

\newpage

\section{Theoretical Analysis of Importance-Weighted GAE\\
for Multi-Head Branching Dueling Q-Network Critics}
\label{appendix:theory}

\subsection{Markov Decision Process}

Let $(\calS, \calA, P, r, \gamma)$ denote a discounted MDP where $\calS$ is the state space, $\calA = \prod_{h=1}^{H} \calA_h$ is the joint action space factored over $H$ action heads, $P: \calS \times \calA \to \Delta(\calS)$ is the transition kernel, $r: \calS \times \calA \to \R$ is the scalar reward function, and $\gamma \in [0,1)$ is the discount factor.

The policy $\pi$ is factored as:
\begin{equation}
    \pi(\bfa \mid s) = \prod_{h=1}^{H} \pi_h(a_h \mid s),
    \label{eq:factored_policy}
\end{equation}
where each head $\pi_h: \calS \to \Delta(\calA_h)$ is parameterised by $\theta_h$ and is conditionally independent of all other heads given the state $s$. This conditional independence is a \emph{design constraint} of the branching architecture (BDQ), not an assumption imposed on the environment.

\subsection{Branching Dueling Q-Network (BDQ) Critic}

\begin{definition}[BDQ Value Decomposition]
The BDQ critic represents the joint state-action value as:
\begin{equation}
    Q_\phi(s, \bfa) = V_\phi(s) + \sum_{h=1}^{H} A_{\phi,h}(s, a_h),
    \label{eq:bdq_decomp}
\end{equation}
where $V_\phi: \calS \to \R$ is the shared state-value stream and $A_{\phi,h}: \calS \times \calA_h \to \R$ is the advantage stream for head $h$, both parameterised by $\phi$.
\end{definition}

To ensure identifiability of $V_\phi$ and each $A_{\phi,h}$, a zero-mean normalisation must be enforced with respect to the \emph{current policy}:
\begin{equation}
    \E_{a_h \sim \pi_h(\cdot \mid s)}\!\left[A_{\phi,h}(s, a_h)\right] = 0 \quad \forall\, s \in \calS,\; h \in \calH.
    \label{eq:advantage_normalisation}
\end{equation}
Equation~\eqref{eq:advantage_normalisation} ensures that the value function $V_\phi(s)$ strictly targets the on-policy state value $V^\pi(s)$ without absorbing expected policy advantages.

\subsection{Per-Head Importance Measure and Weights}

\begin{definition}[Per-Head Importance and Annealed Weights]
For state $s$ and action head $h$, define the importance as the range of the advantage function:
\begin{equation}
    I_h(s) = \max_{a_h \in \calA_h} A_{\phi,h}(s, a_h) - \min_{a_h \in \calA_h} A_{\phi,h}(s, a_h) \geq 0.
\end{equation}
Given an annealing parameter $\alpha \in [0, 1]$, the importance weight is:
\begin{equation}
    w_h^\alpha(s) = \frac{I_h(s)^\alpha}{\sum_{h'=1}^{H} I_{h'}(s)^\alpha},
    \label{eq:weights}
\end{equation}
where $0^0 = 1$ enforces uniform weights $w_h^0(s) = 1/H$ at $\alpha = 0$.
\end{definition}

\subsection{Critic Optimisation and TD Error Structure}

\subsubsection{The $V$-only Target and Critic Loss}

In environments with highly factored action spaces, using the full state-action value $Q(s_{t+1}, \bfa_{t+1})$ in the regression target introduces severe variance due to noisy advantage estimation. To stabilise learning and ensure convergence to the on-policy state value, the BDQ critic is trained using a $V$-only target:
\begin{equation}
    y_t = r_t + \gamma V_\phi(s_{t+1}).
\end{equation}
Under this target formulation, the regression loss for the parameterised critic $Q_\phi = V_\phi + \sum_h A_{\phi,h}$ is:
\begin{align}
    \mathcal{L}(\phi) 
    &= \frac{1}{2} \E_{\tau \sim \pi} \!\left[ \left( Q_\phi(s_t, \bfa_t) - y_t \right)^2 \right] \nonumber \\
    &= \frac{1}{2} \E_{\tau \sim \pi} \!\left[ \left( V_\phi(s_t) + \sum_{h=1}^H A_{\phi,h}(s_t, a_{h,t}) - \left(r_t + \gamma V_\phi(s_{t+1})\right) \right)^2 \right].
    \label{eq:critic_loss_full}
\end{align}

\subsubsection{Temporal Difference Residuals}

Given a trajectory $\tau = (s_0, \bfa_0, r_0, \dots)$, define the one-step $V$-based TD residual:
\begin{equation}
    \delta_t^V = r_t + \gamma V_\phi(s_{t+1}) - V_\phi(s_t).
    \label{eq:td_residual}
\end{equation}

Rearranging Equation~\eqref{eq:critic_loss_full} reveals that optimising the BDQ architecture with a $V$-only target is mathematically equivalent to fitting the sum of the per-head advantages to the TD residual:
\begin{equation}
    \mathcal{L}(\phi) = \frac{1}{2} \E_{\tau \sim \pi} \!\left[ \left( \sum_{h=1}^H A_{\phi,h}(s_t, a_{h,t}) - \delta_t^V \right)^2 \right].
\end{equation}
This explicitly bridges the architecture's loss function with the theoretical construct of Generalized Advantage Estimation (GAE).

\begin{lemma}[Unbiasedness of the TD Residual]
\label{lem:td_unbiased}
Let $\calF_t = \sigma(s_0, \bfa_0, \dots, s_t, \bfa_t)$ be the filtration up to action $\bfa_t$. Let $A^\pi(s, \bfa) = Q^\pi(s, \bfa) - V^\pi(s)$. Then:
\begin{equation}
    \E\!\left[\delta_t^V \;\middle|\; \calF_t\right] = A^\pi(s_t, \bfa_t) + \gamma\!\left(V^\pi(s_{t+1}) - V_\phi(s_{t+1})\right) - \left(V^\pi(s_t) - V_\phi(s_t)\right).
\end{equation}
When the critic has converged ($V_\phi = V^\pi$), $\E[\delta_t^V \mid \calF_t] = A^\pi(s_t, \bfa_t)$.
\end{lemma}

\subsection{The Modified Per-Head Estimator}

\begin{definition}[Importance-Weighted Per-Head GAE]
\label{def:iw_gae}
For each action head $h \in \calH$, the importance-weighted GAE is:
\begin{equation}
    \hat{A}_t^{(h)} = \sum_{l=0}^{T-t-1} (\gamma\lambda)^l \, w_h^\alpha(s_{t+l}) \cdot \delta_{t+l}^V.
    \label{eq:iw_gae_def}
\end{equation}
The weights $w_h^\alpha(s_{t+l})$ are computed using the \textbf{current} parameter values $\phi$ and are treated as fixed scalars (no gradient flows through them). Crucially, $w_h^\alpha(s_{t+l})$ depends only on the state $s_{t+l}$, not the action sampled at $t+l$.
\end{definition}

\subsection{Formal Assumptions}

\begin{assumption}[Factored Policy Independence]
\label{ass:factored_independence}
The policy satisfies \eqref{eq:factored_policy}: $\pi(\bfa \mid s) = \prod_{h=1}^H \pi_h(a_h \mid s)$.
\end{assumption}

\begin{assumption}[On-Policy Advantage Normalisation]
\label{ass:normalisation}
The critic ensures zero-mean advantage under the \emph{current} policy distribution $\pi$: $\E_{a_h \sim \pi_h}[A_{\phi,h}(s, a_h)] = 0$.
\end{assumption}

\begin{assumption}[Critic Convergence]
\label{ass:convergence}
For the primary bias analysis, $V_\phi(s) = V^\pi(s)$ for all $s \in \calS$. 
\end{assumption}

\begin{assumption}[Non-Degeneracy and Stop-Gradient]
\label{ass:nondegenerate}
At every state $s$, $\sum_{h=1}^H I_h(s) > 0$. The weights $w_h^\alpha$ are treated as constants with respect to policy parameters $\theta$.
\end{assumption}

\subsection{Main Results}

\subsubsection{Unbiasedness of the Estimator}

\begin{theorem}[Unbiasedness of the Weighted Advantage]
\label{thm:unbiased}
Under Assumptions~\ref{ass:factored_independence}--\ref{ass:convergence}, the modified estimator $\hat{A}_t^{(h)}$ satisfies:
\begin{equation}
    \E_{\tau \sim \pi}\!\left[\hat{A}_t^{(h)} \;\middle|\; s_t, \bfa_t\right] = \E_{\tau \sim \pi}\!\left[\sum_{l=0}^{T-t-1} (\gamma\lambda)^l \, w_h^\alpha(s_{t+l}) \cdot A^\pi(s_{t+l}, \bfa_{t+l}) \;\middle|\; s_t, \bfa_t\right].
    \label{eq:unbiased_statement}
\end{equation}
That is, it is an unbiased estimator of the expected discounted sum of importance-weighted advantages along the trajectory.
\end{theorem}

\begin{proof}
\textbf{Step 1: Expand the expectation and define filtrations.}
\begin{equation}
    \E\!\left[\hat{A}_t^{(h)} \;\middle|\; s_t, \bfa_t\right]
    = \sum_{l=0}^{T-t-1} (\gamma\lambda)^l \; \E\!\left[w_h^\alpha(s_{t+l}) \cdot \delta_{t+l}^V \;\middle|\; s_t, \bfa_t\right].
\end{equation}
Define the extended filtration $\calG_{t+l} = \sigma(s_0, \bfa_0, \dots, s_{t+l}, \bfa_{t+l}, s_{t+l+1})$. Note that $\delta_{t+l}^V$ is $\calG_{t+l}$-measurable, while $w_h^\alpha(s_{t+l})$ is strictly $\sigma(s_{t+l})$-measurable.

\medskip
\textbf{Step 2: Factor using the tower property.}
Apply the tower property conditioning on the state $s_{t+l}$ and action $\bfa_{t+l}$:
\begin{align}
    \E\!\left[w_h^\alpha(s_{t+l}) \delta_{t+l}^V \;\middle|\; s_t, \bfa_t\right] 
    &= \E\!\left[\E\!\left[w_h^\alpha(s_{t+l}) \delta_{t+l}^V \;\middle|\; s_{t+l}, \bfa_{t+l}\right] \;\middle|\; s_t, \bfa_t\right].
\end{align}

\medskip
\textbf{Step 3: Measurability of importance weights.}
Because $w_h^\alpha(s_{t+l})$ is a deterministic function of the state $s_{t+l}$ alone, it is fully determined given the inner conditioning. It can be factored out:
\begin{align}
    \E\!\left[w_h^\alpha(s_{t+l}) \delta_{t+l}^V \;\middle|\; s_{t+l}, \bfa_{t+l}\right]
    &= w_h^\alpha(s_{t+l}) \cdot \E\!\left[\delta_{t+l}^V \;\middle|\; s_{t+l}, \bfa_{t+l}\right].
\end{align}

\medskip
\textbf{Step 4: Invoke Lemma~\ref{lem:td_unbiased}.}
Under Assumption~\ref{ass:convergence} ($V_\phi = V^\pi$), Lemma~\ref{lem:td_unbiased} gives $\E[\delta_{t+l}^V \mid s_{t+l}, \bfa_{t+l}] = A^\pi(s_{t+l}, \bfa_{t+l})$. Substituting this back completes the proof:
\begin{equation}
    \E\!\left[\hat{A}_t^{(h)} \;\middle|\; s_t, \bfa_t\right]
    = \E\!\left[\sum_{l=0}^{T-t-1} (\gamma\lambda)^l \, w_h^\alpha(s_{t+l}) \cdot A^\pi(s_{t+l}, \bfa_{t+l}) \;\middle|\; s_t, \bfa_t\right].
\end{equation}
\end{proof}

\begin{corollary}[Sum Recovers Standard GAE]
\label{cor:sum}
Because $\sum_{h=1}^H w_h^\alpha(s) = 1$ identically for all states $s$:
\begin{equation}
    \sum_{h=1}^H \E\!\left[\hat{A}_t^{(h)} \;\middle|\; s_t, \bfa_t\right]
    = \E\!\left[\sum_{l=0}^{T-t-1} (\gamma\lambda)^l A^\pi(s_{t+l}, \bfa_{t+l}) \;\middle|\; s_t, \bfa_t\right]
    = \E\!\left[\hat{A}_t^{\mathrm{GAE}} \;\middle|\; s_t, \bfa_t\right].
\end{equation}
The global policy gradient remains unbiased.
\end{corollary}

\subsubsection{Validity of the Policy Gradient}

\begin{theorem}[Valid Policy Gradient Direction]
\label{thm:policy_gradient}
Under Assumptions~\ref{ass:factored_independence}--\ref{ass:nondegenerate}, the per-head gradient estimator $\hat{g}_h = \hat{A}_t^{(h)} \nabla_{\theta_h} \log \pi_h(a_{t,h} \mid s_t)$ satisfies:
\begin{equation}
    \mathrm{sign}\!\left(\E[\hat{g}_h]^\top \nabla_{\theta_h} J\right) > 0.
\end{equation}
\end{theorem}
\begin{proof}
By the stop-gradient constraint (Assumption~\ref{ass:nondegenerate}), the weights do not interfere with the score function. The expected update is a positive, state-dependent rescaling of the true gradient, guaranteeing ascent in expectation.
\end{proof}

\subsection{Variance Reduction Analysis}

\begin{proposition}[Monotonic Variance Reduction]
\label{prop:variance_reduction}
Under Assumptions~\ref{ass:factored_independence}--\ref{ass:convergence}, the per-step contribution to estimator variance from head $h$ is strictly reduced for low-importance heads relative to uniform weighting ($\alpha=0$).
\begin{equation}
    (w_h^\alpha(s))^2 \leq \frac{1}{H^2} \quad \text{for heads with } I_h(s) \leq \frac{1}{H}\sum_{j=1}^H I_j(s).
\end{equation}
\end{proposition}

\begin{proof}
For any $\alpha > 0$, the function $x \mapsto x^\alpha / \sum x_j^\alpha$ is monotonic with respect to $x$. Therefore, if the advantage range $I_h(s)$ is below the arithmetic mean importance $\bar{I}(s) = \frac{1}{H}\sum I_j(s)$, its corresponding weight $w_h^\alpha(s)$ is strictly less than its uniform share $1/H$. Since the variance of the estimator sum scales with $(w_h^\alpha(s))^2 \Var[\delta^V \mid s]$, the variance injected by uninformative heads is systematically suppressed.
\end{proof}

\subsection{Finite-Sample and Non-Converged Critic Analysis}
\label{sec:finite_sample}

\begin{proposition}[Bias Under Critic Error]
\label{prop:bias_critic_error}
Define the finite-sample critic error as $\epsilon_\phi(s) = V_\phi(s) - V^\pi(s)$. Because the target is $V$-only, the bias of the modified estimator is precisely:
\begin{align}
    \mathrm{Bias}\!\left[\hat{A}_t^{(h)}\right]
    = \sum_{l=0}^{T-t-1} (\gamma\lambda)^l \; \E\!\left[w_h^\alpha(s_{t+l}) \left(\gamma \epsilon_\phi(s_{t+l+1}) - \epsilon_\phi(s_{t+l})\right)\right].
    \label{eq:bias_critic_error}
\end{align}
Since $w_h^\alpha(s) \in (0,1)$, the magnitude of the bias contribution from noisy heads is attenuated relative to standard unweighted GAE.
\end{proposition}

\section{Advantage Centering in Monotonic Mixing Architectures}
\label{appendix:centering}

In factored dueling architectures such as those used for DQN and SAC, the total action-value $Q_{tot}$ is decomposed into a global state-value $V(s)$ and a set of agent-specific advantages $A_h(s, a_h)$, recombined via a monotonic mixer $f_{mix}$:

\begin{equation}
Q_{tot}(s, \mathbf{a}) = V(s) + f_{mix}(A_1(s, a_1), \dots, A_n(s, a_n); s)
\end{equation}

While this decomposition provides a powerful inductive bias, it introduces a critical \textit{identifiability crisis} regarding the magnitude of the advantage terms. When using a monotonic mixer (e.g., QPLEX), the requirement for \textbf{max-centering} ($\max_a A(s, a) = 0$) becomes a structural necessity rather than a convention.

\subsection{The Threat of Gain Inflation}
In Q-PLEX mode, the mixer $f_{mix}$ is parameterized by a hypernetwork that generates non-negative weights $w_h(s) \geq 0$ on the advantages. This ensures the monotonicity constraint $\frac{\partial Q_{tot}}{\partial A_h} \geq 0$. However, because these weights are state-dependent and effectively unbounded, the network can minimize TD error by arbitrarily ``inflating'' the gain of the mixer.

If the advantages are not strictly grounded (e.g., if we use mean-centering or no centering), a positive feedback loop can emerge during bootstrapping:
\begin{enumerate}
    \item The value head $V(s)$ drifts (e.g., decreases).
    \item To maintain the target $Q$, the hypernetwork increases the weights $w_h$.
    \item The larger $w_h$ ``amplifies'' the advantage signal, making $Q_{tot}$ highly sensitive to small changes in $A_h$.
    \item This increased gain drives larger gradients back into the encoder, further ratcheting the weights and leading to a \textbf{bootstrapped divergence} (loss and weights exploding toward infinity).
\end{enumerate}

\subsection{Max-Centering as a Structural Clamp}
Max-centering ($A = A - \max A$) enforces the constraint $A(s, a) \leq 0$ for all actions. This grounding serves as a ``structural clamp'' in two ways:

\begin{itemize}
    \item \textbf{Fixed Upper Bound:} By forcing the maximum advantage to be zero, the term $f_{mix}(0, \dots, 0; s)$ becomes the fixed upper bound of the mixer's contribution. If the mixer is implemented without state-dependent biases (as in the \texttt{QS} class when \texttt{dueling=True}), then $f_{mix}(\mathbf{0}; s) = 0$.
    \item \textbf{Decoupling V from Gain:} With $\max A = 0$, the identity $Q_{tot}(s, \mathbf{a}^*) = V(s)$ is strictly enforced for the optimal joint action $\mathbf{a}^*$. This prevents the hypernetwork from using gain inflation to ``swallow'' the state-value $V(s)$ into the advantage terms. The only way to increase the total Q-value of the optimal action is to increase $V(s)$ directly, which is subject to the standard TD backup and is not amplified by the mixer's weights.
\end{itemize}

\subsection{Mathematical Identifiability in SAC and DQN}
For off-policy algorithms like DQN and SAC, the objective is to estimate the optimal value $V^*(s) = \max_a Q(s, a)$. Max-centering is the only constraint that aligns the internal state-value head $V(s)$ with this mathematical definition. 

Without max-centering, the mixer's asymmetry (often caused by activations like \texttt{ReLU} or \texttt{ELU}) can lead to a ``representational gap'': the mixer can represent large positive contributions easily but struggles to represent very negative values. Max-centering forces all sub-optimal actions to be represented as negative offsets from $V(s)$. This ensures that the ``noise'' or ``drift'' from sub-optimal actions does not leak into the estimation of the state value, preserving the stability of the greedy policy.


\newpage


\section{Hyperparameter Configurations}
\label{appendix:hyperparameters}

The complete list of hyperparameter settings across all reinforcement learning evaluations is detailed below.

\begin{longtable}{@{} M{2.0cm} P{2.2cm} l r P{5.0cm} @{}}
\caption{Hyperparameter configurations for all evaluation environments.} \label{tab:hyperparameters} \\
\toprule
\textbf{Family} & \textbf{Environment} & \textbf{Algo} & \textbf{Max Steps} & \textbf{Algorithm-Specific Parameters} \\
\midrule
\endfirsthead

\multicolumn{5}{c}{{\bfseries Table \thetable\ \textit{(Continued)}}} \\
\toprule
\textbf{Family} & \textbf{Environment} & \textbf{Algo} & \textbf{Max Steps} & \textbf{Algorithm-Specific Parameters} \\
\midrule
\endhead

\bottomrule
\multicolumn{5}{r}{\textit{Continued on next page}} \\
\endfoot
\bottomrule
\endlastfoot

\multirow{9}{*}{\parbox{2.0cm}{Dependent\\Push\\\small(Hidden:\\{[256, 256]})}} 
  & Discrete & PPO & 2,000,000 & Buffer: 4096, Batch: 128, Epochs: 4, $\gamma=0.98$, Clip: 0.2, $\lambda_{\text{GAE}}=0.95$ \\
  && SAC & 1,000,000 & Replay: 50k, Batch: 128, $\gamma=0.98$, $\tau_{\text{SAC}}=0.01$, Ent\%: 0.5 \\
  && DQN & 1,000,000 & Replay: 50k, Batch: 128, $\gamma=0.98$, $\tau=0.01$, Update: 4 \\ \nopagebreak \cline{2-5}
  & Continuous & PPO & 2,000,000 & Buffer: 4096, Batch: 128, Epochs: 4, $\gamma=0.98$, Clip: 0.2, $\lambda_{\text{GAE}}=0.95$ \\
  && SAC & 1,000,000 & Replay: 50k, Batch: 256, $\gamma=0.98$, $\tau_{\text{SAC}}=0.01$, Ent\%: 0.5 \\
  && DQN & 1,000,000 & Replay: 50k, Batch: 256, $\gamma=0.98$, $\tau=0.01$, Update: 4 \\ \nopagebreak \cline{2-5}
  & Hybrid & PPO & 3,000,000 & Buffer: 4096, Batch: 128, Epochs: 5, $\gamma=0.98$, Clip: 0.2, $\lambda_{\text{GAE}}=0.95$ \\
  && SAC & 2,000,000 & Replay: 50k, Batch: 256, $\gamma=0.98$, $\tau_{\text{SAC}}=0.01$, Ent\%: 0.5 \\
  && DQN & 2,000,000 & Replay: 50k, Batch: 256, $\gamma=0.98$, $\tau=0.01$, Update: 4 \\ 
\midrule

\multirow{9}{*}{\parbox{2.0cm}{Lander\\\small(Hidden Dims:\\{[128, 128]})}} 
  & Discrete & PPO &   600,000 & Buffer: 2048, Batch: 64, Epochs: 4, $\gamma=0.99$, Clip: 0.2, $\lambda_{\text{GAE}}=0.95$ \\
  && SAC &   400,000 & Replay: 50k, Batch: 128, $\gamma=0.99$, $\tau_{\text{SAC}}=0.01$, Ent\%: 0.5 \\
  && DQN &   400,000 & Replay: 50k, Batch: 128, $\gamma=0.99$, $\tau=0.01$, Update: 4 \\ \nopagebreak \cline{2-5}
  & Continuous & PPO &   600,000 & Buffer: 2048, Batch: 64, Epochs: 4, $\gamma=0.99$, Clip: 0.2, $\lambda_{\text{GAE}}=0.95$ \\
  && SAC &   400,000 & Replay: 50k, Batch: 128, $\gamma=0.99$, $\tau_{\text{SAC}}=0.01$, Ent\%: 0.5 \\
  && DQN &   400,000 & Replay: 50k, Batch: 128, $\gamma=0.99$, $\tau=0.01$, Update: 4 \\ \nopagebreak \cline{2-5}
  & Hybrid & PPO & 1,000,000 & Buffer: 2048, Batch: 64, Epochs: 4, $\gamma=0.99$, Clip: 0.2, $\lambda_{\text{GAE}}=0.95$ \\
  && SAC &   600,000 & Replay: 50k, Batch: 128, $\gamma=0.99$, $\tau_{\text{SAC}}=0.01$, Ent\%: 0.5 \\
  && DQN &   600,000 & Replay: 50k, Batch: 128, $\gamma=0.99$, $\tau=0.01$, Update: 4 \\ 
\midrule

\multirow{18}{*}{\parbox{2.0cm}{Shoot\\\small(Hidden Dims:\\{[256, 256]})}} 
  & Dep. Hybrid & PPO & 1,500,000 & Buffer: 4096, Batch: 128, Epochs: 5, $\gamma=0.99$, Clip: 0.2, $\lambda_{\text{GAE}}=0.95$ \\
  && SAC &   800,000 & Replay: 100k, Batch: 256, $\gamma=0.99$, $\tau_{\text{SAC}}=0.01$, Ent\%: 0.5 \\
  && DQN &   800,000 & Replay: 100k, Batch: 256, $\gamma=0.99$, $\tau=0.01$, Update: 4 \\ \nopagebreak \cline{2-5}
  & Indep. Hybrid & PPO & 1,500,000 & Buffer: 4096, Batch: 128, Epochs: 5, $\gamma=0.99$, Clip: 0.2, $\lambda_{\text{GAE}}=0.95$ \\
  && SAC &   800,000 & Replay: 100k, Batch: 256, $\gamma=0.99$, $\tau_{\text{SAC}}=0.01$, Ent\%: 0.5 \\
  && DQN &   800,000 & Replay: 100k, Batch: 256, $\gamma=0.99$, $\tau=0.01$, Update: 4 \\ \nopagebreak \cline{2-5}
  & Dep. Discrete & PPO & 1,500,000 & Buffer: 4096, Batch: 128, Epochs: 4, $\gamma=0.99$, Clip: 0.2, $\lambda_{\text{GAE}}=0.95$ \\
  && SAC &   800,000 & Replay: 100k, Batch: 256, $\gamma=0.99$, $\tau_{\text{SAC}}=0.01$, Ent\%: 0.5 \\
  && DQN &   800,000 & Replay: 100k, Batch: 256, $\gamma=0.99$, $\tau=0.01$, Update: 4 \\ \nopagebreak \cline{2-5}
  & Indep. Discrete & PPO & 1,500,000 & Buffer: 4096, Batch: 128, Epochs: 4, $\gamma=0.99$, Clip: 0.2, $\lambda_{\text{GAE}}=0.95$ \\
  && SAC &   800,000 & Replay: 100k, Batch: 256, $\gamma=0.99$, $\tau_{\text{SAC}}=0.01$, Ent\%: 0.5 \\
  && DQN &   800,000 & Replay: 100k, Batch: 256, $\gamma=0.99$, $\tau=0.01$, Update: 4 \\ \nopagebreak \cline{2-5}
  & Dep. Continuous & PPO & 1,500,000 & Buffer: 4096, Batch: 128, Epochs: 4, $\gamma=0.99$, Clip: 0.2, $\lambda_{\text{GAE}}=0.95$ \\
  && SAC &   800,000 & Replay: 100k, Batch: 256, $\gamma=0.99$, $\tau_{\text{SAC}}=0.01$, Ent\%: 0.5 \\
  && DQN &   800,000 & Replay: 100k, Batch: 256, $\gamma=0.99$, $\tau=0.01$, Update: 4 \\ \nopagebreak \cline{2-5}
  & Indep. Continuous & PPO & 1,500,000 & Buffer: 4096, Batch: 128, Epochs: 4, $\gamma=0.99$, Clip: 0.2, $\lambda_{\text{GAE}}=0.95$ \\
  && SAC &   800,000 & Replay: 100k, Batch: 256, $\gamma=0.99$, $\tau_{\text{SAC}}=0.01$, Ent\%: 0.5 \\
  && DQN &   800,000 & Replay: 100k, Batch: 256, $\gamma=0.99$, $\tau=0.01$, Update: 4 \\ 
\midrule

\multirow{9}{*}{\parbox{2.0cm}{Platform\\\small(Hidden Dims:\\{[256, 256]})}} 
  & Discrete & PPO &   500,000 & Buffer: 4096, Batch: 128, Epochs: 4, $\gamma=0.99$, Clip: 0.2, $\lambda_{\text{GAE}}=0.95$ \\
  && SAC &   200,000 & Replay: 100k, Batch: 256, $\gamma=0.99$, $\tau_{\text{SAC}}=0.01$, Ent\%: 0.5 \\
  && DQN &   200,000 & Replay: 100k, Batch: 256, $\gamma=0.99$, $\tau=0.01$, Update: 4 \\ \nopagebreak \cline{2-5}
  & Continuous & PPO &   500,000 & Buffer: 4096, Batch: 128, Epochs: 4, $\gamma=0.99$, Clip: 0.2, $\lambda_{\text{GAE}}=0.95$ \\
  && SAC &   200,000 & Replay: 100k, Batch: 256, $\gamma=0.99$, $\tau_{\text{SAC}}=0.01$, Ent\%: 0.5 \\
  && DQN &   200,000 & Replay: 100k, Batch: 256, $\gamma=0.99$, $\tau=0.01$, Update: 4 \\ \nopagebreak \cline{2-5}
  & Hybrid & PPO &   500,000 & Buffer: 4096, Batch: 128, Epochs: 5, $\gamma=0.99$, Clip: 0.2, $\lambda_{\text{GAE}}=0.95$ \\
  && SAC &   200,000 & Replay: 100k, Batch: 256, $\gamma=0.99$, $\tau_{\text{SAC}}=0.01$, Ent\%: 0.5 \\
  && DQN &   200,000 & Replay: 100k, Batch: 256, $\gamma=0.99$, $\tau=0.01$, Update: 4 \\ 
\end{longtable}

\noindent \textbf{Note:} All experimental runs globally shared a learning rate ($lr$) value of $5.0 \times 10^{-4}$.

\section{Additional Figures and Results}
\label{app:appendixD}

\subsection{Aggregated and Consolidated Results}
\begin{figure}[H]
    \centering
    \includegraphics[width=0.75\textwidth]{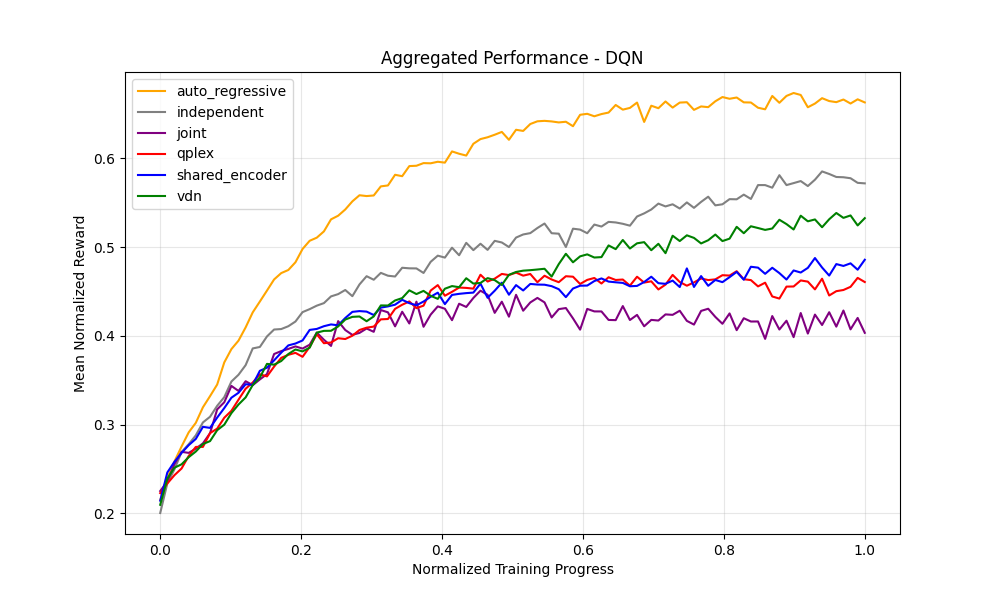}
    \caption{Across action types auto-regressive and independent actions perform best (Both because parameter count scales with action size and AR because dependence is properly accounted for)}
    \label{fig:agg_dqn}
\end{figure}

\begin{figure}[H]
    \centering
    \includegraphics[width=0.75\textwidth]{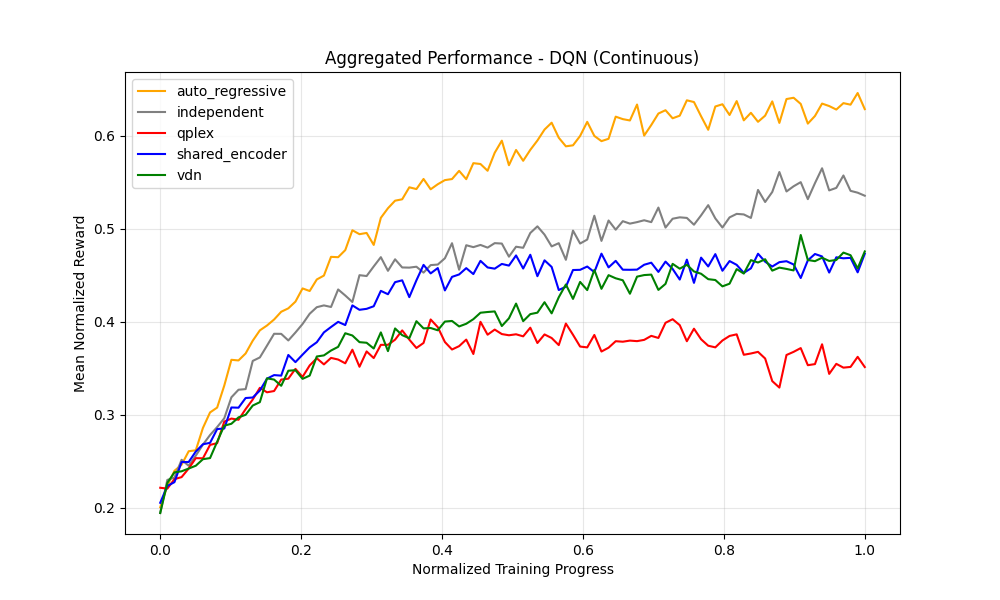}
    \caption{"Continuous" Action types only (11 buckets per action), zoomed in from the joint figure in the paper}
    \label{fig:agg_dqn_continuous}
\end{figure}

\begin{figure}[H]
    \centering
    \includegraphics[width=0.75\textwidth]{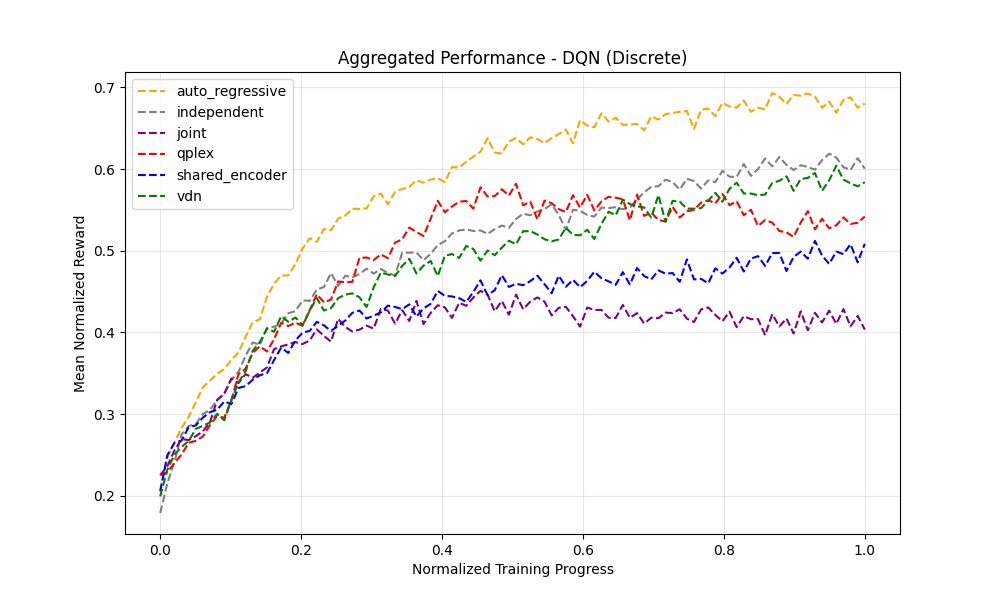}
    \caption{"Discrete" Action types only, zoomed in from the joint figure in the paper}
    \label{fig:agg_dqn_discrete}
\end{figure}

\begin{figure}[H]
    \centering
    \includegraphics[width=0.75\textwidth]{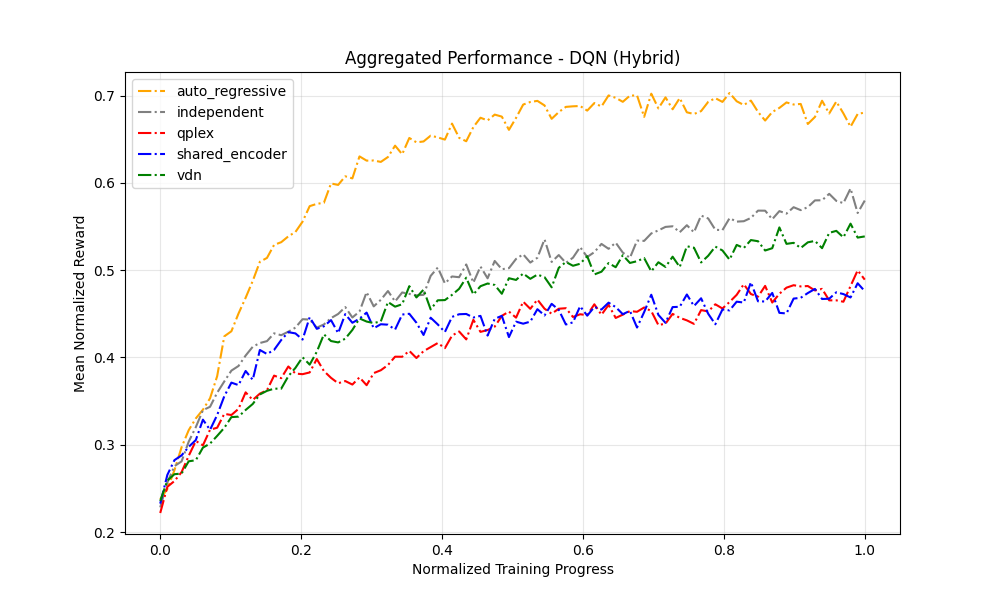}
    \caption{"Hybrid" Action types only, first half discrete second half continuous, zoomed in from the joint figure in the paper}
    \label{fig:agg_dqn_hybrid}
\end{figure}

\begin{figure}[H]
    \centering
    \includegraphics[width=0.75\textwidth]{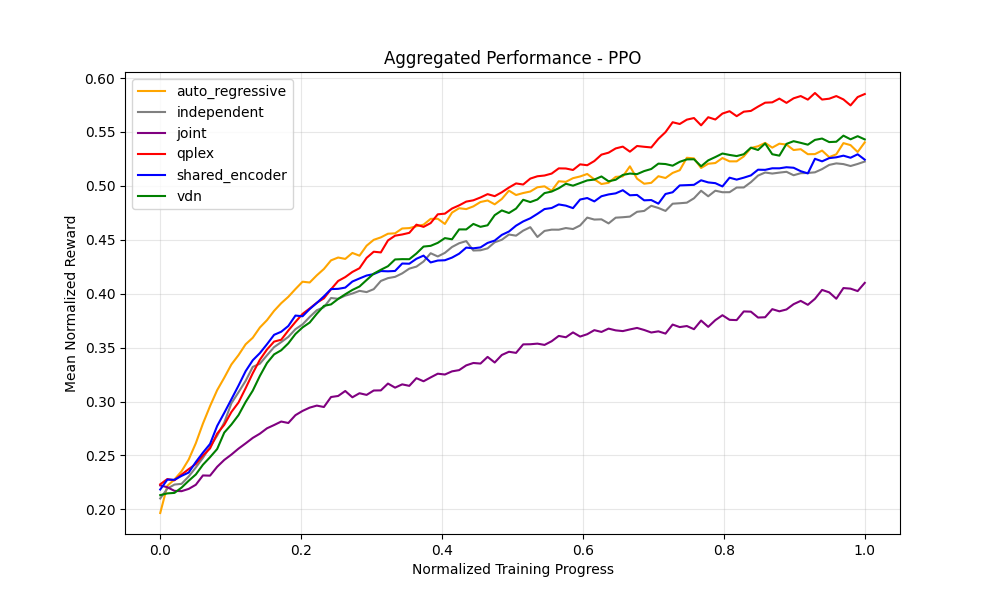}
    \caption{Overall, Q-PLEX performs the best when taking all action types into account, though only marginally. VDN is close behind but significantly simpler}
    \label{fig:agg_ppo}
\end{figure}

\begin{figure}[H]
    \centering
    \includegraphics[width=0.75\textwidth]{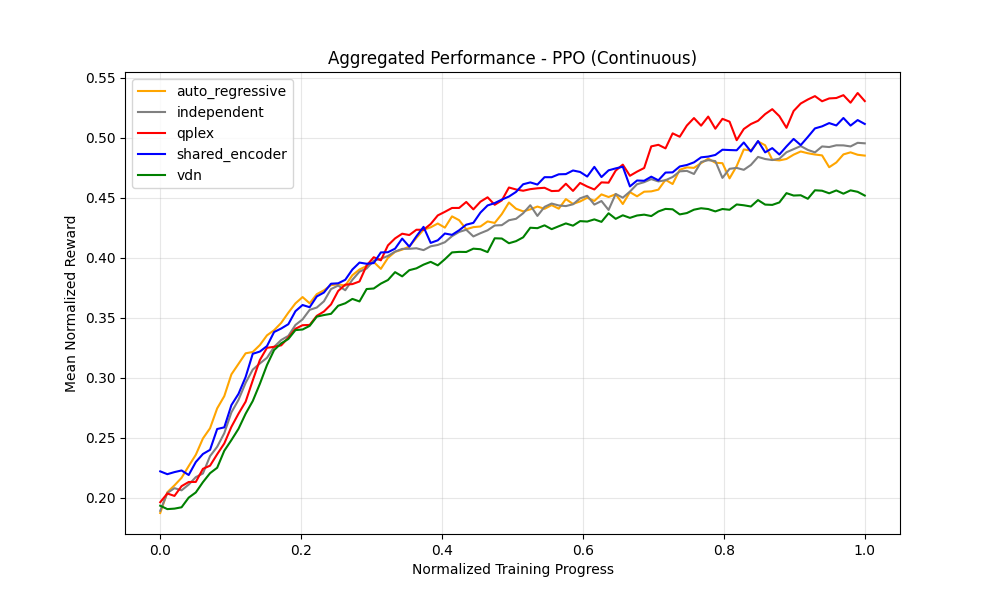}
    \caption{For continuous actions squashed gaussian PPO is stable with respect to factorization}
    \label{fig:agg_ppo_continuous}
\end{figure}

\begin{figure}[H]
    \centering
    \includegraphics[width=0.75\textwidth]{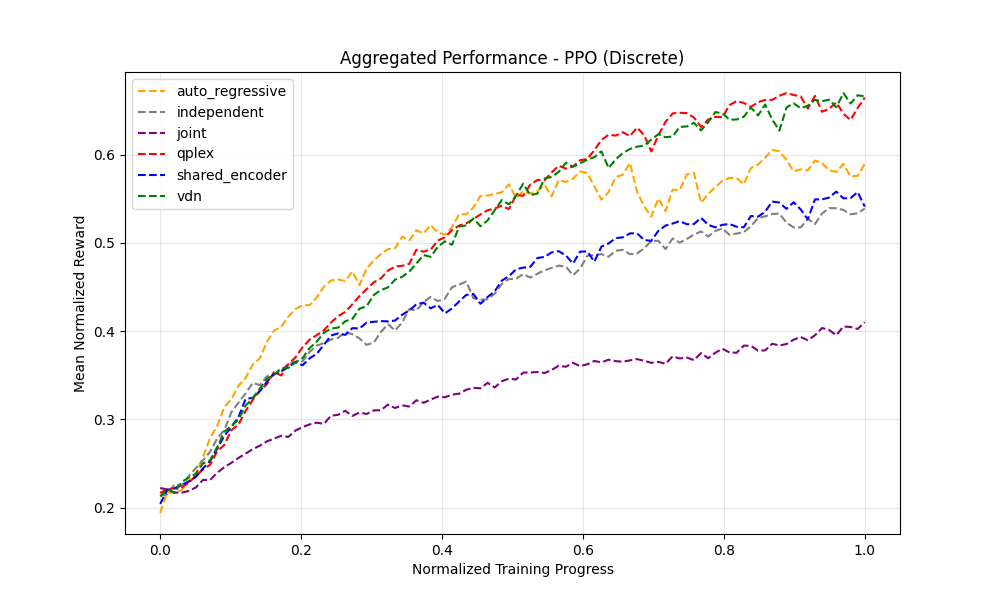}
    \caption{For discrete actions, VDN and Q-PLEX perform best. KL-divergence is applied to the joint space here so we hypothesize that importance weighting and reduced variance are major factors in "using up" the available divergence productively in the correct action heads.}
    \label{fig:agg_ppo_discrete}
\end{figure}

\begin{figure}[H]
    \centering
    \includegraphics[width=0.75\textwidth]{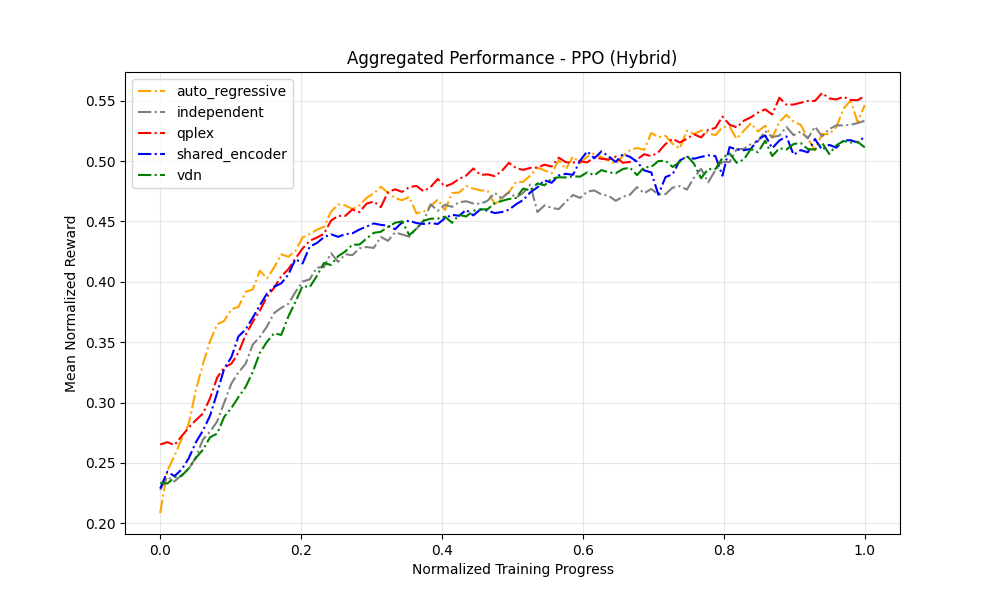}
    \caption{We are unsure what it is about continuous actions' inclusion that reduces the impact on factorization}
    \label{fig:agg_ppo_hybrid}
\end{figure}

\begin{figure}[H]
    \centering
    \includegraphics[width=0.75\textwidth]{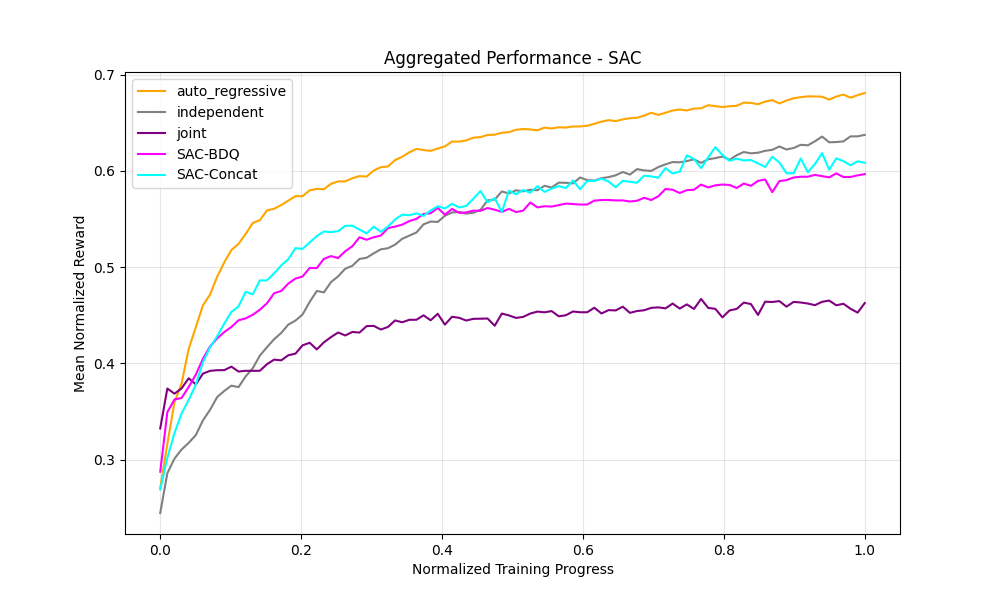}
    \caption{Joint discrete is the only true under-performer but AR-SAC is best on this bench suite. Discrete actions with an entropy auto-tuner are very sensitive to the level of entropy and large discrete action spaces overpower the continuous dims and env obs influence on the critic}
    \label{fig:agg_sac}
\end{figure}

\begin{figure}[H]
    \centering
    \includegraphics[width=0.75\textwidth]{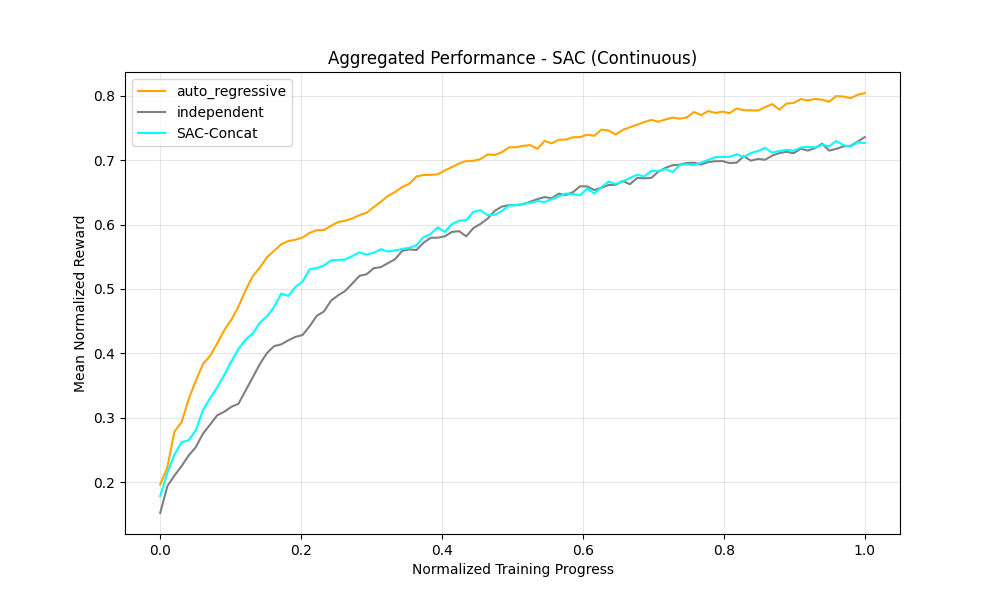}
    \caption{The best performing algorithm overall on all environments}
    \label{fig:agg_sac_continuous}
\end{figure}

\begin{figure}[H]
    \centering
    \includegraphics[width=0.75\textwidth]{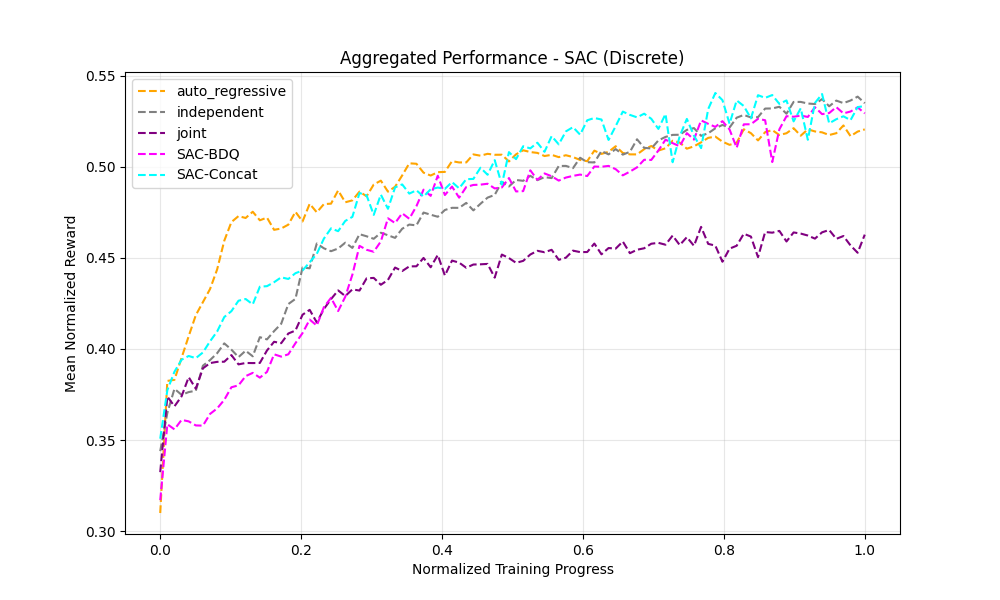}
    \caption{Only Joint underperforms as in the discussion in Figure \ref{fig:agg_sac}}
    \label{fig:agg_sac_discrete}
\end{figure}

\begin{figure}[H]
    \centering
    \includegraphics[width=0.75\textwidth]{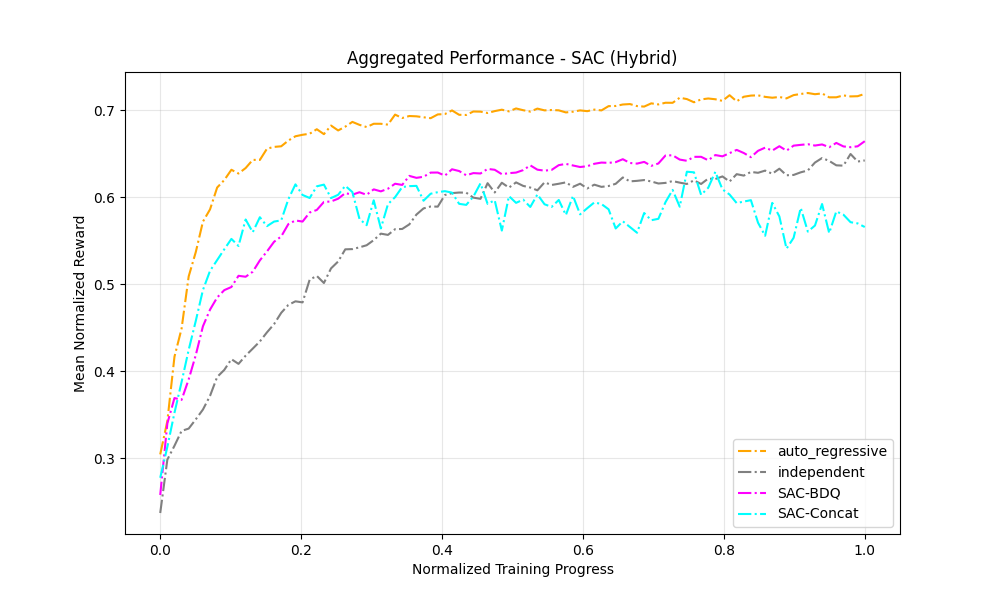}
    \caption{Auto-regressive sac in the hybrid space uses standard SAC for continuous dims and branching dueling D-SAC or Discrete-SAC for the discrete dims. With proper tuning AR is the most natural mapping. BDQ introduces the overestimation bias of Q learning into the continuous actor gradient but it does factor the environment more naturally. The best choice is unclear}
    \label{fig:agg_sac_hybrid}
\end{figure}

\begin{figure}[H]
    \centering
    \includegraphics[width=0.75\textwidth]{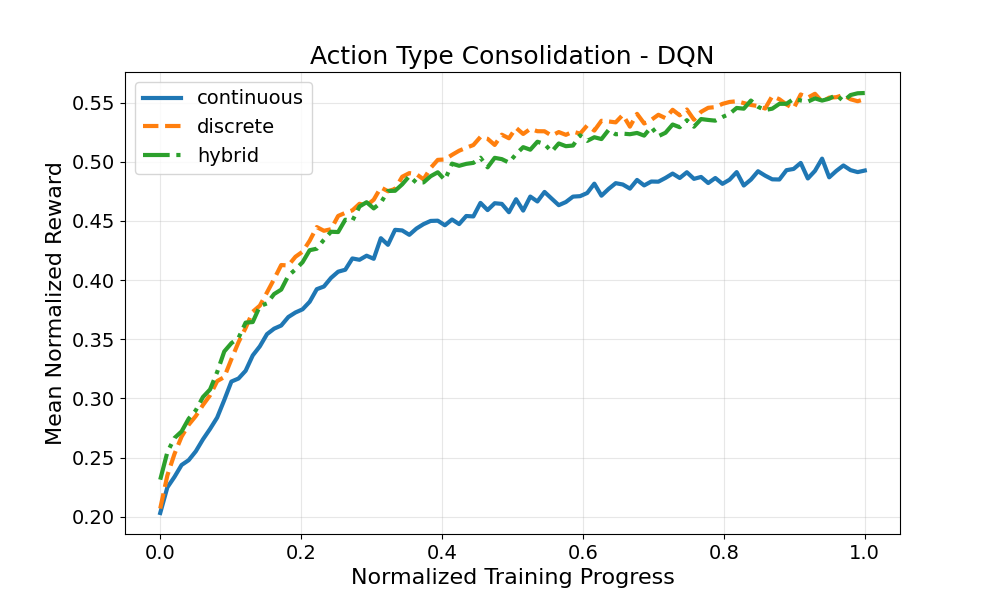}
    \caption{Action Types Aggregated for DQN}
    \label{fig:cons_dqn}
\end{figure}

\begin{figure}[H]
    \centering
    \includegraphics[width=0.75\textwidth]{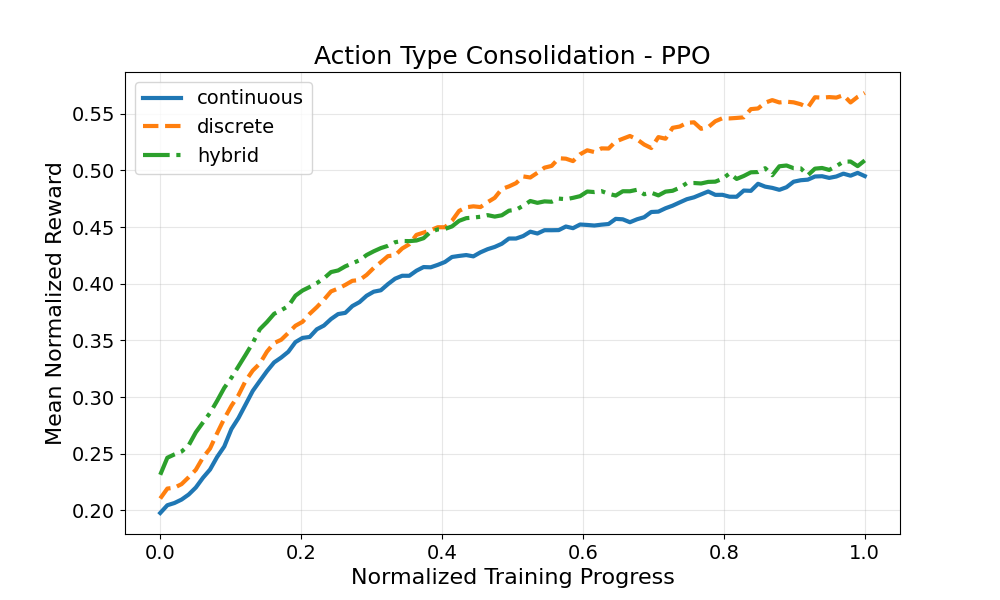}
    \caption{Action Types Aggregated for PPO}
    \label{fig:cons_ppo}
\end{figure}

\begin{figure}[H]
    \centering
    \includegraphics[width=0.75\textwidth]{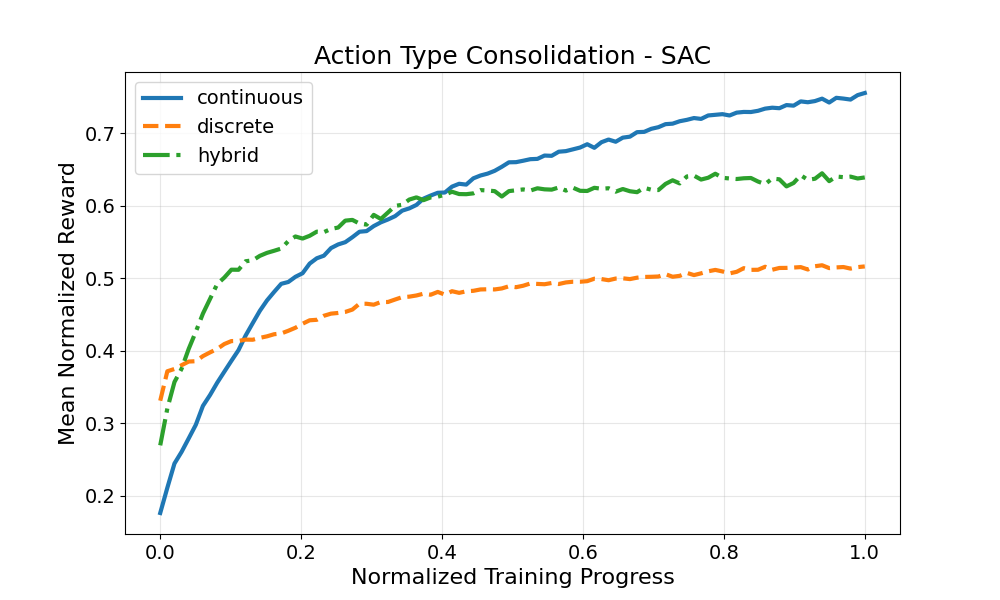}
    \caption{Action Types Aggregated for SAC}
    \label{fig:cons_sac}
\end{figure}

\newpage
\subsection{Detailed Results by Environment and Action Space}

\begin{figure}[H]
    \centering
    \begin{subfigure}[b]{0.48\textwidth}
        \centering
        \includegraphics[width=\textwidth]{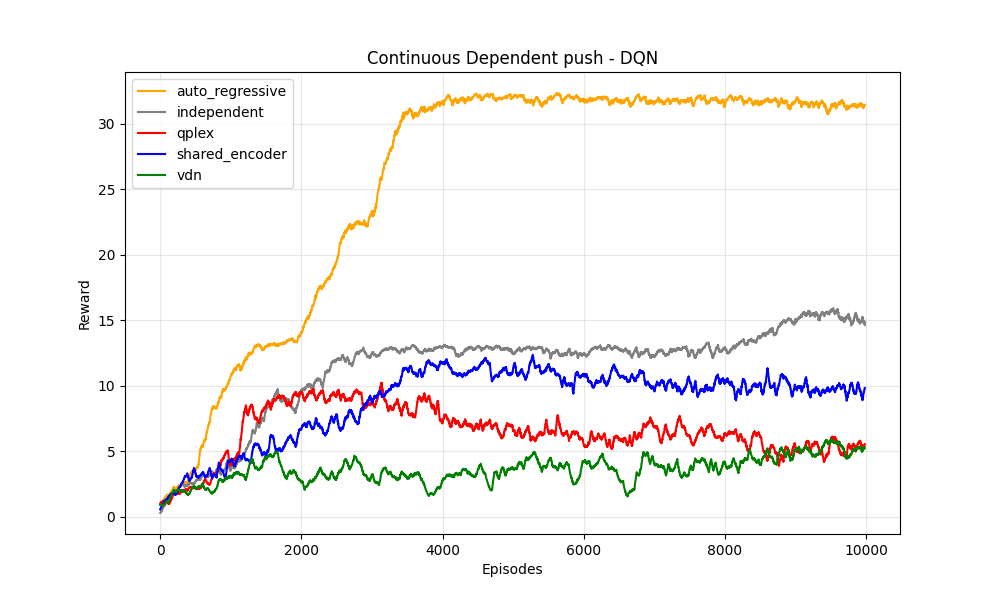}
        \caption{DQN}
    \end{subfigure}
    \hfill
    \begin{subfigure}[b]{0.48\textwidth}
        \centering
        \includegraphics[width=\textwidth]{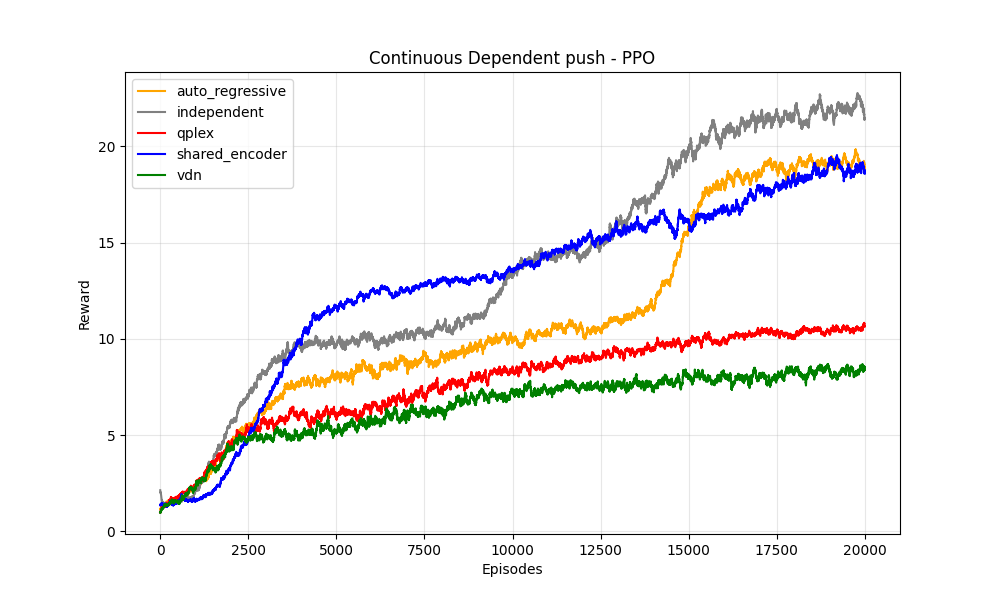}
        \caption{PPO}
    \end{subfigure}
    \vskip\baselineskip
    \begin{subfigure}[b]{0.48\textwidth}
        \centering
        \includegraphics[width=\textwidth]{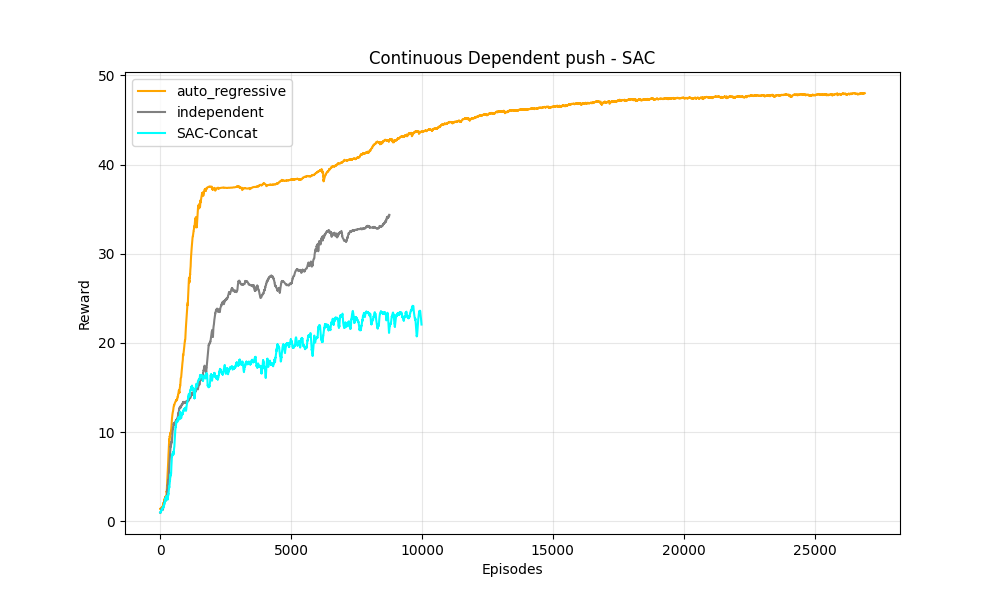}
        \caption{SAC}
    \end{subfigure}
    \caption{Non-aggregated results for Continuous Dependent Push.}
    \label{fig:continuous_dependent_push}
\end{figure}

\begin{figure}[H]
    \centering
    \begin{subfigure}[b]{0.48\textwidth}
        \centering
        \includegraphics[width=\textwidth]{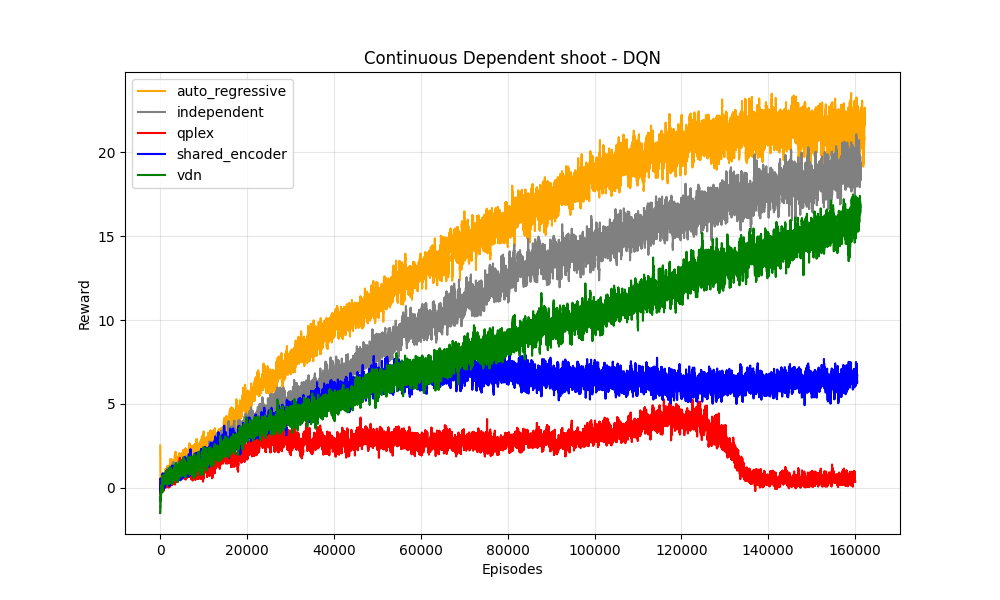}
        \caption{DQN}
    \end{subfigure}
    \hfill
    \begin{subfigure}[b]{0.48\textwidth}
        \centering
        \includegraphics[width=\textwidth]{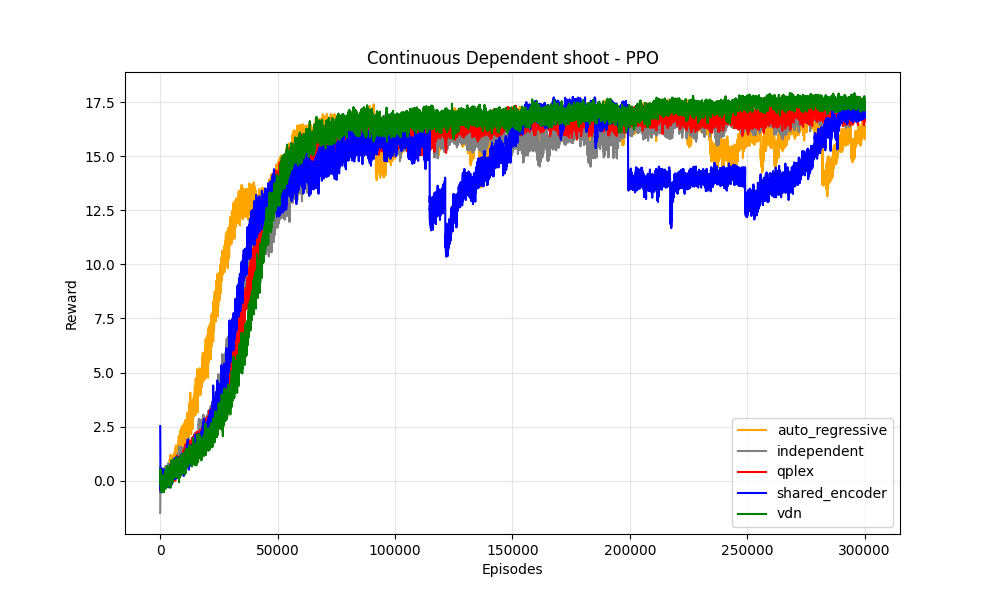}
        \caption{PPO}
    \end{subfigure}
    \vskip\baselineskip
    \begin{subfigure}[b]{0.48\textwidth}
        \centering
        \includegraphics[width=\textwidth]{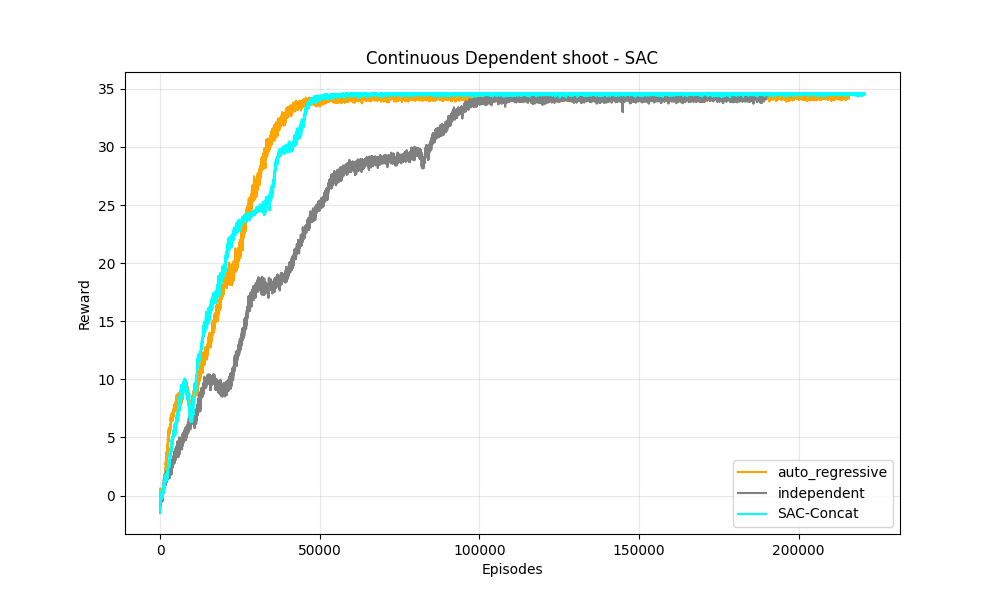}
        \caption{SAC}
    \end{subfigure}
    \caption{Non-aggregated results for Continuous Dependent Shoot.}
    \label{fig:continuous_dependent_shoot}
\end{figure}

\begin{figure}[H]
    \centering
    \begin{subfigure}[b]{0.48\textwidth}
        \centering
        \includegraphics[width=\textwidth]{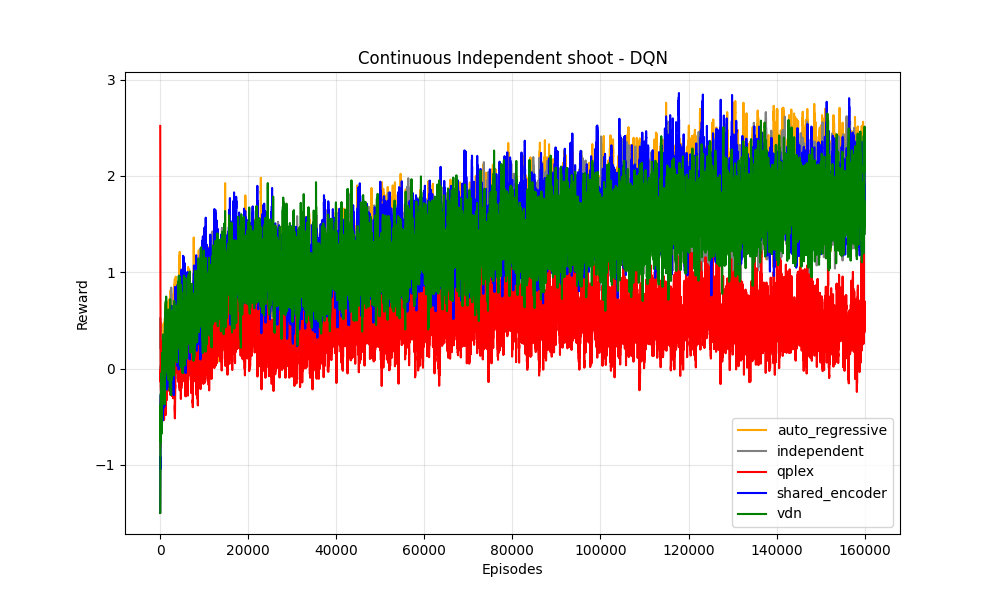}
        \caption{DQN}
    \end{subfigure}
    \hfill
    \begin{subfigure}[b]{0.48\textwidth}
        \centering
        \includegraphics[width=\textwidth]{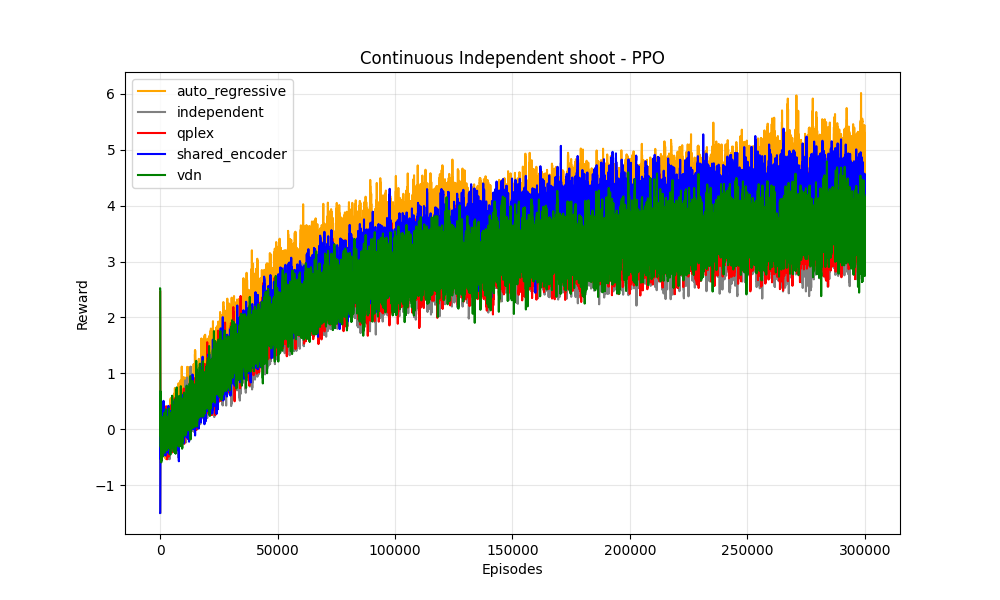}
        \caption{PPO}
    \end{subfigure}
    \vskip\baselineskip
    \begin{subfigure}[b]{0.48\textwidth}
        \centering
        \includegraphics[width=\textwidth]{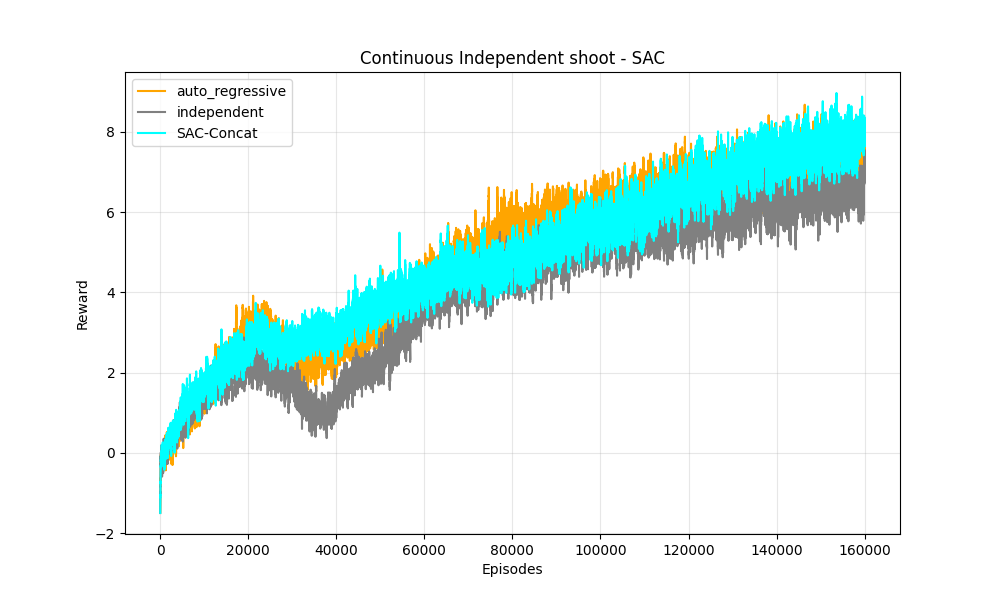}
        \caption{SAC}
    \end{subfigure}
    \caption{Non-aggregated results for Continuous Independent Shoot.}
    \label{fig:continuous_independent_shoot}
\end{figure}

\begin{figure}[H]
    \centering
    \begin{subfigure}[b]{0.48\textwidth}
        \centering
        \includegraphics[width=\textwidth]{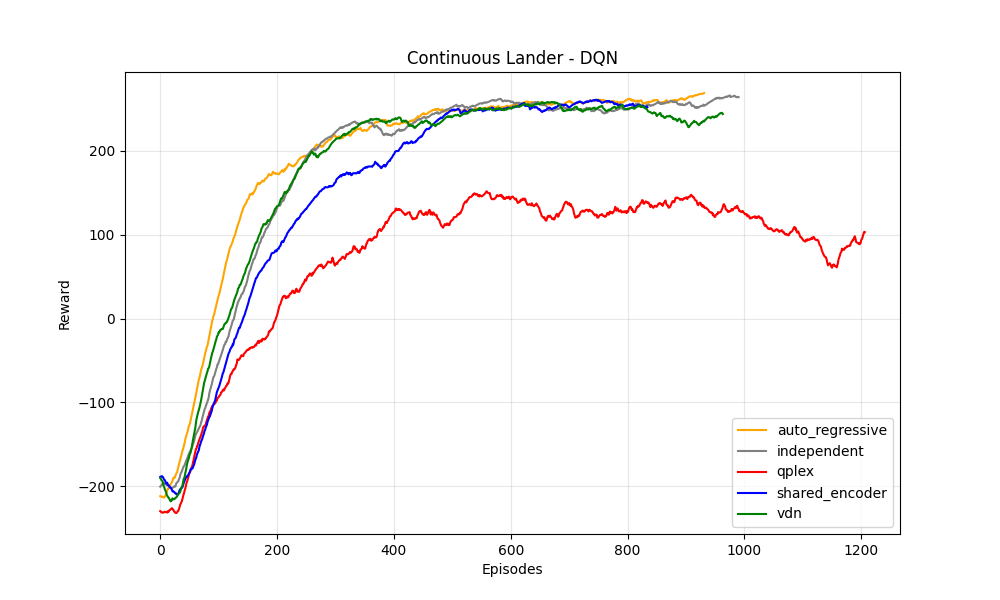}
        \caption{DQN}
    \end{subfigure}
    \hfill
    \begin{subfigure}[b]{0.48\textwidth}
        \centering
        \includegraphics[width=\textwidth]{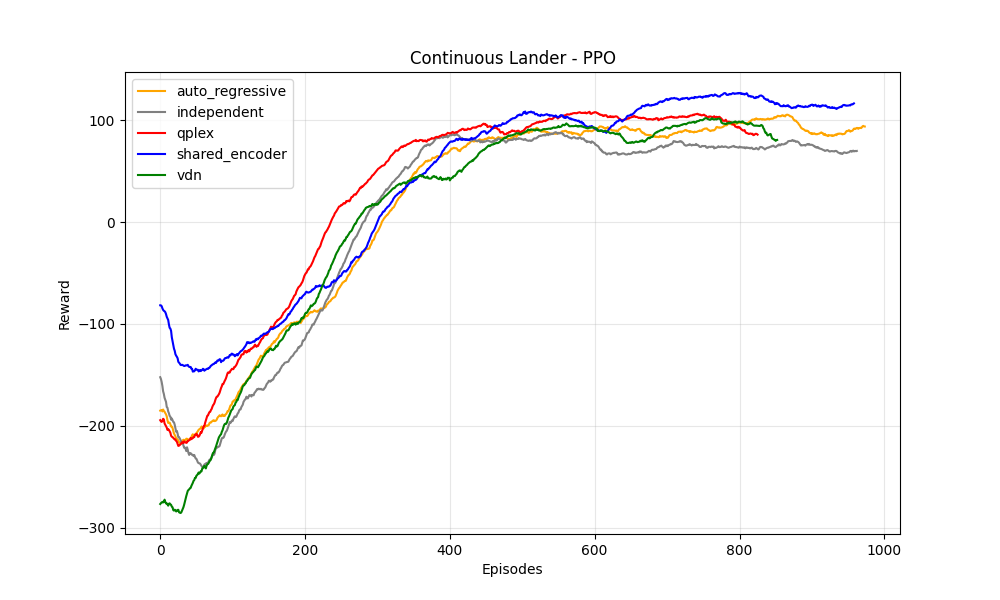}
        \caption{PPO}
    \end{subfigure}
    \vskip\baselineskip
    \begin{subfigure}[b]{0.48\textwidth}
        \centering
        \includegraphics[width=\textwidth]{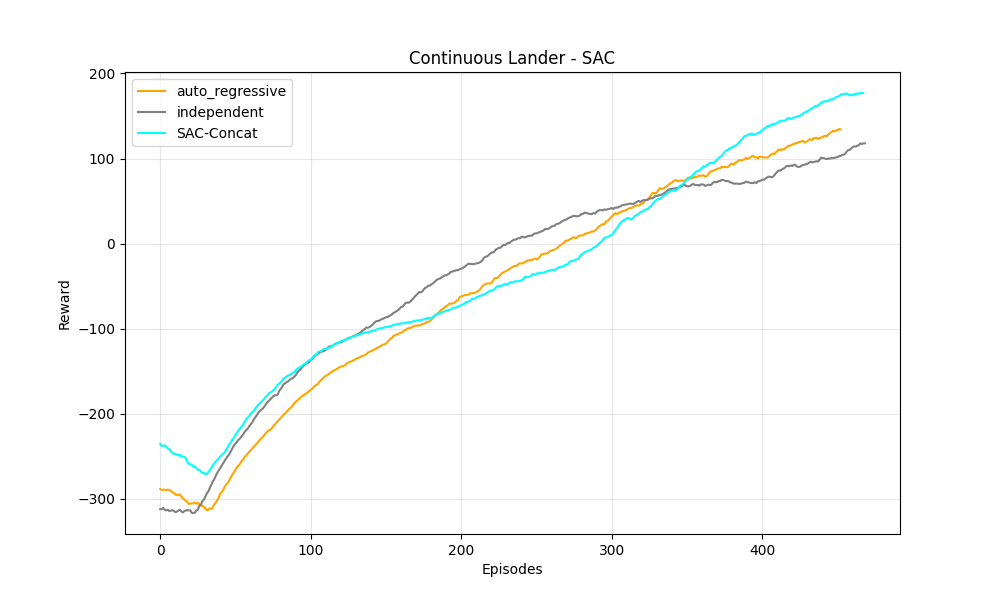}
        \caption{SAC}
    \end{subfigure}
    \caption{Non-aggregated results for Continuous Lander.}
    \label{fig:continuous_lander}
\end{figure}

\begin{figure}[H]
    \centering
    \begin{subfigure}[b]{0.48\textwidth}
        \centering
        \includegraphics[width=\textwidth]{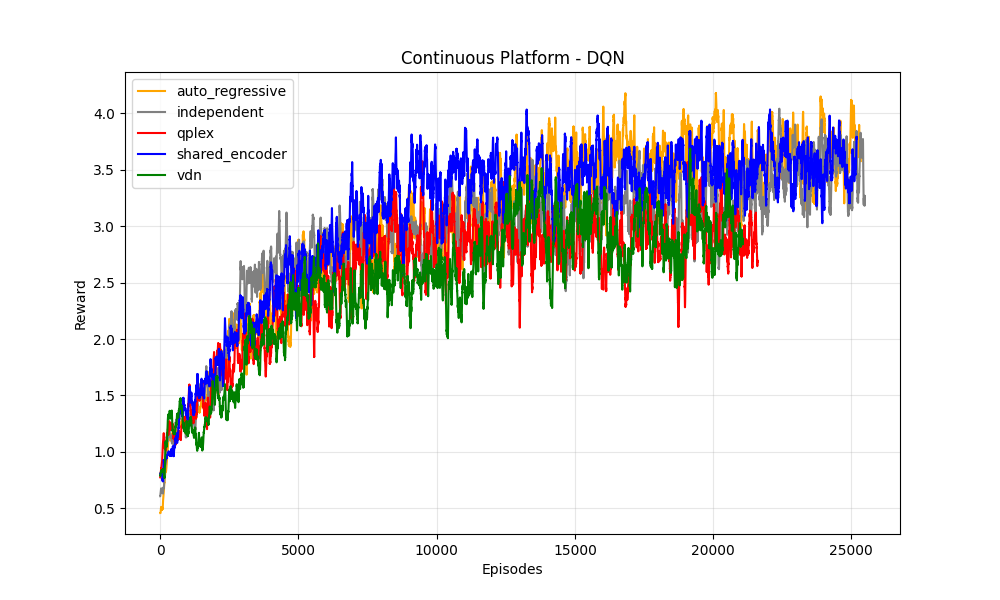}
        \caption{DQN}
    \end{subfigure}
    \hfill
    \begin{subfigure}[b]{0.48\textwidth}
        \centering
        \includegraphics[width=\textwidth]{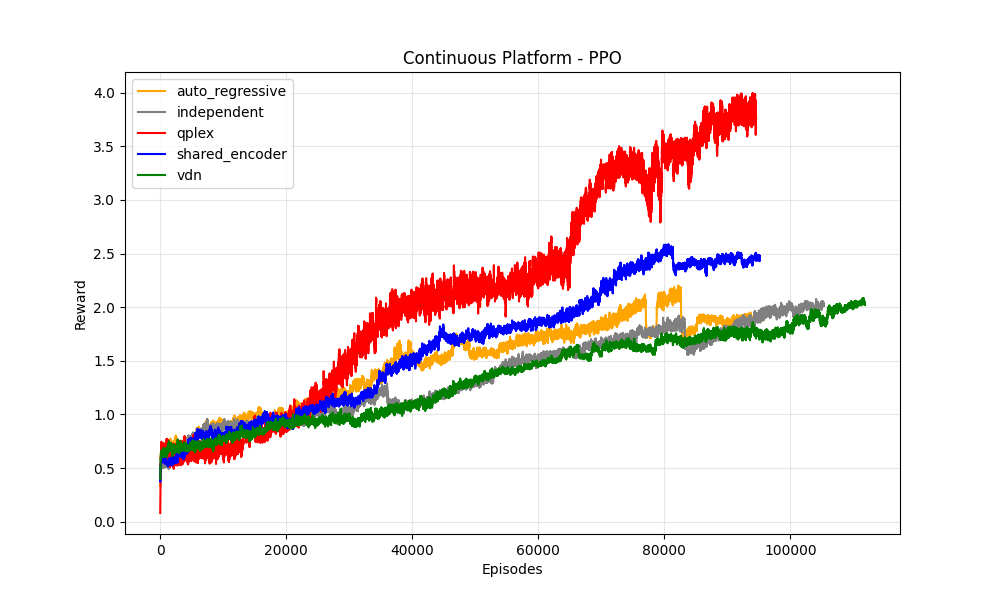}
        \caption{PPO}
    \end{subfigure}
    \vskip\baselineskip
    \begin{subfigure}[b]{0.48\textwidth}
        \centering
        \includegraphics[width=\textwidth]{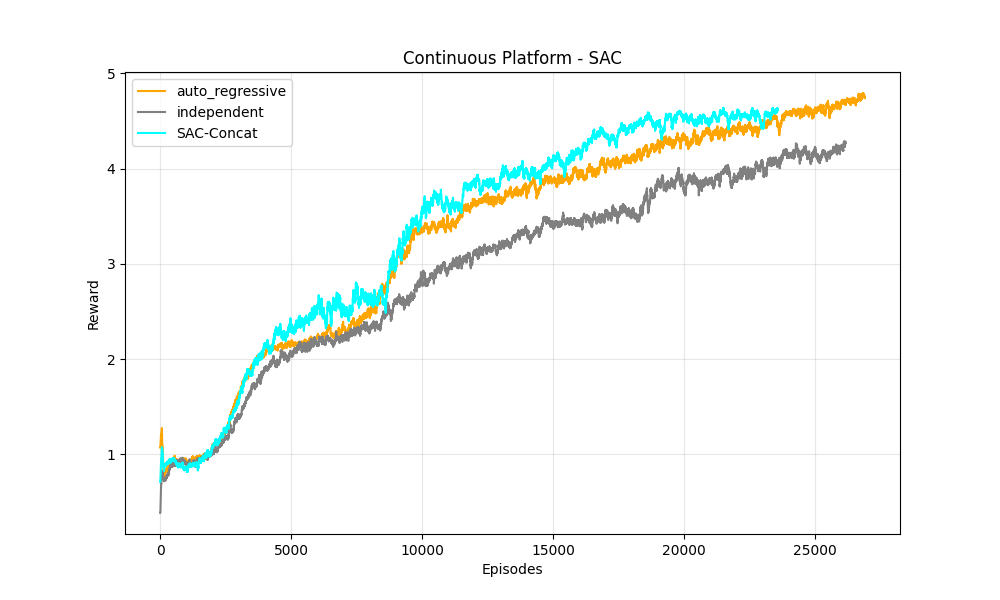}
        \caption{SAC}
    \end{subfigure}
    \caption{Non-aggregated results for Continuous Platform.}
    \label{fig:continuous_platform}
\end{figure}

\begin{figure}[H]
    \centering
    \begin{subfigure}[b]{0.48\textwidth}
        \centering
        \includegraphics[width=\textwidth]{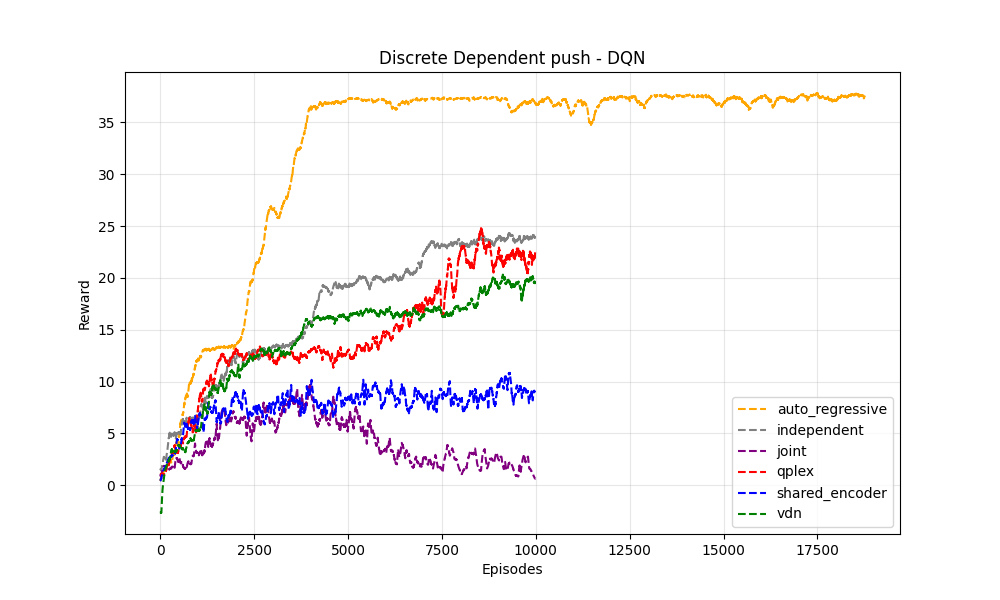}
        \caption{DQN}
    \end{subfigure}
    \hfill
    \begin{subfigure}[b]{0.48\textwidth}
        \centering
        \includegraphics[width=\textwidth]{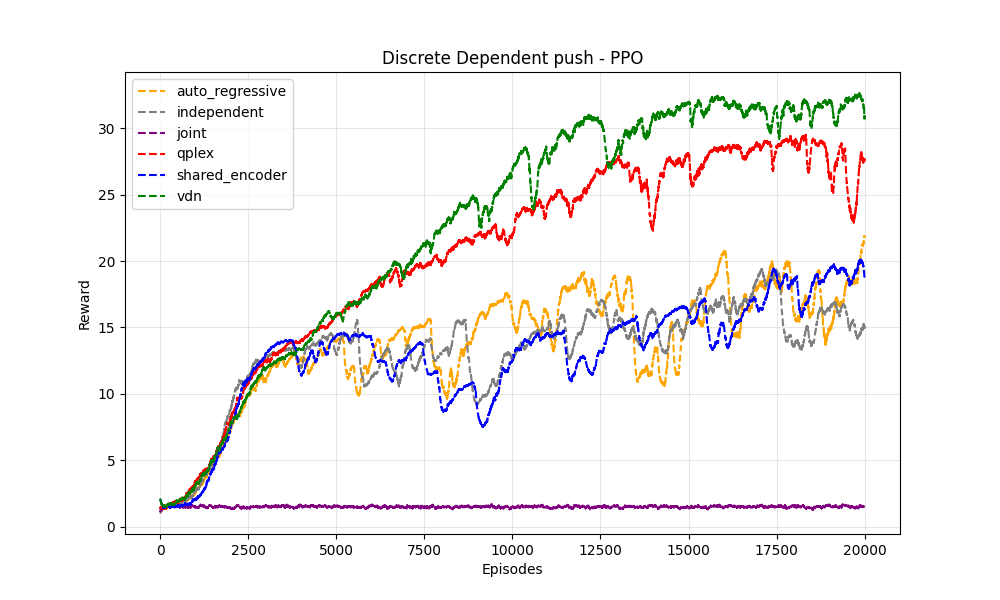}
        \caption{PPO}
    \end{subfigure}
    \vskip\baselineskip
    \begin{subfigure}[b]{0.48\textwidth}
        \centering
        \includegraphics[width=\textwidth]{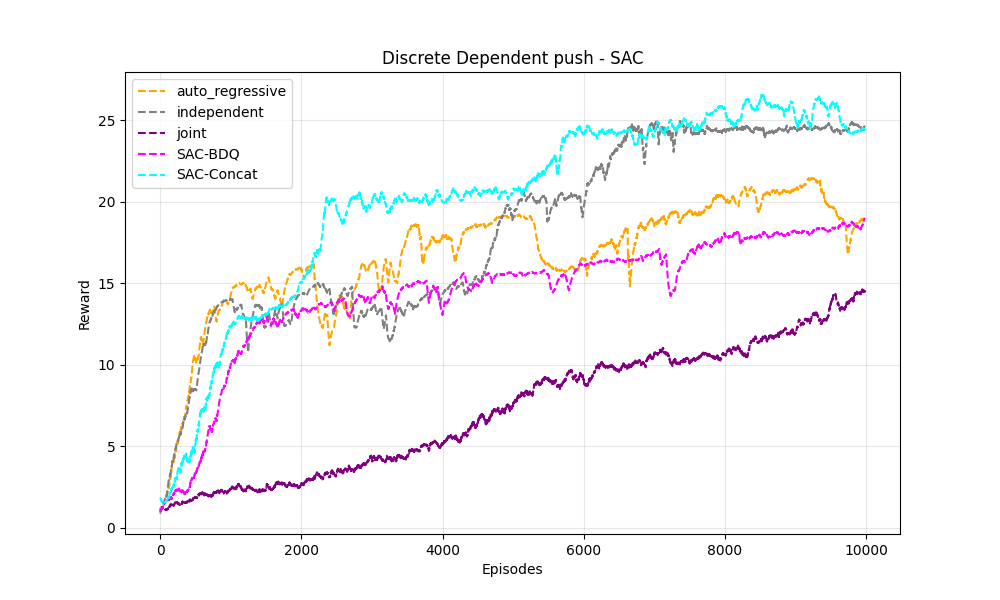}
        \caption{SAC}
    \end{subfigure}
    \caption{Non-aggregated results for Discrete Dependent Push.}
    \label{fig:discrete_dependent_push}
\end{figure}

\begin{figure}[H]
    \centering
    \begin{subfigure}[b]{0.48\textwidth}
        \centering
        \includegraphics[width=\textwidth]{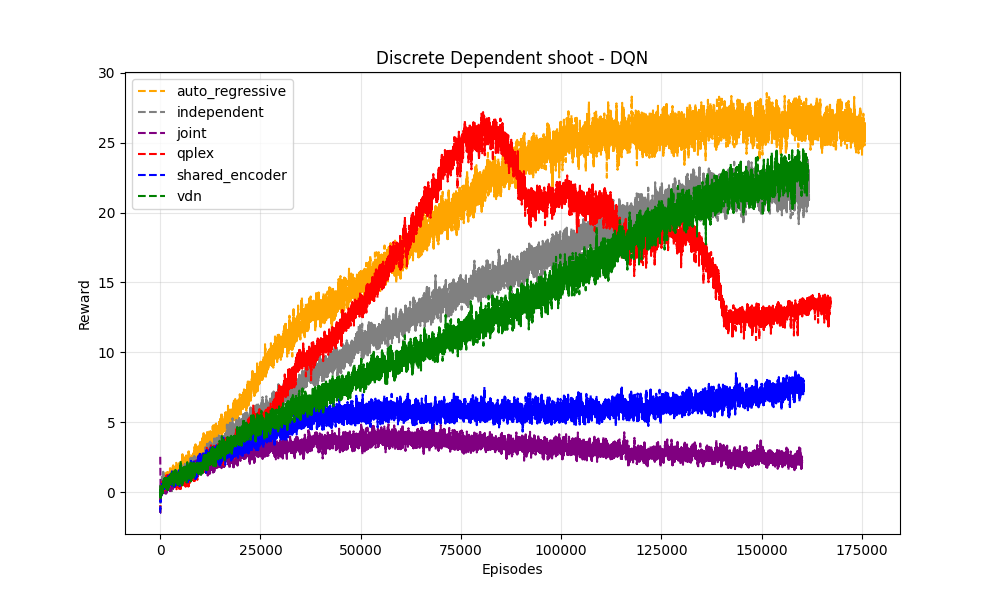}
        \caption{DQN}
    \end{subfigure}
    \hfill
    \begin{subfigure}[b]{0.48\textwidth}
        \centering
        \includegraphics[width=\textwidth]{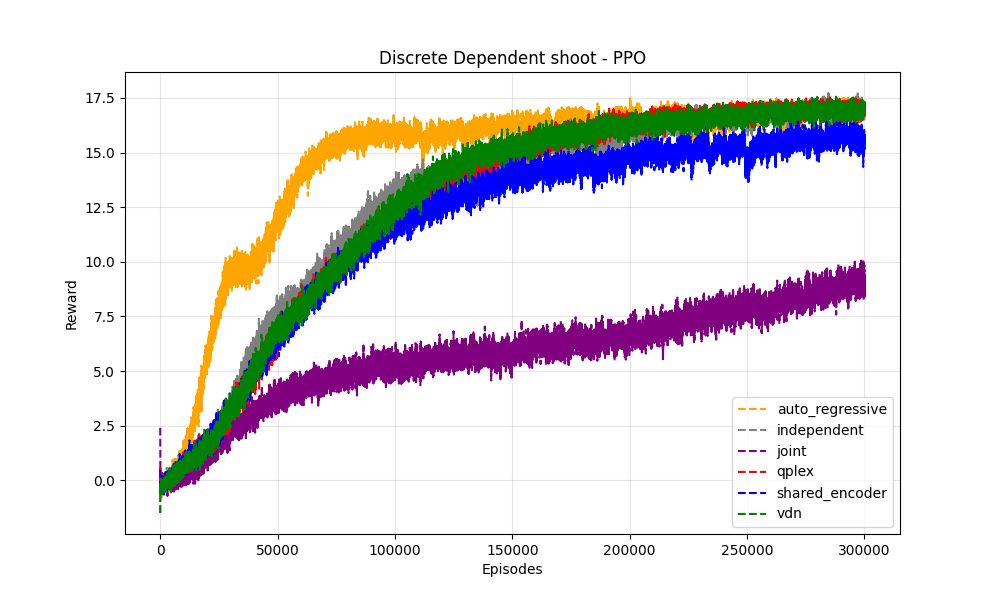}
        \caption{PPO}
    \end{subfigure}
    \vskip\baselineskip
    \begin{subfigure}[b]{0.48\textwidth}
        \centering
        \includegraphics[width=\textwidth]{vectorized_graphs/discrete_dependent_shoot_sac.png}
        \caption{SAC}
    \end{subfigure}
    \caption{Non-aggregated results for Discrete Dependent Shoot.}
    \label{fig:discrete_dependent_shoot}
\end{figure}

\begin{figure}[H]
    \centering
    \begin{subfigure}[b]{0.48\textwidth}
        \centering
        \includegraphics[width=\textwidth]{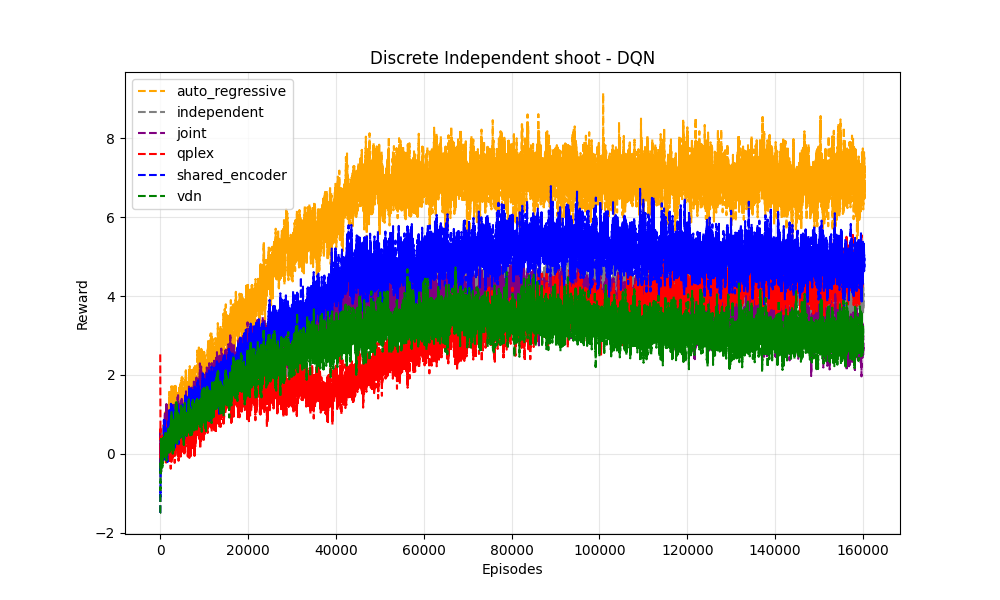}
        \caption{DQN}
    \end{subfigure}
    \hfill
    \begin{subfigure}[b]{0.48\textwidth}
        \centering
        \includegraphics[width=\textwidth]{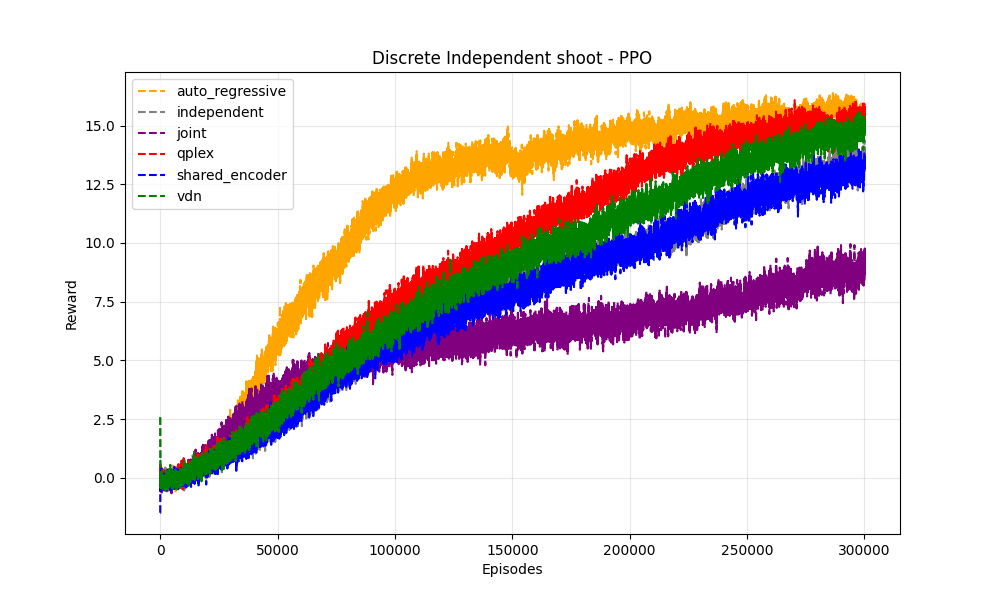}
        \caption{PPO}
    \end{subfigure}
    \vskip\baselineskip
    \begin{subfigure}[b]{0.48\textwidth}
        \centering
        \includegraphics[width=\textwidth]{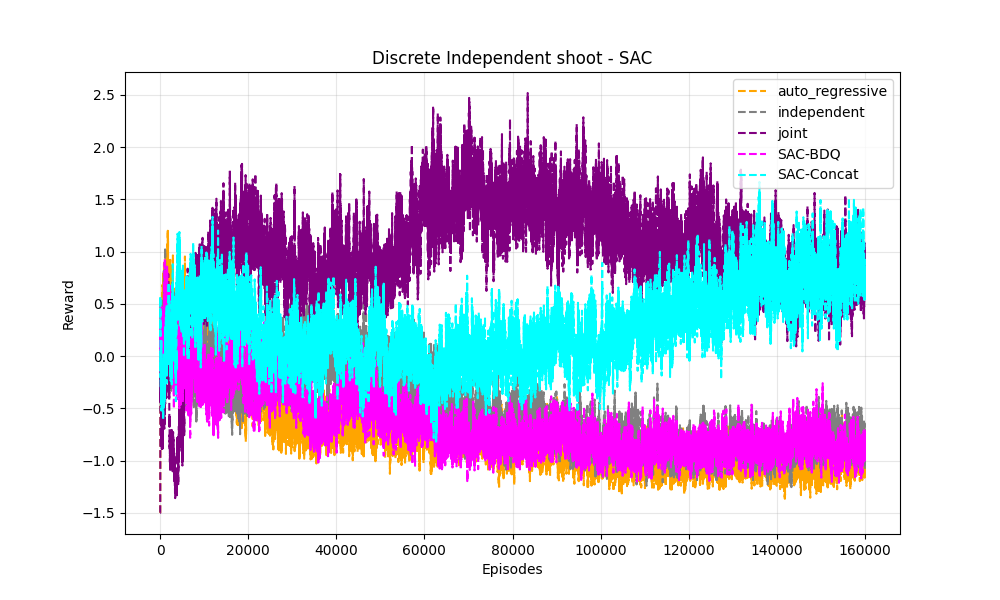}
        \caption{SAC}
    \end{subfigure}
    \caption{Non-aggregated results for Discrete Independent Shoot.}
    \label{fig:discrete_independent_shoot}
\end{figure}

\begin{figure}[H]
    \centering
    \begin{subfigure}[b]{0.48\textwidth}
        \centering
        \includegraphics[width=\textwidth]{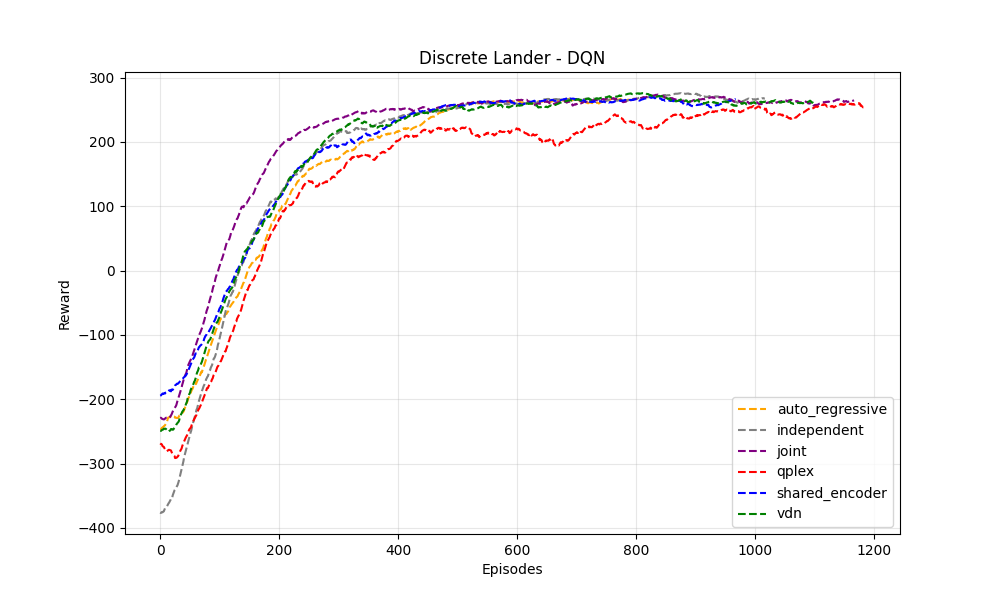}
        \caption{DQN}
    \end{subfigure}
    \hfill
    \begin{subfigure}[b]{0.48\textwidth}
        \centering
        \includegraphics[width=\textwidth]{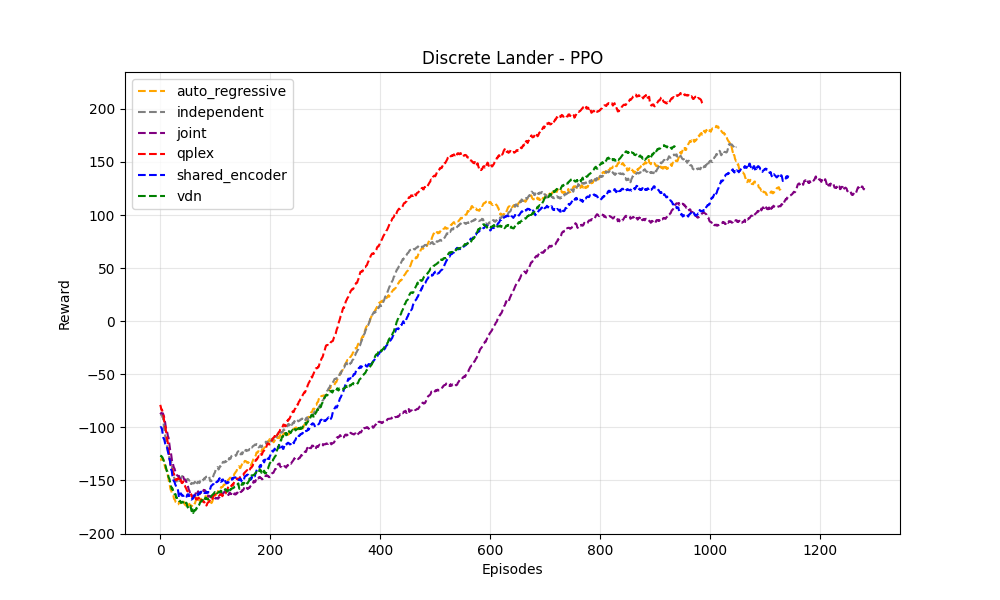}
        \caption{PPO}
    \end{subfigure}
    \vskip\baselineskip
    \begin{subfigure}[b]{0.48\textwidth}
        \centering
        \includegraphics[width=\textwidth]{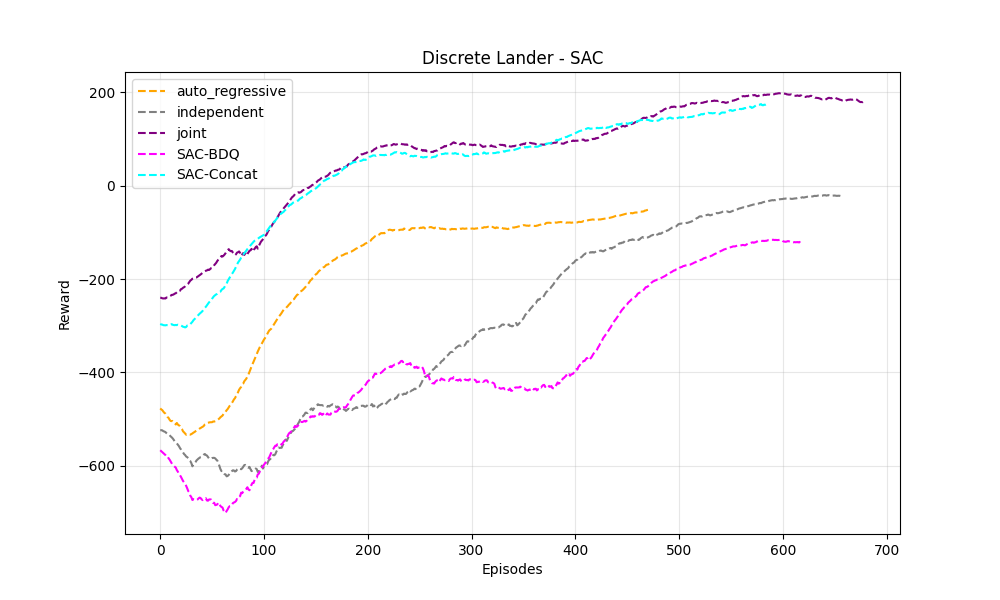}
        \caption{SAC}
    \end{subfigure}
    \caption{Non-aggregated results for Discrete Lander.}
    \label{fig:discrete_lander}
\end{figure}

\begin{figure}[H]
    \centering
    \begin{subfigure}[b]{0.48\textwidth}
        \centering
        \includegraphics[width=\textwidth]{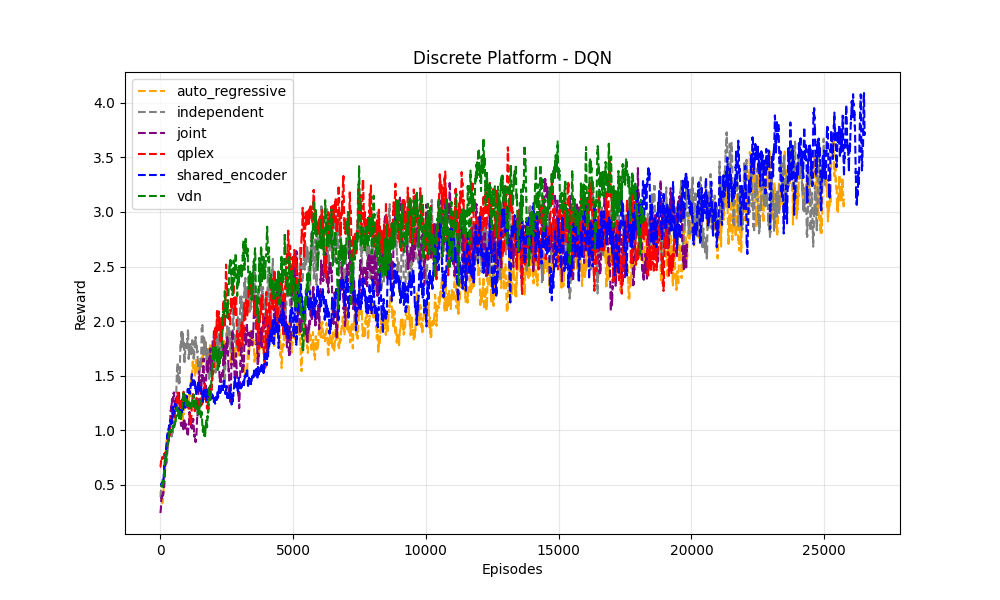}
        \caption{DQN}
    \end{subfigure}
    \hfill
    \begin{subfigure}[b]{0.48\textwidth}
        \centering
        \includegraphics[width=\textwidth]{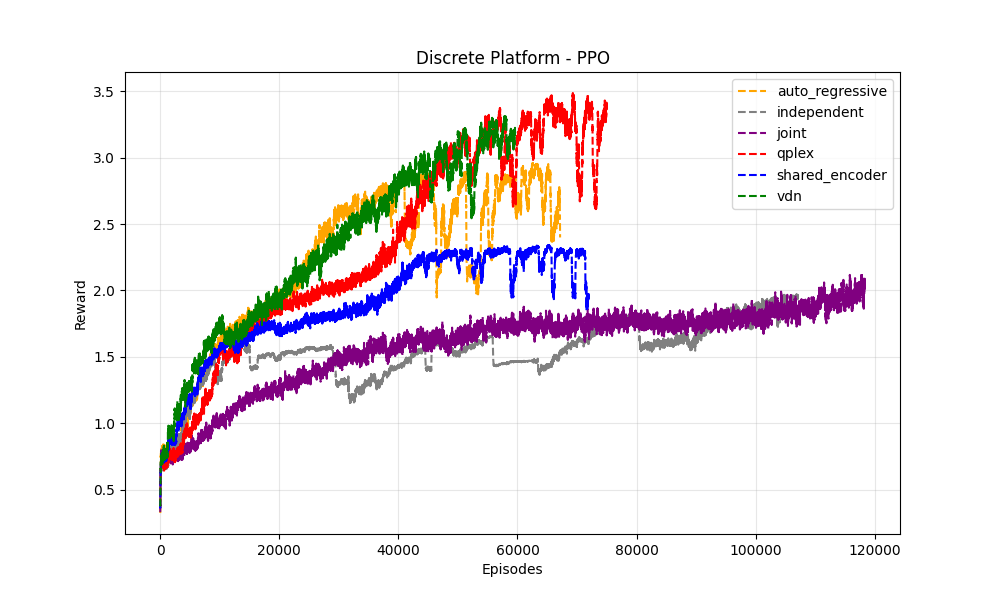}
        \caption{PPO}
    \end{subfigure}
    \vskip\baselineskip
    \begin{subfigure}[b]{0.48\textwidth}
        \centering
        \includegraphics[width=\textwidth]{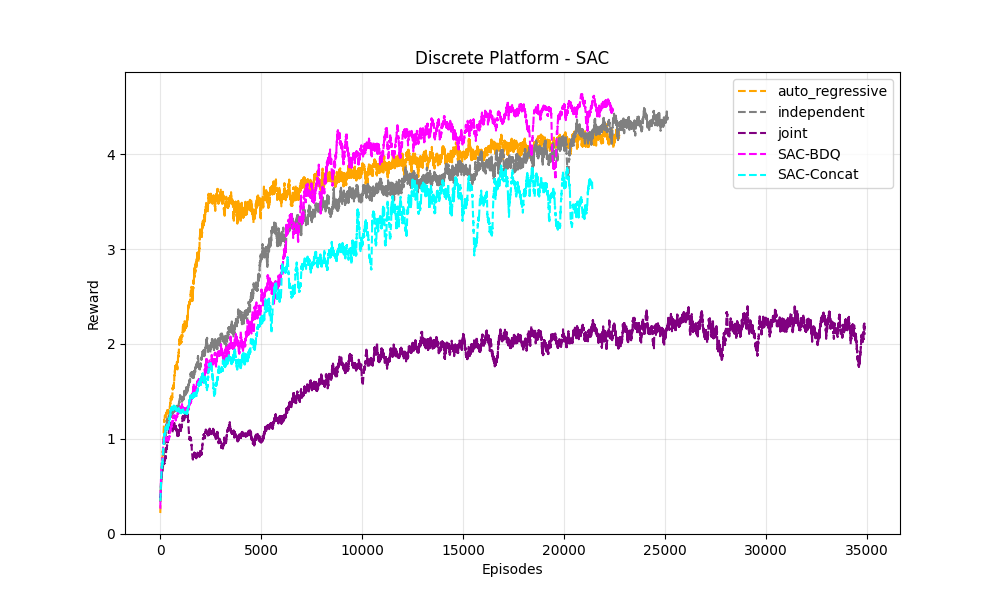}
        \caption{SAC}
    \end{subfigure}
    \caption{Non-aggregated results for Discrete Platform.}
    \label{fig:discrete_platform}
\end{figure}

\begin{figure}[H]
    \centering
    \begin{subfigure}[b]{0.48\textwidth}
        \centering
        \includegraphics[width=\textwidth]{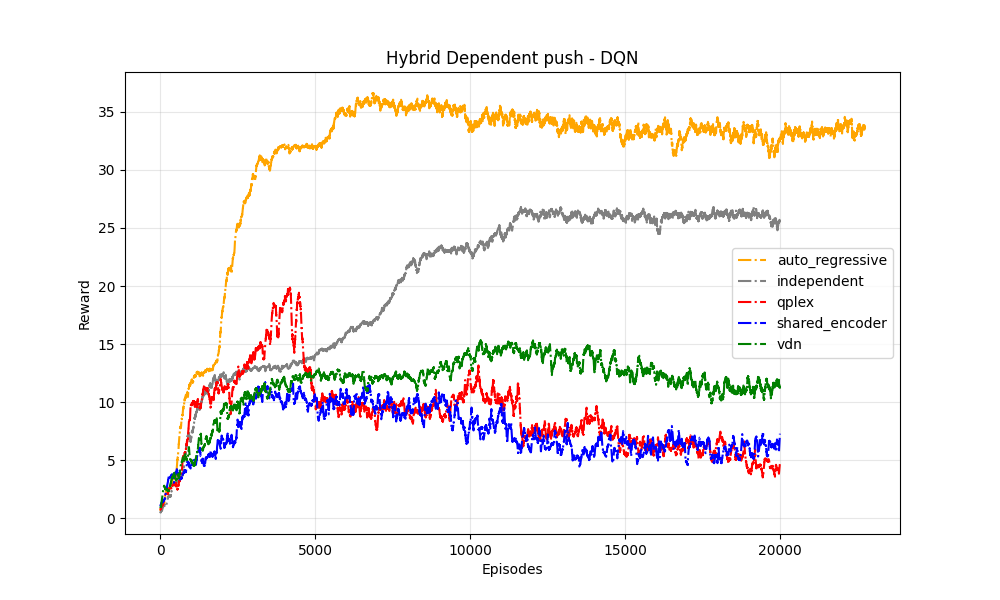}
        \caption{DQN}
    \end{subfigure}
    \hfill
    \begin{subfigure}[b]{0.48\textwidth}
        \centering
        \includegraphics[width=\textwidth]{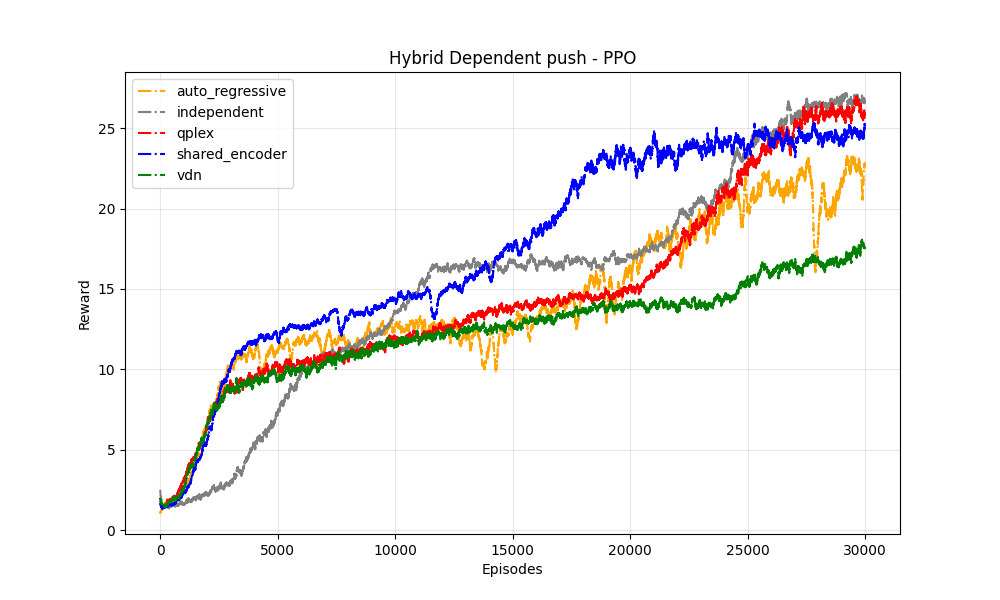}
        \caption{PPO}
    \end{subfigure}
    \vskip\baselineskip
    \begin{subfigure}[b]{0.48\textwidth}
        \centering
        \includegraphics[width=\textwidth]{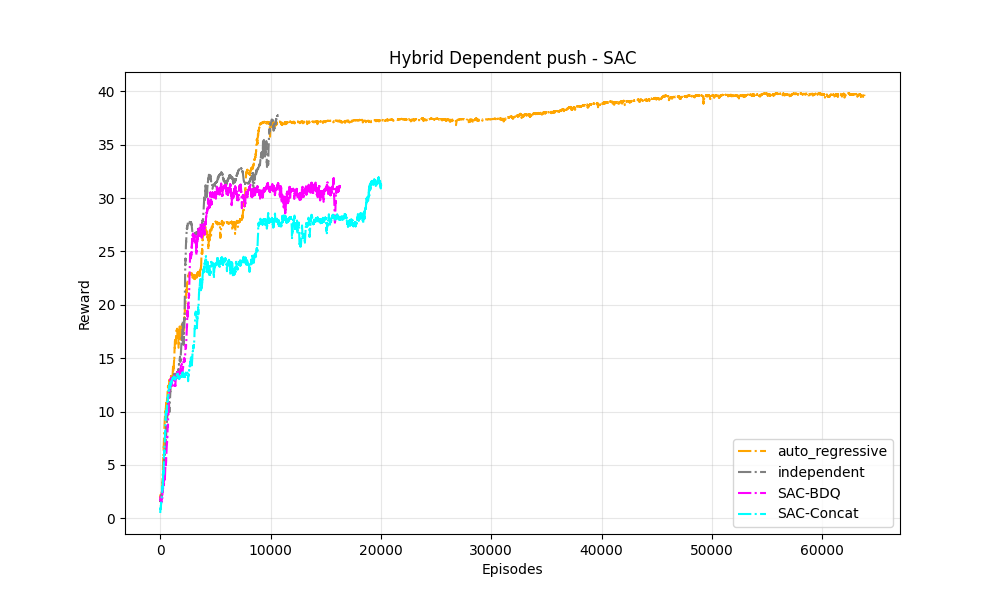}
        \caption{SAC}
    \end{subfigure}
    \caption{Non-aggregated results for Hybrid Dependent Push.}
    \label{fig:hybrid_dependent_push}
\end{figure}

\begin{figure}[H]
    \centering
    \begin{subfigure}[b]{0.48\textwidth}
        \centering
        \includegraphics[width=\textwidth]{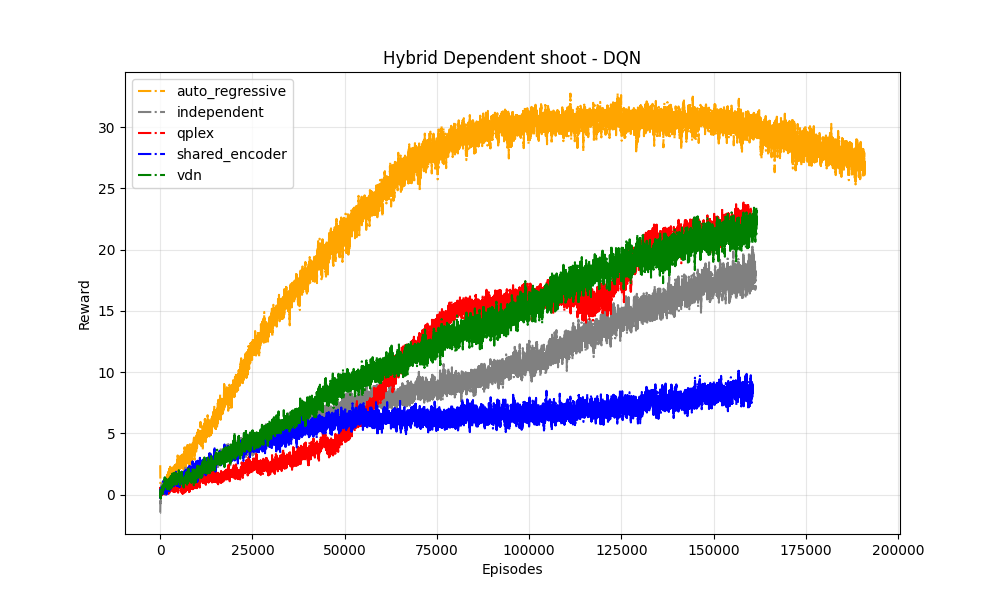}
        \caption{DQN}
    \end{subfigure}
    \hfill
    \begin{subfigure}[b]{0.48\textwidth}
        \centering
        \includegraphics[width=\textwidth]{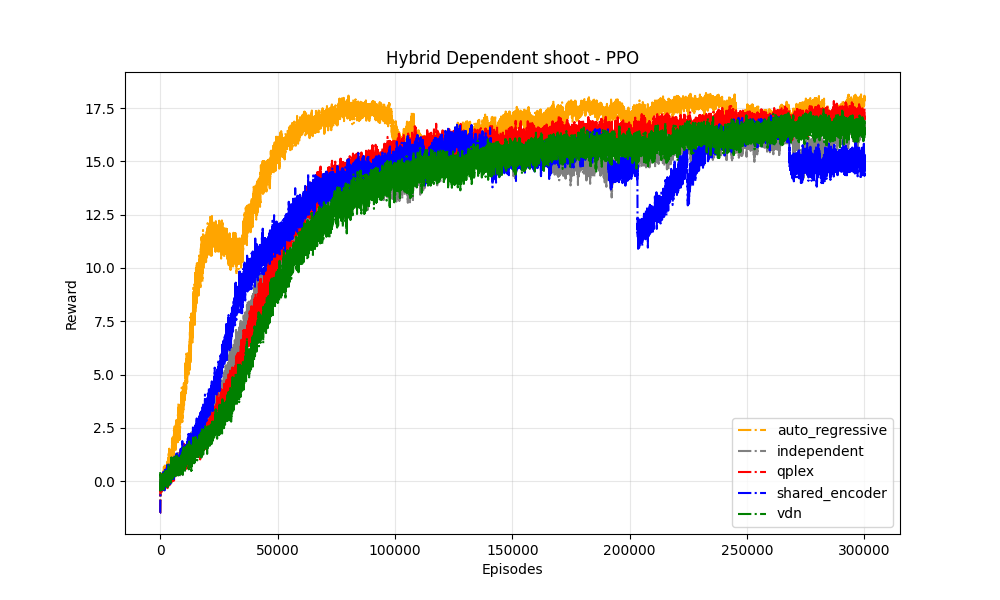}
        \caption{PPO}
    \end{subfigure}
    \vskip\baselineskip
    \begin{subfigure}[b]{0.48\textwidth}
        \centering
        \includegraphics[width=\textwidth]{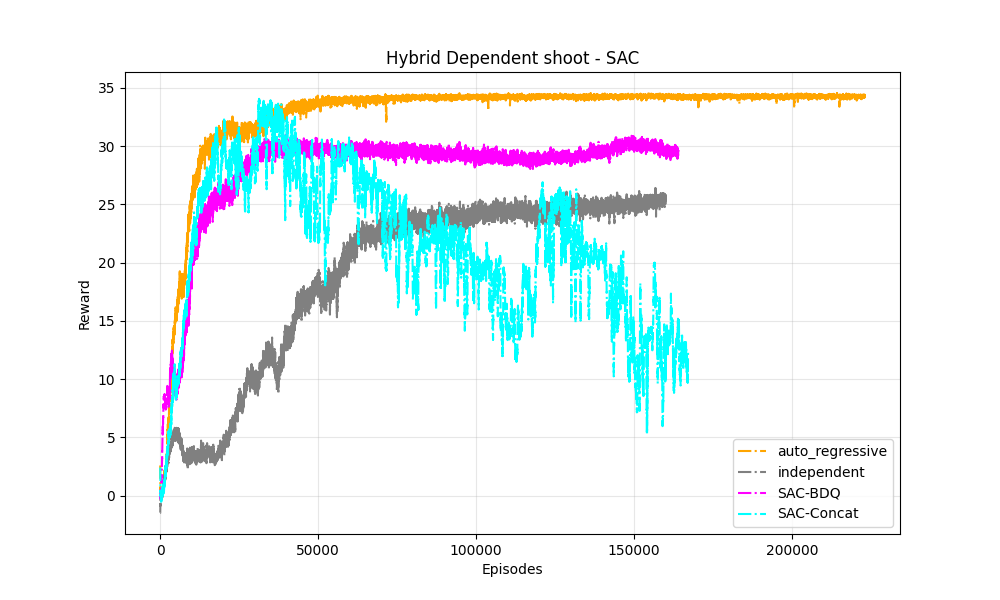}
        \caption{SAC}
    \end{subfigure}
    \caption{Non-aggregated results for Hybrid Dependent Shoot.}
    \label{fig:hybrid_dependent_shoot}
\end{figure}

\begin{figure}[H]
    \centering
    \begin{subfigure}[b]{0.48\textwidth}
        \centering
        \includegraphics[width=\textwidth]{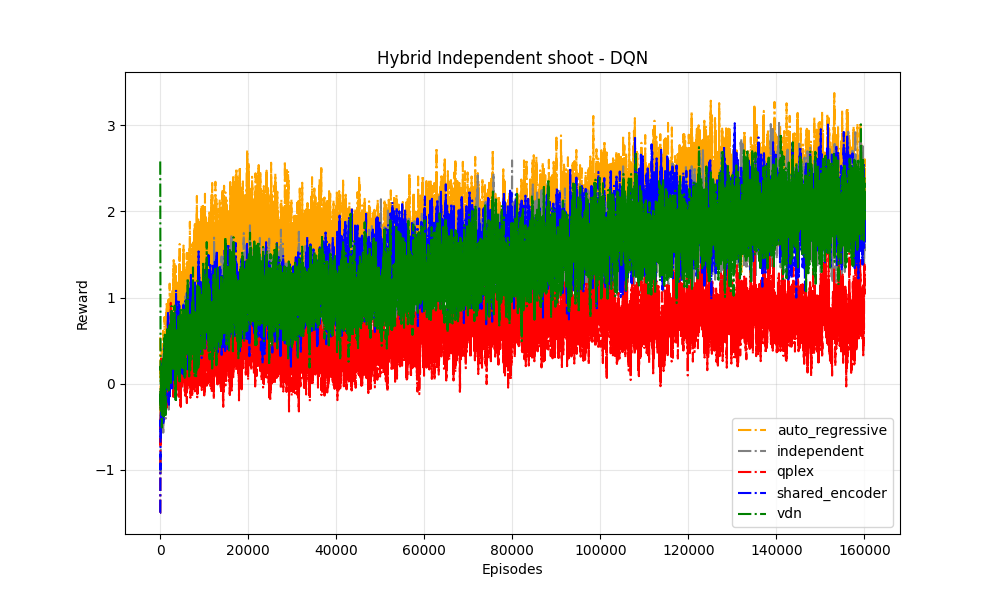}
        \caption{DQN}
    \end{subfigure}
    \hfill
    \begin{subfigure}[b]{0.48\textwidth}
        \centering
        \includegraphics[width=\textwidth]{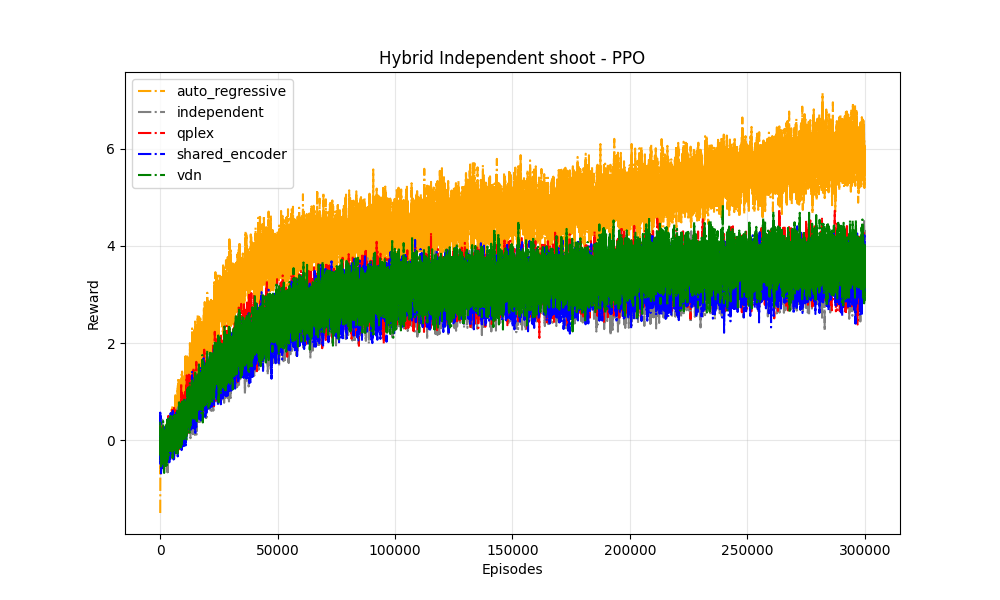}
        \caption{PPO}
    \end{subfigure}
    \vskip\baselineskip
    \begin{subfigure}[b]{0.48\textwidth}
        \centering
        \includegraphics[width=\textwidth]{vectorized_graphs/hybrid_independent_shoot_sac.png}
        \caption{SAC}
    \end{subfigure}
    \caption{Non-aggregated results for Hybrid Independent Shoot.}
    \label{fig:hybrid_independent_shoot}
\end{figure}

\begin{figure}[H]
    \centering
    \begin{subfigure}[b]{0.48\textwidth}
        \centering
        \includegraphics[width=\textwidth]{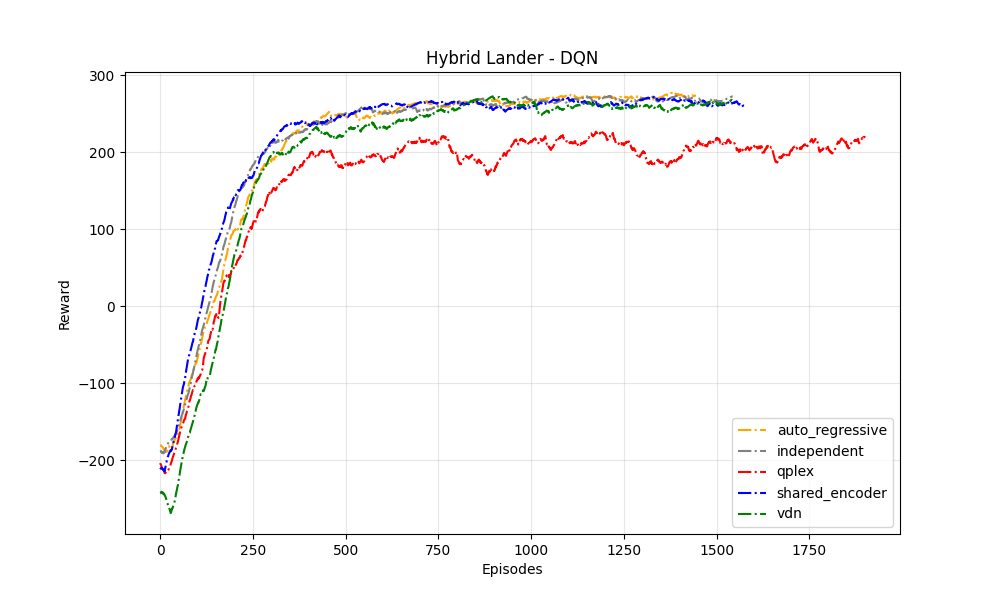}
        \caption{DQN}
    \end{subfigure}
    \hfill
    \begin{subfigure}[b]{0.48\textwidth}
        \centering
        \includegraphics[width=\textwidth]{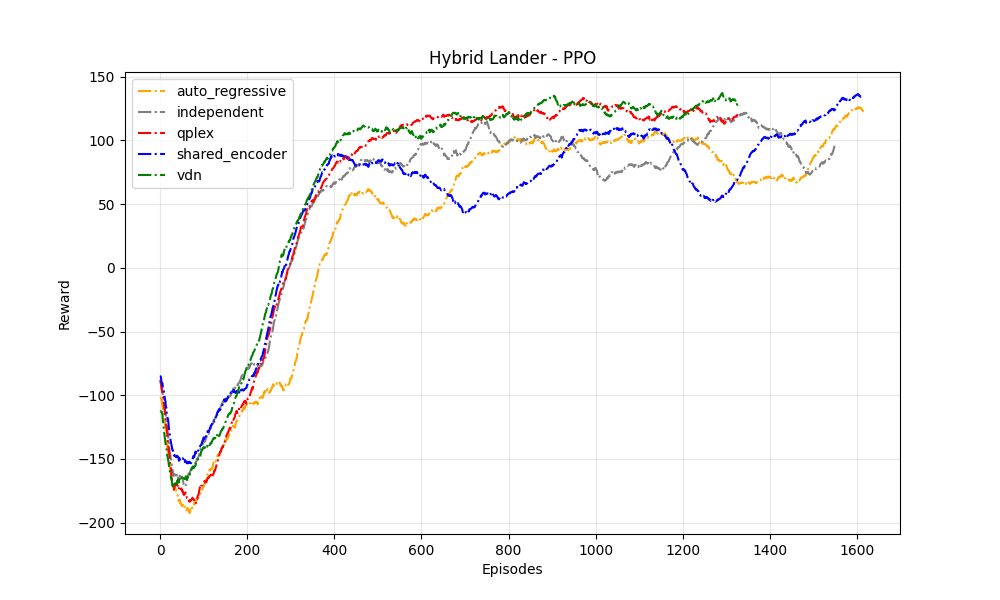}
        \caption{PPO}
    \end{subfigure}
    \vskip\baselineskip
    \begin{subfigure}[b]{0.48\textwidth}
        \centering
        \includegraphics[width=\textwidth]{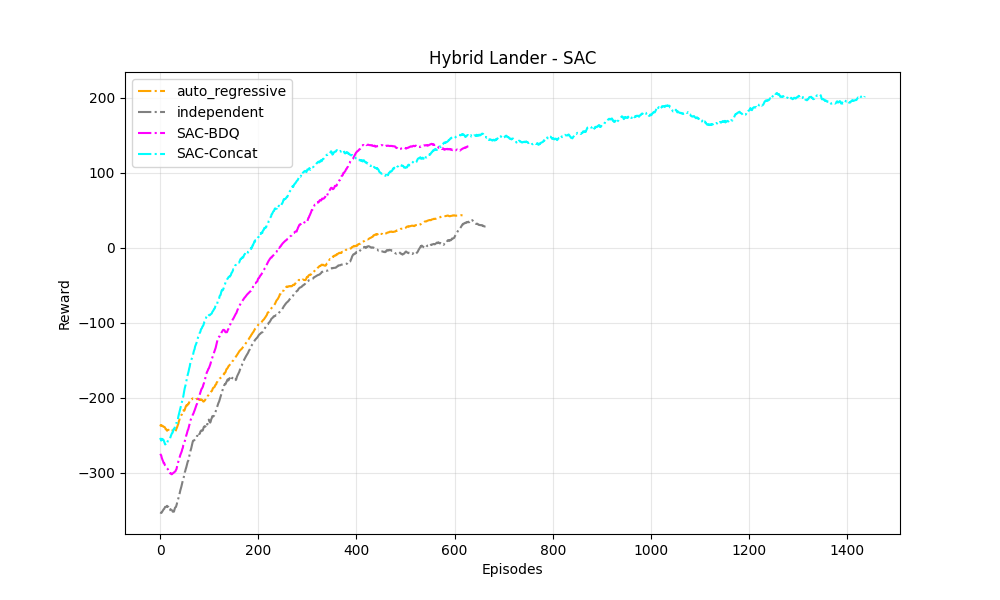}
        \caption{SAC}
    \end{subfigure}
    \caption{Non-aggregated results for Hybrid Lander.}
    \label{fig:hybrid_lander}
\end{figure}

\begin{figure}[H]
    \centering
    \begin{subfigure}[b]{0.48\textwidth}
        \centering
        \includegraphics[width=\textwidth]{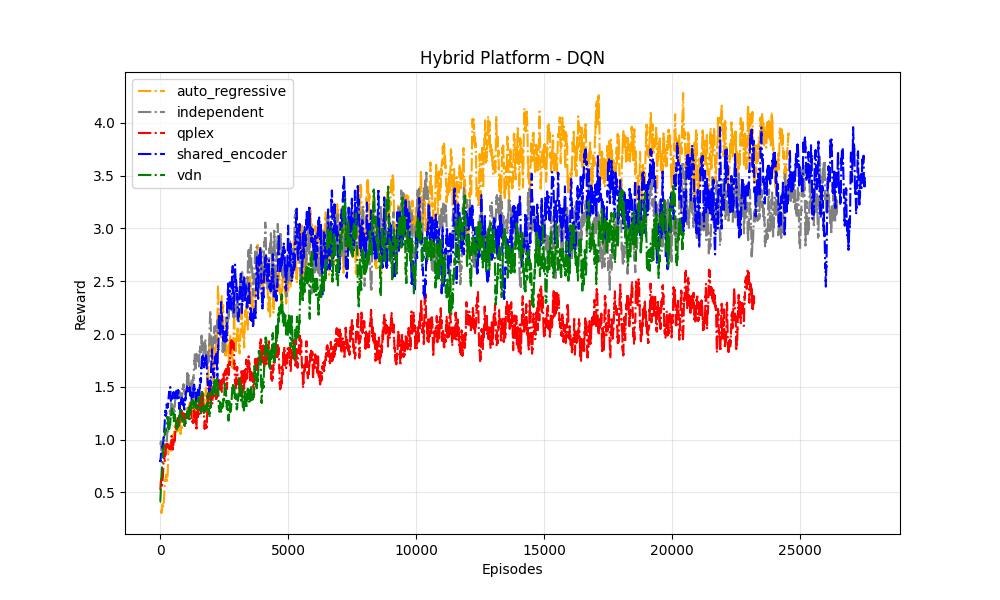}
        \caption{DQN}
    \end{subfigure}
    \hfill
    \begin{subfigure}[b]{0.48\textwidth}
        \centering
        \includegraphics[width=\textwidth]{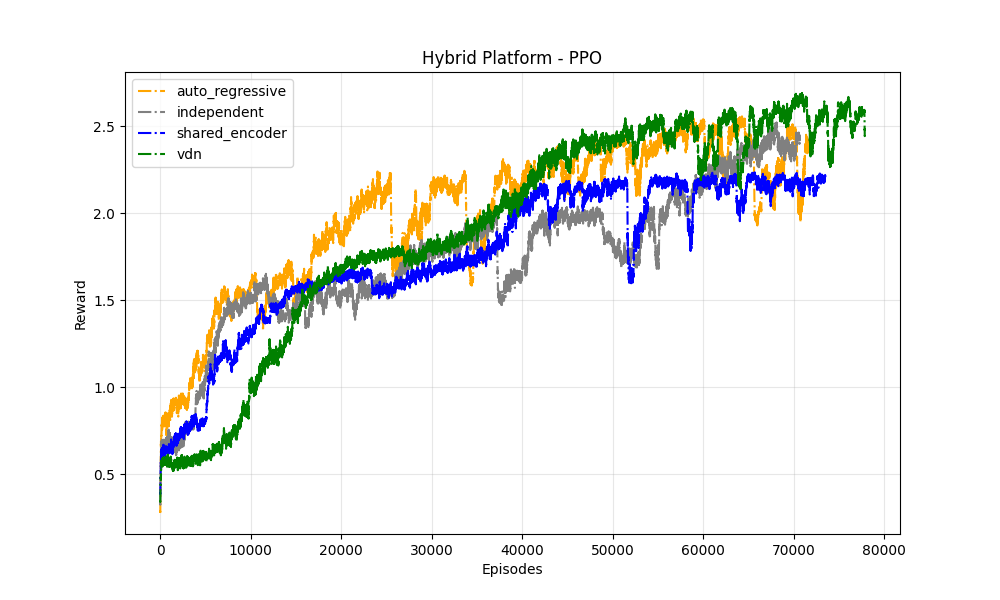}
        \caption{PPO}
    \end{subfigure}
    \vskip\baselineskip
    \begin{subfigure}[b]{0.48\textwidth}
        \centering
        \includegraphics[width=\textwidth]{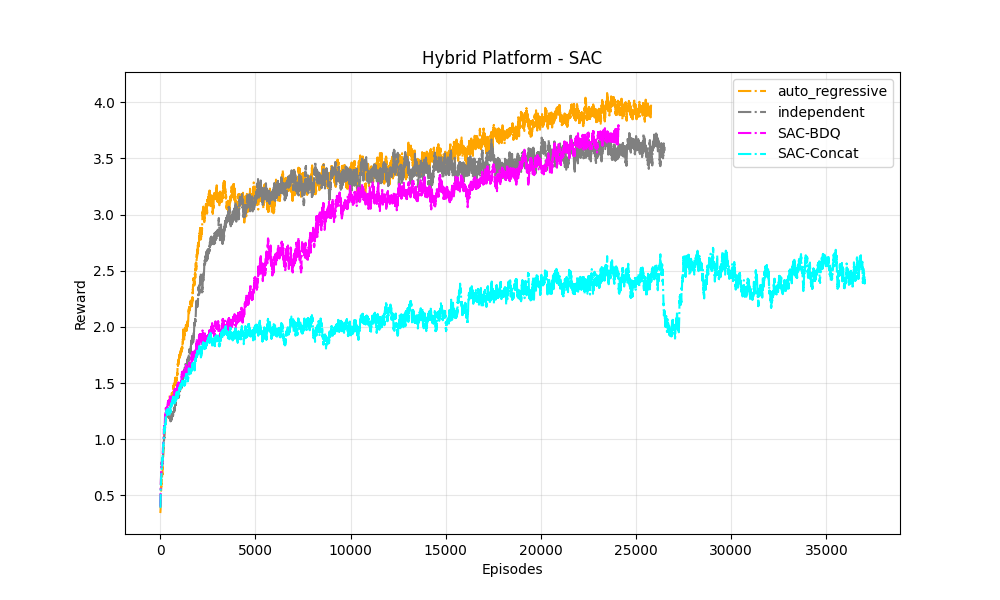}
        \caption{SAC}
    \end{subfigure}
    \caption{Non-aggregated results for Hybrid Platform.}
    \label{fig:hybrid_platform}
\end{figure}

\end{document}